\documentclass[12pt,a4]{article}
\usepackage{amsmath}
\usepackage{amssymb}
\usepackage{epsfig}
\usepackage{amssymb,amsfonts}
\usepackage{graphicx}
\usepackage{authblk}
\usepackage{subfig} 
\usepackage{mlmath}
\usepackage{cleveref}

\newtheorem{innercustomgeneric}{\customgenericname}
\providecommand{\customgenericname}{}

\providecommand{\keywords}[1]
{
  \small	
  \textbf{\textit{Keywords---}} #1
}
\newtheorem{lemma}{Lemma}
\newtheorem{definition}{Definition}
\newtheorem{remark}{Remark}
\newtheorem{assumption}{Assumption}

\newtheorem{proposition}{Proposition}
\newtheorem{theorem}{Theorem}

\begin{document}
\numberwithin{equation}{section}
\title{Phase Diagram of Initial Condensation for Two-layer Neural Networks}
\author{
Zhengan Chen\textsuperscript{\rm 1,2}, 
Yuqing Li\textsuperscript{\rm 1,2} \thanks{Corresponding author: liyuqing\underline{~}551@sjtu.edu.cn},
Tao Luo\textsuperscript{\rm 1,2,3,4,6}, 
Zhangchen Zhou\textsuperscript{\rm 1,5}, 
Zhi-Qin John Xu\textsuperscript{\rm 1,3,4}\thanks{Corresponding author: xuzhiqin@sjtu.edu.cn.} \\
\textsuperscript{\rm 1}  School of Mathematical Sciences, Shanghai Jiao Tong University \\
\textsuperscript{\rm 2}  CMA-Shanghai, Shanghai Jiao Tong University\\
\textsuperscript{\rm 3} Institute of Natural Sciences, MOE-LSC, Shanghai Jiao Tong University \\
\textsuperscript{\rm 4} Qing Yuan Research Institute, Shanghai Jiao Tong University\\
\textsuperscript{\rm 5} Zhiyuan College, Shanghai Jiao Tong University\\
\textsuperscript{\rm 6} Shanghai Artificial Intelligence Laboratory

\{zhengan\underline{~}chen, liyuqing\underline{~}551, luotao41, zczhou1115, xuzhiqin\}@sjtu.edu.cn.
}
 
\date{\today}
\maketitle
\begin{abstract}

The phenomenon of distinct behaviors exhibited by neural networks under varying scales of initialization remains an enigma in deep learning research. In this paper,  based on the earlier work  by Luo et  al.~\cite{luo2021phase}, we present a phase diagram of initial condensation for two-layer neural networks. Condensation is  a phenomenon wherein the weight vectors of neural networks concentrate on isolated orientations during the training process, and it is a  feature in non-linear learning process  that enables neural networks to possess better generalization abilities. Our phase diagram serves to provide a comprehensive understanding of the dynamical regimes of neural networks and their dependence on the choice of hyperparameters related to initialization. Furthermore, we demonstrate in detail the underlying mechanisms by which small initialization leads to condensation at the initial training stage.

\end{abstract}
\allowdisplaybreaks
\section{Introduction} \label{sec...Introduction}
 
In deep learning, one intriguing observation is the distinct behaviors exhibited by Neural Networks (NNs) depending on the scale of initialization.
Specifically, in a particular regime, NNs trained with gradient descent can be viewed as a kernel regression predictor known as the Neural Tangent Kernel (NTK) ~\cite{Jacot2018Neural,Du2018Gradient,Huang2019Dynamics,Yuqing2022ResNet}, and   Chizat et al.~\cite{chizat2019lazy} identify it as the lazy training regime in which the parameters of   overparameterized NNs trained with gradient based methods hardly varies.  However, under a  different scaling, the Gradient Flow~(GF) of NN shows highly nonlinear features and a mean-field analysis~\cite{mei2018mean,rotskoff2018parameters,chizat2018global,sirignano2020mean}  has been established for infinitely wide two-layer networks  to analyze its behavior.  Additionally, small initialization is proven to give rise to condensation~\cite{maennel2018gradient,luo2021phase,zhou2021towards,zhou2022empirical}, a phenomenon where the weight vectors of NNs concentrate on isolated orientations during the training process. This is significant as NNs with condensed weight vectors are equivalent to  ``smaller'' NNs with fewer parameters, as revealed by the embedding principle (the loss landscape of a DNN ``contains'' all the critical points of all the narrower DNNs~\cite{zhang2021embedding,zhang2022embedding}),  thus reducing the complexity of the output functions of NNs. As the generalization error can be bounded in terms of the complexity~\cite{bartlett2002rademacher}, NNs with condensed parameters tend to possess  better generalization abilities. In addition, the study of the embedding principle also found the number of the descent directions in a condensed large network is no less than that of the equivalent small effective network, which may lead to easier training of a large network~\cite{zhang2021embedding,zhang2022embedding}. 

Taken together, identifying the regime of condensation and understanding the mechanism of condensation are important to understand the non-linear training of neural networks. Our contributions can be categorized into two aspects.

\begin{figure}[ht]
    \includegraphics[width=\textwidth]{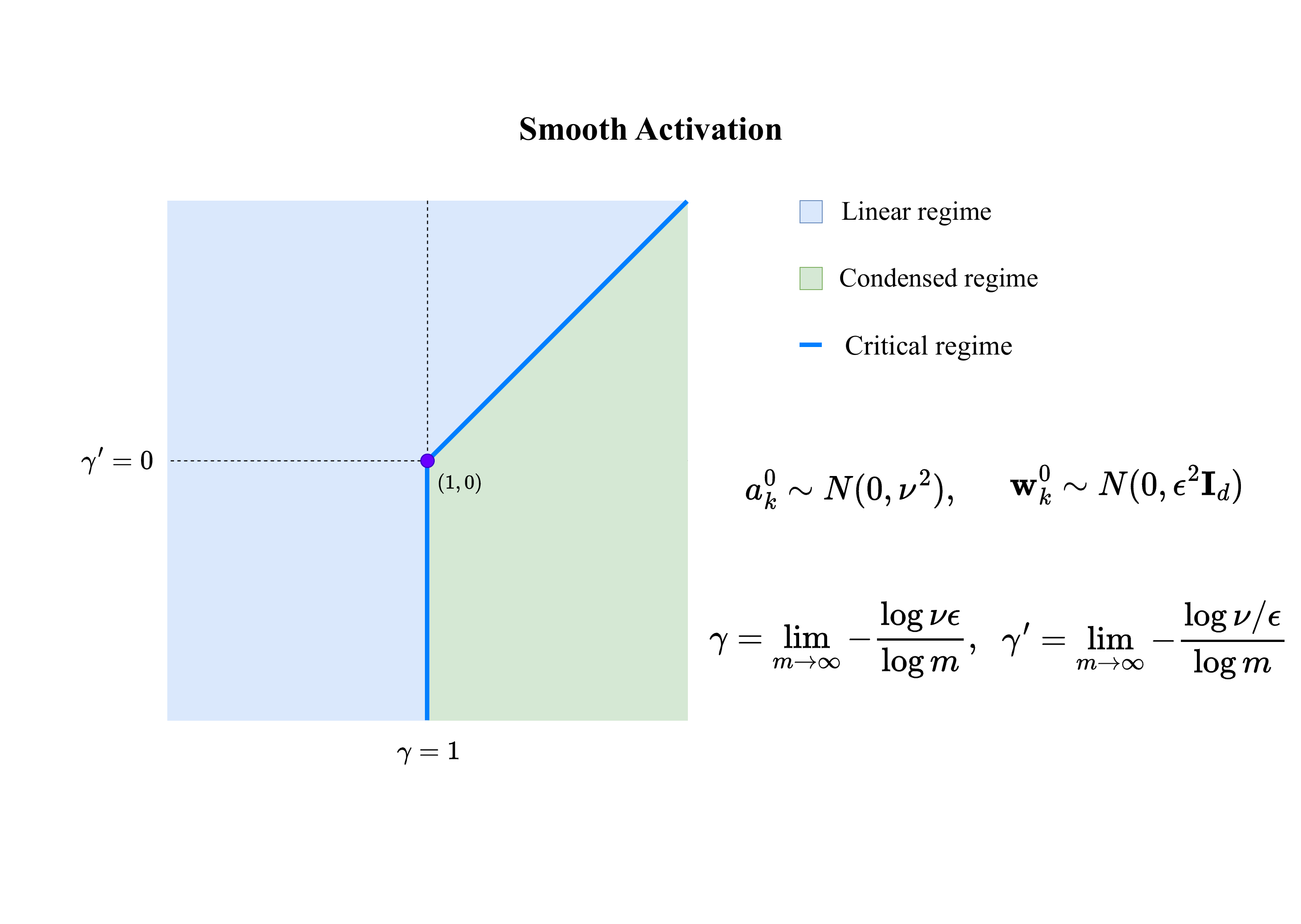}
    \caption{Phase diagram of two-layer NNs.}
    \label{fig:phase-diagram}
\end{figure}
 Firstly, we established the phase diagram of initial condensation for two-layer neural networks (NNs) with a wide class of smooth activation functions, as illustrated \Cref{fig:phase-diagram}.  Note that the phase diagram drawn in \cite{luo2021phase} is only for two-layer wide ReLU networks and the phase diagram in \cite{zhou2022empirical} is empirical for three-layer wide ReLU networks. The phase diagram of a two-layer neural network refers to a graphical representation of the dynamical behavior of the network as a function of its initialization scales. In this diagram, different regions correspond to different types of behaviors exhibited by NNs, such as the linear regime, where the network behaves like a linear model, and the condensed regime, where the network exhibits the initial condensation phenomenon. 
 
Secondly, we reveal the mechanism of initial condensation for two-layer NNs and identify the directions towards which the weight parameters condense. There has been a flurry of recent papers  endeavor to  analyze    the mechanism underlying the condensation of NNs at the initial training stage under  small initialization ~\cite{maennel2018gradient,pellegrini2021analytic,luo2021phase,lyu2021gradient,zhou2022empirical}. For instance, Maennel et al.~\cite{maennel2018gradient} uncovered that for two-layer ReLU
NNs, the  GF limits the weight vectors to a certain number of directions depending sorely on  the input data. Zhou et al.~\cite{zhou2022empirical} showed empirically  that condensation is a common feature in non-linear training regime for three-layer ReLU NNs. Theoretically, Maennel et al.~\cite{maennel2018gradient} argued that GF prefers ``simple'' functions over  ``complex'' ones, and Zhou et. al.~\cite{zhou2021towards} demonstrated that the maximal number of condensed orientations at initial training stage is twice the multiplicity~(\Cref{def}) of the activation function. However, these proofs are heuristic as they do not account for the dynamics of parameters. Pellegrini and Biroli~\cite{pellegrini2021analytic} derived a  mean-field model  demonstrating that two-layer ReLU NNs, when trained with hinge loss and infinite data, lead to a linear classifier. Nonetheless, their analysis does not illustrate how the initial condensation depends on the scale of initialization and does not specify which directions NNs condense on.


 The organization of the paper is listed  as follows. In \Cref{section....RelatedWorks}, we  discuss some related works. In \Cref{section....Preliminaries},  we give  some preliminary introduction to our problems.  
 In \Cref{section....MainResults}, we state our main results and show some  empirical evidence.  In \Cref{section....Technique}, we give out the outline of proofs for our main results, and conclusions are drawn in \Cref{section....Conclusion}. All the details of the proof are deferred to the Appendix.
\section{Related Works}\label{section....RelatedWorks}
There has been a rich literature on   the choice of initialization schemes   in order to
facilitate neural network training~\cite{glorot2010understanding,He2016Deep,mei2018mean,sirignano2020mean}, and most of the work identified the width $m$ as a hyperparameter, where the kernel regime is reached when the width grows towards infinity~\cite{Jacot2018Neural,yang2019scaling,Du2018Gradienta}. 
However, with the introduction of lazy training by Chizat et al.~\cite{chizat2019lazy},  instead of the width $m$, one shall take the initialization scale   as the relevant hyperparameter. The lazy training refers to the phenomenon in which a heavily over-parameterized NN trained with gradient-based methods could converge exponentially fast to zero training loss  with its parameters hardly varying, and  such   phenomenon can be observed in any non-convex model  accompanied by the choice of an appropriate
scaling  factor of the initialization. Follow-up
work by Woodworth et al.~\cite{woodworth2020kernel}   focus on
how the scale of initialization acts as a controlling quantity for the transition  between two very different
regimes, namely the kernel regime and the rich regime, for the matrix factorization problems. As for two-layer ReLU NNs, the  phase digram in Luo et al.~\cite{luo2021phase}  identified three regimes,  namely the   {linear} regime, the {critical} regime and the {condensed} regime, based on the relative change of input weights as the width $m$ approaches infinity. In summary, the selection  of  appropriate initialization scales plays a crucial role in the training of NNs. 

Several theoretical works studying the dynamical behavior of NNs with small initialization can be connected to implicit regularization effect provided by the weight initialization schemes,  and the condensation phenomenon has also been studied under different names. Ji and Telgarsky~\cite{jiTelgarsky2018align} analyzed the implicit regularization of GF on deep linear
networks  and observed the matrix alignment phenomena, i.e., weight matrices belonging to different layers share the same direction.  
The weight quantization
effect~\cite{maennel2018gradient} in training two-layer ReLU NNs with small initialization is really the condensation phenomenon in disguise, and 
so is the case for the weight cluster effect~\cite{brutzkus2019larger} in  learning the  MNIST task for a three-layer CNN. Luo et al.~\cite{luo2021phase} focused on how the   condensation phenomenon can be clearly detected by the choice of initialization schemes, but they did not show the reason behind it. Zhang et al.~\cite{zhang2021understanding,zhang2021embedding} proposed a general Embedding Principle of loss landscape of  DNNs,  showing  that a larger DNN can
experience critical points with condensed parameter, and its output is the same as that of a much
smaller  DNN, but their analysis did not involve its dynamical behavior. Zhou et al.~\cite{zhou2021towards} presented a theory for the initial direction towards which the weight vector  condenses, yet it is 
far from satisfactory.

\section{Preliminaries}\label{section....Preliminaries}
\subsection{Notations}\label{subsection...Notations}
We begin this section by introducing some notations that will be used in the rest of this paper.  We set $n$ for the number of input samples and $m$ for the width of the neural network.
We set   $\fN(\vmu, \Sigma)$ as the normal distribution with mean $\vmu$ and covariance $\Sigma$. 
We let $[n]=\{1,2, \ldots, n\}$.
We denote vector $L^2$ norm as $\Norm{\cdot}_2$, vector or function $L_{\infty}$ norm as $\Norm{\cdot}_{\infty}$, matrix spectral~(operator) norm as $\Norm{\cdot}_{2\to 2}$,   matrix infinity norm as $\Norm{\cdot}_{\infty\to\infty}$, and  matrix Frobenius norm as $\Norm{\cdot}_{\mathrm{F}}.$
 For a matrix $\mathbf{A}$, we use $\mathbf{A}_{i, j}$ to denote its $(i, j)$-th entry. We will also use $\mathbf{A}_{i, \text { : }}$ to denote the $i$-th row vector of $\mathbf{A}$ and define $\mathbf{A}_{i, j: k}=$ $\left[\mathbf{A}_{i, j}, \mathbf{A}_{i, j+1}, \cdots, \mathbf{A}_{i, k}\right]^\T$ as part of the vector. Similarly, $\mathbf{A}_{:, i}$ is the $i$-th column vector and $\mathbf{A}_{j: k, i}$ is a part of the $i$-th column vector.
For a semi-positive-definite  matrix  $\mA,$ we denote its smallest eigenvalue by $\lambda_{\min}(\mA),$ and correspondingly,  its largest eigenvalue by $\lambda_{\max}(\mA)$.
We use $\fO(\cdot)$ and $\Omega(\cdot)$ for the standard Big-O and Big-Omega notations. 
We finally denote the set of continuous functions $f(\cdot):\sR\to\sR$ possessing   continuous derivatives of order up to and including $r$ by $\fC^{r}(\sR)$,  the set of analytic functions $f(\cdot):\sR\to\sR$   by $\fC^{\omega}(\sR)$, and $\left<\cdot,\cdot\right>$ for standard inner product between two vectors.

\subsection{Problem Setup}\label{subsection....ProblemSetup}

We use almost the same settings in Luo et al.~\cite{luo2021phase} by  starting with the original model
\begin{equation}
    f_{\vtheta}(\vx) = \sum_{k=1}^{m}a_k\sigma(\vw_k^{\T}\vx),
\end{equation}
whose parameters $\vtheta^0:=\mathrm{vec}(\vtheta_a^0,\vtheta_{\vw}^0)$ are initialized by
\begin{equation}
    a_k^0\sim \fN(0, \nu^2), \quad \vw_k^0\sim \fN(\vzero, \eps^2 \mI_d),
\end{equation}
and the empirical risk is
\begin{equation}
    \RS(\vtheta)=\frac{1}{2n}\sum_{i=1}^n {(f_{\vtheta}(\vx_i)-y_i)}^2.
\end{equation}
Then the training dynamics based on gradient descent~(GD) at the continuous limit obeys the following gradient flow precisely reads: For $k\in[m]$,  
\begin{equation}\label{eq...text...OriginalDynamics}
\begin{aligned}
\frac{\D a_k}{\D t}& = -\frac{1}{n}\sum_{i=1}^n \left(\sum_{k'=1}^{m}{a}_{k'}{\sigma({\vw}_{k'}^{\T}\vx_i)}-y_i\right)\sigma(\vw_k^{\T}\vx_i),\\
\frac{\D \vw_k}{\D t}& = -\frac{1}{n}\sum_{i=1}^n \left(\sum_{k'=1}^{m}{a}_{k'}{\sigma({\vw}_{k'}^{\T}\vx_i)}-y_i\right)a_k\sigma^{(1)}(\vw_k^{\T}\vx_i)\vx_i.
\end{aligned}
\end{equation}
We identify the parameters $\vtheta_a:=\mathrm{vec}(\{a_k\}_{k=1}^{m})$  and $\vtheta_{\vw}:=\mathrm{vec}(\{\vw_k\}_{k=1}^{m})$ as variables of order one by setting 
\[ 
 a_k=\nu \Bar{a}_k,\quad   \vw_k=\eps \Bar{\vw}_k,
\]
then the rescaled dynamics can be written as 
\begin{equation}
\begin{aligned}
\nu\frac{\D \bar{a}_k}{\D {t}}
& = -\frac{1}{n}\sum_{i=1}^n \left( \sum_{k'=1}^{m}\nu\eps\bar{a}_{k'}\frac{\sigma(\eps\bar{\vw}_{k'}^{\T}\vx_i)}{\eps}-y_i\right)\eps\frac{\sigma(\eps\bar{\vw}_k^\T\vx_i)}{\eps} , \\
\eps\frac{\D \bar{\vw}_k}{\D {t}}
& = -\frac{1}{n}\sum_{i=1}^n\left( \sum_{k'=1}^{m}\nu\eps\bar{a}_{k'}\frac{\sigma(\eps\bar{\vw}_{k'}^{\T}\vx_i)}{\eps}-y_i\right)\nu\bar{a}_k\sigma^{(1)}(\eps\bar{\vw}_k^\T\vx_i)\vx_i. 
\end{aligned}     
\end{equation}
For the case where $\eps\ll 1$ and $\eps\gg 1$, the expressions $\frac{\sigma(\eps\bar{\vw}_k^\T\vx_i)}{\eps}$ and ${\sigma^{(1)}(\eps\bar{\vw}_k^\T\vx_i)}$ are hard to handle   at first glance. However, in the case where $\eps\ll 1$, under the condition~(\Cref{Assumption....ActivationFunctions}) that $\sigma(0)=0$ and $\sigma^{(1)}(0)=1$,  we obtain that 
\[
\frac{\sigma(\eps\bar{\vw}_k^\T\vx_i)}{\eps}\approx \bar{\vw}_k^\T\vx_i, \quad {\sigma^{(1)}(\eps\bar{\vw}_k^\T\vx_i)}\approx 1,
\]
hence $\frac{\sigma(\eps\bar{\vw}_k^\T\vx_i)}{\eps}$ and ${\sigma^{(1)}(\eps\bar{\vw}_k^\T\vx_i)}$ are  of order one. 

In the case where $\eps\gg 1$,   under the condition~(\Cref{Assumption....ActivationFunctions...NTK}) that 
\[   \lim_{x\to-\infty}{\sigma^{(1)}(x)}=a,\quad\lim_{x\to+\infty}{\sigma^{(1)}(x)}=b, \] 
we obtain that 
\[
\frac{\sigma(\eps\bar{\vw}_k^\T\vx_i)}{\eps}\approx \sigma^{(1)}(\eps\bar{\vw}_k^\T\vx_i),
\]
hence $\frac{\sigma(\eps\bar{\vw}_k^\T\vx_i)}{\eps}$ and ${\sigma^{(1)}(\eps\bar{\vw}_k^\T\vx_i)}$ are also of order one. Under these two aforementioned  conditions, $\sigma(\cdot)$ acts like a linear activation in the case where $\eps\ll 1$, and  acts like a leaky-ReLU activation activation in the case where $\eps\gg 1$, both of which are homogeneous functions.
Hence the above dynamics can be simplified into
\begin{equation}\label{eq...text...Prelim...ProbSetup...1}
\begin{aligned}
\frac{\D \bar{a}_k}{\D {t}}& = -\frac{1}{n}\sum_{i=1}^n \left( \sum_{k'=1}^{m}\nu\eps\bar{a}_{k'}\frac{\sigma(\eps\bar{\vw}_{k'}^{\T}\vx_i)}{\eps}-y_i\right) \frac{\eps}{\nu} \frac{\sigma(\eps\bar{\vw}_k^\T\vx_i)}{\eps}, \\
\frac{\D \bar{\vw}_k}{\D {t}}& = -\frac{1}{n}\sum_{i=1}^n\left( \sum_{k'=1}^{m}\nu\eps\bar{a}_{k'}\frac{\sigma(\eps\bar{\vw}_{k'}^{\T}\vx_i)}{\eps}-y_i\right) \frac{\nu}{\eps} \bar{a}_k\sigma^{(1)}(\eps\bar{\vw}_k^\T\vx_i)\vx_i.
\end{aligned}     
\end{equation}
We  hereby introduce two scaling parameters
\begin{equation}
    \kappa :=  {\nu\eps}, \quad \kappa' :=\frac{\nu}{\eps},
\end{equation}
then the   dynamics \eqref{eq...text...Prelim...ProbSetup...1} can be written as a {\emph{normalized}} flow
\begin{equation}\label{eq...text...Prelim...ProbSetup...NormalizedDynamics}
\begin{aligned}
\frac{\D \bar{a}_k}{\D {t}}& = -\frac{1}{n}\sum_{i=1}^n \left( \sum_{k'=1}^{m}\kappa\bar{a}_{k'}\frac{\sigma(\eps\bar{\vw}_{k'}^{\T}\vx_i)}{\eps}-y_i\right)  \frac{1}{\kappa'} \frac{\sigma(\eps\bar{\vw}_k^\T\vx_i)}{\eps} ,               \\
\frac{\D \bar{\vw}_k}{\D {t}}& = -\frac{1}{n}\sum_{i=1}^n\left( \sum_{k'=1}^{m}\kappa\bar{a}_{k'}\frac{\sigma(\eps\bar{\vw}_{k'}^{\T}\vx_i)}{\eps}-y_i\right)  \kappa' \bar{a}_k\sigma^{(1)}(\eps\bar{\vw}_k^\T\vx_i)\vx_i.
\end{aligned}     
\end{equation}
with the following initialization
\begin{equation}
\bar{a}_k^0\sim \fN(0,1), \quad \bar{\vw}_k^0\sim \fN(\vzero,\mI_d).
\end{equation}
In the following discussion throughout this paper, we always refer to this rescaled model~\eqref{eq...text...Prelim...ProbSetup...NormalizedDynamics}  and drop all the ``bar''s of $\{a_k\}_{k=1}^m$ and  $\{\vw_k\}_{k=1}^m$  for  notational simplicity.

As  $\kappa$ and $\kappa'$ are always in specific power-law relations to the width $m$, we introduce two independent coordinates
\begin{equation}
\gamma:=\lim_{m\to\infty}-\frac{\log \kappa}{\log m}, \quad \gamma':=\lim_{m\to\infty}-\frac{\log\kappa'}{\log m},
\end{equation}
which meet all the guiding principles \cite{luo2021phase} for finding the coordinates of a phase diagram.
  
Before we end this section,  we list out some commonly-used initialization methods   with their scaling parameters as shown in  \Cref{table}. 
\begin{table}
\begin{tabular}{ |p{3cm}|p{1cm}p{1cm}p{2cm}p{1.7cm}p{0.3cm}p{0.5cm} | }
 \hline
  Name& $\nu$  &$\eps$&$\kappa~(\nu\eps)$& $\kappa'~({\nu}/{\eps})$&$\gamma$   &$\gamma'$\\
 \hline
 LeCun~\cite{lecun2012efficient}   & $\sqrt{\frac{1}{m}} $  & $\sqrt{\frac{1}{d}}$  &   $\sqrt{\frac{1}{md}}$& $\sqrt{\frac{d}{m}}$&$\frac{1}{2}$&$\frac{1}{2}$\\
He~\cite{he2015delving}   & $\sqrt{\frac{2}{m}} $  & $\sqrt{\frac{2}{d}}$  &   $\sqrt{\frac{4}{md}}$& $\sqrt{\frac{d}{m}}$&$\frac{1}{2}$&$\frac{1}{2}$\\
 Xavier~\cite{glorot2010understanding} & $\sqrt{\frac{2}{m+1}} $  & $\sqrt{\frac{2}{m+d}}$  &   $\sqrt{\frac{4}{(m+1)(d+1)}}$& $\sqrt{\frac{m+d}{m+1}}$&$1$&$0$\\
Huang~\cite{Huang2019Dynamics} & $1 $  & $\sqrt{\frac{1}{m}}$  &   $\sqrt{\frac{1}{m}}$& $\sqrt{m}$&$\frac{1}{2}$&$-\frac{1}{2}$\\
\hline 
\end{tabular}
\caption{\label{table}Initialization methods with their scaling parameters}
\end{table}
\section{Main Results}\label{section....MainResults}
\subsection{Activation function and input }
In this part, we shall impose some   technical conditions  on the activation function and input samples.
We start with a   technical condition~\cite[Definition 1]{zhou2021towards}  on the activation function $\sigma(\cdot)$ 
\begin{definition}[Multiplicity $p$]\label{def}
   $\sigma(\cdot):\sR\to\sR$ has multiplicity $p$
    if there exists an integer $p\geq 1$, such that  for all $0\leq s\leq p-1$, the $s$-th order derivative satisfies $\sigma^{(s)}(0)=0$, and $\sigma^{(p)}(0) \neq 0$.
\end{definition}
\noindent We   list out some examples of activation functions with different multiplicity.
\begin{remark}
~\\
\begin{itemize}
\item $tanh(x):=\frac{\exp(x)-\exp(-x)}{\exp(x)+\exp(-x)}$ is with multiplicity $p=1$;\\
\item $SiLU(x):=\frac{x}{1+ \exp(-x)}$ is with multiplicity $p=1$;\\
\item $xtanh(x):=\frac{x\exp(x)-x\exp(-x)}{\exp(x)+\exp(-x)}$ is with multiplicity $p=2$.
\end{itemize}
\end{remark}
\begin{assumption}[Multiplicity $1$]\label{Assumption....ActivationFunctions}
The activation function $\sigma\in\fC^2(\sR)$, and there exists a universal constant $C_L>0$, such that   its first and second  derivatives   satisfy 
\begin{equation} 
   \Norm{\sigma^{(1)}(\cdot)}_{\infty}\leq C_L,\quad \Norm{\sigma^{(2)}(\cdot)}_{\infty}\leq C_L.
\end{equation}
Moreover,
\begin{equation} 
\sigma(0)=0,\quad \sigma^{(1)}(0)=1.
\end{equation}
\end{assumption}
\begin{remark}
We remark that $\sigma$ has multiplicity $1$.
$\sigma^{(1)}(0)=1$   can be replaced by  $\sigma^{(1)}(0)\neq 0$, and we set   $\sigma^{(1)}(0)=1$  for simplicity, and it can be easily satisfied by replacing the original activation $\sigma(\cdot)$ with $\frac{\sigma(\cdot)}{\sigma^{(1)}(0)}$.

We note that \Cref{Assumption....ActivationFunctions} can be satisfied by using the tanh activation:
\[
 \sigma(x)=\frac{\exp(x)-\exp(-x)}{\exp(x)+\exp(-x)},
\]
and the scaled SiLU activation
\[
 \sigma(x)=\frac{2x}{1+ \exp(-x)}.
\]
\end{remark}
\begin{assumption}\label{Assumption....ActivationFunctions...NTK}
The activation function $\sigma\in\fC^{\omega}(\sR)$ and is not a polynomial function, also 
its function value at $0$ satisfy 
\begin{equation}\label{eq for assumption...the uniform constant on activation function +derivatives}
     {\sigma(0)}=0,
\end{equation}
also   there exists a universal constant $C_L>0$,   such that  its first  and second derivatives  satisfy 
\begin{equation}\label{eq for assumption...the uniform constant on activation function +derivatives+more}
{\sigma^{(1)}(0)}=1,\quad\Norm{\sigma^{(1)}(\cdot)}_{\infty}\leq C_L,\quad\Norm{\sigma^{(2)}(\cdot)}_{\infty}\leq C_L.
\end{equation}
Moreover, 
\begin{equation}\label{eq for assumption...limitexists}
\lim_{x\to-\infty}{\sigma^{(1)}(x)}=a,\quad\lim_{x\to+\infty}{\sigma^{(1)}(x)}=b, 
\end{equation}
 and $a\neq b$.
\end{assumption}
\begin{remark}

We note that \Cref{Assumption....ActivationFunctions...NTK} can be satisfied by by using the scaled SiLU activation: $$\sigma(x)=\frac{2x}{1+\exp(-x)},$$
where $a=0$ and $b=2$.
 
Some other functions also satisfy this assumption, for instance, the modified scaled softplus activation: $$\sigma(x)=2\left(\log(1+\exp(x))-\log 2\right),$$where $a=0$ and $b=2$.
\end{remark}
\begin{assumption}[Non-degenerate data]\label{assumption...GenericData}
The training inputs and labels  $\fS =\{(\vx_i,y_i)\}_{i=1}^n$ satisfy  that   there exists a universal constant $c>0$, such that for all $i\in[n]$, 
\[  \frac{1}{c}\leq \Norm{\vx_{i}}_2, \quad\Abs{y_{i}}\leq c,\] 
and
\begin{equation} 
\sum_{i=1}^n y_i {\vx}_i \neq \mathbf{0}.
\end{equation}
We denote by
\begin{equation} \label{eq...assump...GenericData...Z}
\vz:=
\frac{1}{n}\sum_{i=1}^n y_i {\vx}_i,  
\end{equation}
and assume further that for some  universal constant $c>0$, the following holds
\begin{equation} \label{eq...assump...GenericData...Z...}
\frac{1}{c}\leq \Norm{\vz}_2\leq c,
\end{equation}
and its unit vector  
\begin{equation} \label{eq...assump...GenericData...UnitZ}
\hat{\vz}:=\frac{ \sum_{i=1}^n y_i {\vx}_i}{\Norm{ \sum_{i=1}^n y_i {\vx}_i}_2}. 
\end{equation}
\end{assumption}
\begin{assumption}\label{Assump...Unparallel}
The training inputs and labels  $\fS =\{(\vx_i,y_i)\}_{i=1}^n$  satisfy that there exists a universal constant $c>0$, such that  for all $i\in[n]$, 
\[\frac{1}{c}\leq\Norm{\vx_{i}}_2, \quad\Abs{y_{i}}\leq c,\] 
and  all training inputs are non-parallel with each other, i.e.,~for any $i\neq j$ and~$i, j\in[n]$, \[\vx_{i}\nparallel \vx_{j}.\] 
\end{assumption}
We remark that the requirements in  \Cref{assumption...GenericData} are easier to meet compared with \Cref{Assump...Unparallel}, and both assumptions impose  the input sample $\fS=\{(\vx_i,y_i)\}_{i=1}^n$ to be of order one.
\begin{assumption}\label{assump...LimitExistence}
The following limit exists
\begin{equation} \label{eq...assump...definition...LimitExistence}
{\gamma}_1:=\lim_{m\to\infty} -\frac{\log \nu}{\log m},\quad {\gamma}_2:=\lim_{m\to\infty} -\frac{\log \eps}{\log m},
\end{equation}
then by definition
\[
\gamma=\lim_{m\to\infty} -\frac{\log \nu\eps}{\log m}={\gamma}_1+\gamma_2,\quad\gamma'=\lim_{m\to\infty} -\frac{\log \frac{\nu}{\eps}}{\log m}={\gamma}_1-\gamma_2.
\]
\end{assumption}
\subsection{Regime Characterization   at Initial Stage}
Before presenting our theory that establishes a consistent boundary to separate the diagram into two distinct areas, namely the linear regime area and the condensed regime area, we introduce a quantity that has proven to be valuable in the analysis  of NNs.

It is known that the  output of a  two-layer NN  is linear with respect to $\vtheta_a$, hence  if the set of parameter $\vtheta_{\vw}$ remain stuck to its initialization throughout the whole training process, then the training dynamics of a  two-layer NN   can be linearized around the initialization. In the phase diagram, the linear regime area precisely corresponds to the region where the output function of  a  two-layer NN  can be well approximated by its linearized model, i.e., in the linear regime area, the following holds  
\begin{equation}\label{eq...text...Taylor}
\begin{aligned}
 f_{\vtheta}(\vx)&\approx f\left(\vx, \vtheta(0)\right)+\left<\nabla_{\vtheta_a} f\left(\vx, \vtheta(0)\right),   \vtheta_a(t)- \vtheta_a(0) \right> \\
&~~~~~~~~~~~~~~~~~~+\left<\nabla_{\vtheta_{\vw}} f\left(\vx, \vtheta(0)\right),   \vtheta_{\vw}(t)- \vtheta_{\vw}(0) \right>.
\end{aligned}
\end{equation}
In general, this linear approximation  holds only when $\vtheta_{\vw}(t)$ remains within a small neighbourhood of $\vtheta_{\vw}(0)$.
Since the size of this neighbourhood  scales with $\norm{\vtheta_{\vw}(0)}_{2}$, therefore we use the following quantity as an indicator of how far $\vtheta_{\vw}(t)$ deviates away from $\vtheta_{\vw}(0)$ throughout the  training process 
\begin{equation}\label{eq...text...MainResults...RegimeCharac}
    \sup\limits_{t\in[0,+\infty)}\mathrm{RD}(\vtheta_{\vw}(t))=\frac{\Norm{\vtheta_{\vw}(t)-\vtheta_{\vw}(0)}_{2}}{\norm{\vtheta_{\vw}(0)}_{2}}.
\end{equation}

We demonstrate that as $m\to\infty$,  under suitable choice of the initialization scales~(the blue area in \Cref{fig:phase-diagram}), the NN training dynamics fall  into the linear regime~(\Cref{thm..linear}), and for large enough $m$,
\[
\sup\limits_{t\in[0,+\infty)}\mathrm{RD}(\vtheta_{\vw}(t))\to 0.
\] 

We also demonstrate that under some other choices of the initialization scales~(the green area in \Cref{fig:phase-diagram}), the NN training dynamics   fall into 
the condensed regime~(\Cref{thm..CondensedRegime}),  where 
\[
\sup\limits_{t\in[0,+\infty)}\mathrm{RD}(\vtheta_{\vw}(t))\to +\infty,
\] 
and  the phenomenon  of condensation   can be  observed, and $\vtheta_{\vw}$ condense toward the direction of $\vz$.  We observe that in both cases, as $\vtheta_{\vw}$ deviates   far away from initialization, the approximation \eqref{eq...text...Taylor}  fails, and NN training dynamics is essentially nonlinear with respect to $\vtheta_{\vw}$.  

Moreover, under the remaining choice of the initialization scales~(the solid blue line in \Cref{fig:phase-diagram}),   the NN training dynamics   fall into the critical regime area, and we conjecture that 
\[
 \sup\limits_{t\in[0,+\infty)}\mathrm{RD}(\vtheta_{\vw}(t))\to \fO(1),
\] 
whose study   is beyond the scope of this paper.
\begin{figure}[ht]
\centering
    \includegraphics[width=0.75\textwidth]{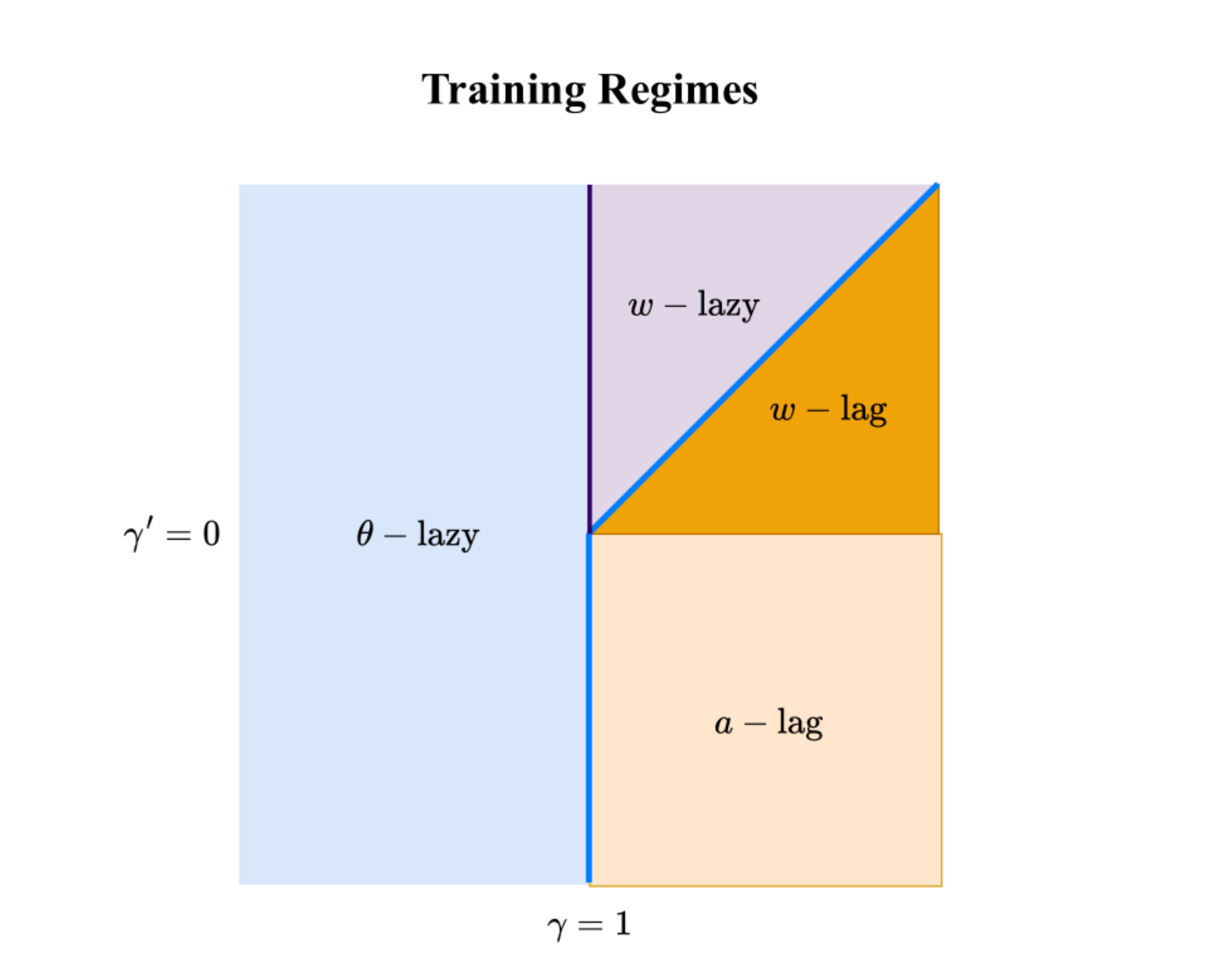}
    \caption{Different training regimes in the phase diagram.}
    \label{fig:Area}
\end{figure}
 \begin{theorem}[Linear regime]\label{thm..linear}
Given any $\delta\in(0,1)$, under \Cref{Assumption....ActivationFunctions...NTK}, \Cref{Assump...Unparallel} and \Cref{assump...LimitExistence}, if $\gamma< 1$ or $\gamma'> \gamma-1$, then
 with probability at least $1-\delta$ over the choice of $\vtheta^0$,
\begin{equation}
\lim_{m\to\infty}\sup\limits_{t\in[0,+\infty)}\frac{\norm{\vtheta_{\vw}(t)-\vtheta_{\vw}(0)}_2}{\norm{\vtheta_{\vw}(0)}_2}=0. 
\end{equation}  
\end{theorem}
\begin{remark}
The linear regime area is split into two parts, one is termed the  $\vtheta$-lazy area~(blue area in \Cref{fig:Area}), where  $\gamma<1$, the other is termed the  $\vw$-lazy area~(pink area in \Cref{fig:Area}), where  $\gamma\geq 1$ and $ \gamma'>\gamma-1>0$. 

In the $\vtheta$-lazy area, the following relation holds 
\begin{equation}\label{eq...ThetaLazy}   
\lim_{m\to\infty}\sup\limits_{t\in[0,+\infty)}\frac{\norm{\vtheta(t)-\vtheta(0)}_2}{\norm{\vtheta(0)}_2}=0,
\end{equation}    
whose detailed reasoning can be found in \Cref{subsec...thetalazy}, and in the $\vw$-lazy area, relation \eqref{eq...ThetaLazy} does not hold.
\end{remark}
\begin{theorem}[Condensed regime]\label{thm..CondensedRegime}
Given any $\delta\in(0,1)$, under \Cref{Assumption....ActivationFunctions}, \Cref{assumption...GenericData} and \Cref{assump...LimitExistence}, if $\gamma> 1$ and $\gamma'<\gamma-1$, 
then	with probability at least $1-\delta$ over the choice of $\vtheta^0$,
there exists $T>0$, such that 
\begin{equation}
\lim_{m\to\infty}\sup\limits_{t\in[0,T]}\frac{\norm{\vtheta_{\vw}(t)-\vtheta_{\vw}(0)}_2}{\norm{\vtheta_{\vw}(0)}_2}=+\infty,
\end{equation}
and
\begin{equation}
\lim_{m\to\infty} \sup\limits_{t\in[0,T]}\frac{\Norm{\vtheta_{\vw, \vz}(t) }_2}{\Norm{\vtheta_{\vw}(t)}_2} =1,
\end{equation}
where $\vtheta_{\vw, \vz}(t):=  \left[\left<\vw_1, \hat{\vz}\right>, \left<\vw_2, \hat{\vz}\right>, \cdots, \left<\vw_m, \hat{\vz}\right>\right]^\T$.
\end{theorem}

\begin{remark}\label{rmk...TimeT}
The condensed regime area is split into two parts, one is termed the  $\vw$-lag area~(orange area in \Cref{fig:Area}), where  $\gamma> 1$ and $0\leq\gamma'<\gamma-1$, the other is termed the  $a$-lag area~(yellow area in \Cref{fig:Area}), where  $\gamma> 1$ and $ \gamma'<0$. 

In the $\vw$-lag regime area, as illustrated in \eqref{eq...text..techinique...pairsreduce}, $\vtheta_{\vw}$ waits for a period of time of order one until $\vtheta_{a}$  attains a magnitude that is commensurate with that of $\vtheta_{\vw}$, and
the time $T$ in \Cref{thm..CondensedRegime} satisfies that 
\begin{equation}\label{eq...rmk...logm}
 T \geq  \log\left(\frac{1}{4}\right)+{\frac{\gamma-\gamma'-1}{8}}\log(m),   
\end{equation}
and  as $m\to\infty$, $T\to\infty$. 

In the $a$-lag regime area, as illustrated in \eqref{eq...text..techinique...pairsreduce}, $\vtheta_{a}$ waits for a period of time of order one until $\vtheta_{\vw}$  attains a magnitude that is commensurate with that of $\vtheta_{a}$, and it is exactly during this  interval  of time that the phenomenon of initial condensation  can be observed.  Hence for some $\alpha>0$, the time $T$ in \Cref{thm..CondensedRegime} can be chosen  as
\begin{equation}\label{eq...rmk...-alpha}
m^{-\alpha}\leq T \leq  2m^{-\alpha},   
\end{equation}
which is  obviously of order one, see \Cref{subsection...append...condense...aLag} for more details.
\end{remark}
\begin{figure}[ht]
\centering
 \subfloat[$\gamma=1.4,~~R^2=0.983$]{
        \includegraphics[width=0.33\textwidth]{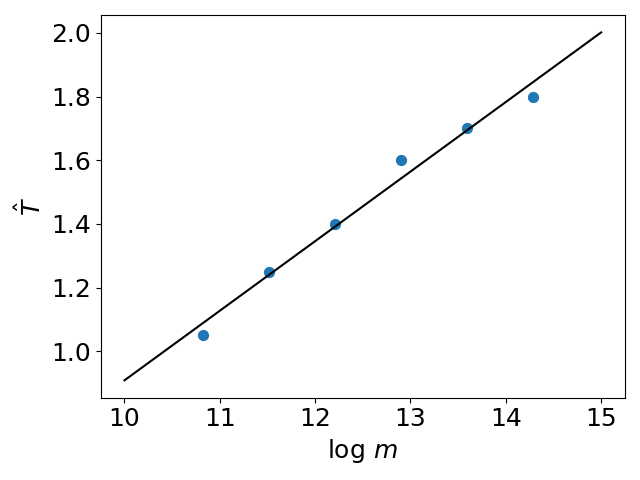}
    }
    \subfloat[$\gamma=1.6,~~R^2=0.991$]{
        \includegraphics[width=0.33\textwidth]{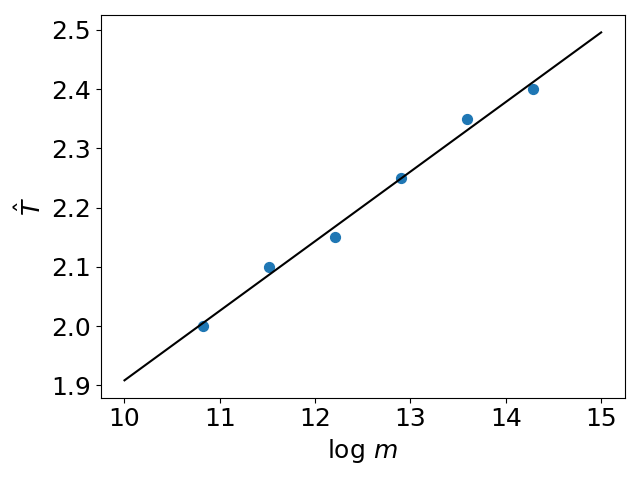}
    }
    \subfloat[$\gamma=1.8,~~R^2=0.995$]{
        \includegraphics[width=0.33\textwidth]{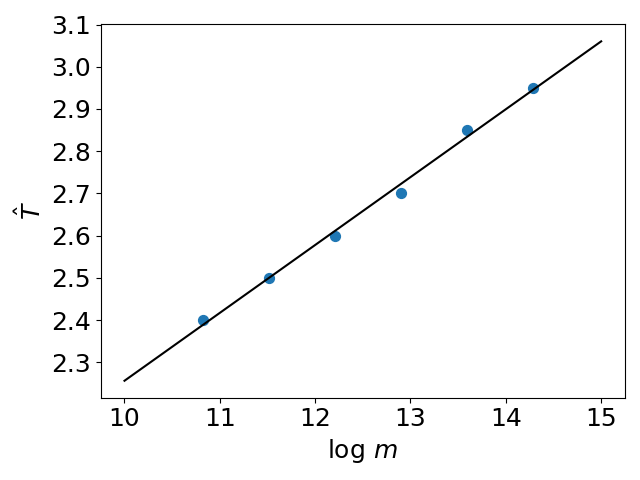}
    }
    \caption{$\hat{T}$~(ordinate) vs $\log m$~(abscissa) with different values of $\gamma$ but fixed $\gamma'=0$ for two-layer NNs with tanh activation  indicated by  blue dots. The black line is a linear fit, and $R^2$ is {\emph{the coefficient of determination}} that provides information about the goodness of fit of a linear regression. The closer $R^2$ is to $1$, the better the model fits the data.}
    \label{fig:logm}
\end{figure}
\subsection{Experimental Demonstration}
In order to distinguish between  the $\vw$-lag regime and $a$-lag regime, it is necessary to estimate the time $T$ in \Cref{rmk...TimeT}, which is also a reasonale way to empirically validate our theoretical analysis.  An empirical approximation of $T$  can be obtained by determining the time interval $\hat{T}$,  starting from the initial stage at $t=0$, up to the point at which the quantity 
$
\frac{\Norm{\vtheta_{\vw, \vz}(t) }_2}{\Norm{\vtheta_{\vw}(t)}_2}
$
reaches its climax for sufficiently large values of  $m$~($m=50000, 100000, 200000, 400000,800000,1600000$), as we are unable to run experiments at $m\to\infty$.
\subsubsection{$\vw$-lag regime}\label{subsec}
We validate the effectiveness of our estimates  by performing a simple linear regression to visualize the relation \eqref{eq...rmk...logm}, where $\hat{T}$ is set as the response variable and $\log m$ as the single independent variable. \Cref{fig:logm} shows that NNs with different values of $\gamma$ but fixed $\gamma'$  satisfy the relation \eqref{eq...rmk...logm}, thereby demonstrating the accuracy and reliability of our estimates.
\begin{figure}[ht]
\centering
 \subfloat[$\gamma=1.2,~~R^2=0.936$]{
        \includegraphics[width=0.34\textwidth]{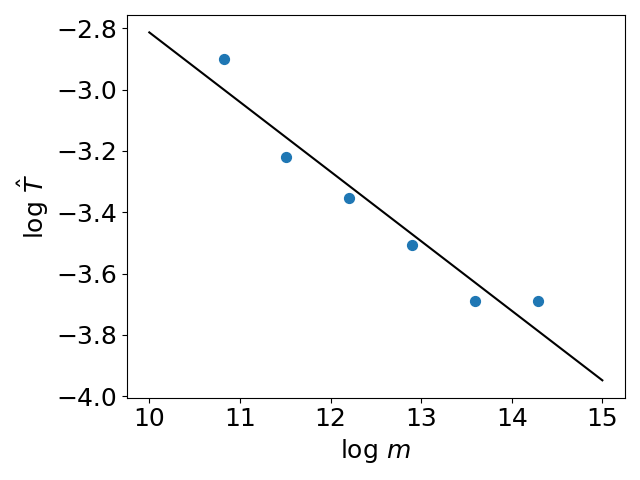}
    }
    \subfloat[$\gamma=1.4,~~R^2=0.977$]{
        \includegraphics[width=0.34\textwidth]{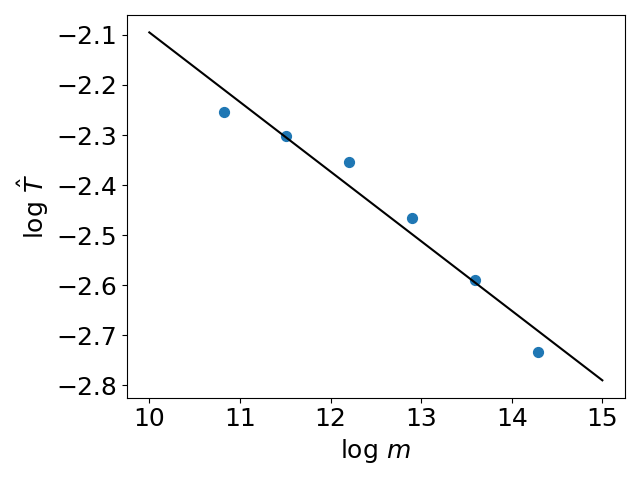}
    }
    \subfloat[$\gamma=1.6,~~R^2=0.993$]{
        \includegraphics[width=0.34\textwidth]{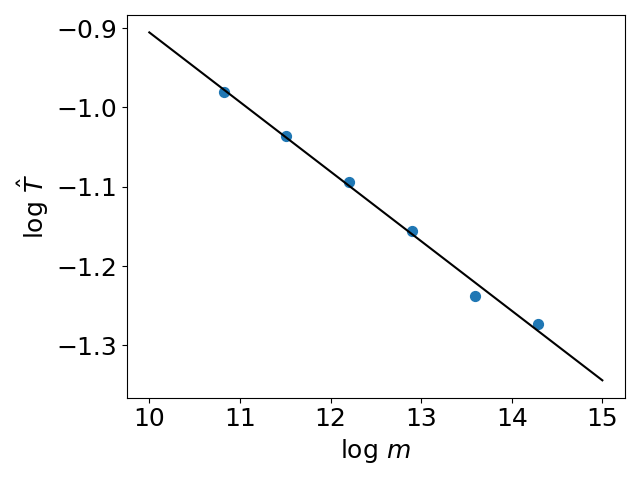}
        }
  
    \caption{$\log \hat{T}$~(ordinate) vs $\log m$~(abscissa) with different values of $\gamma$ but fixed $\gamma'=-0.4$ for two-layer NNs with tanh activation indicated by blue dots. The black line is a linear fit, and $R^2$ is the coefficient of determination.}
    \label{fig:m-alpha}
\end{figure}
\subsubsection{$a$-lag regime}
We repeat the strategy in \Cref{subsec} except that in one hand we are hereby to visualize the relation \eqref{eq...rmk...-alpha} in \Cref{fig:m-alpha}, and in the other,  $\log \hat{T}$ is set as the response variable and  $\gamma'$ is no longer $0$. We can still see a good agreement between the experimental data and its linear fitting, thus, validating the relation \eqref{eq...rmk...-alpha}.
\section{Technique Overview}\label{section....Technique}
In this part, we   describe some technical tools and present the sketch of proofs for the above two theorems. Before we proceed, a rigorous description of the updated notations and definitions is required.

We start by           a two-layer normalized NN model 
\begin{equation}
f_{\vtheta}(\vx) = \sum_{k=1}^{m}\nu\eps a_k\frac{\sigma(\eps\vw_k^{\T}\vx)}{\eps},
\end{equation}
with the normalized parameters $\vtheta^0:=\mathrm{vec}(\vtheta_a^0,\vtheta_{\vw}^0)$  initialized by
\begin{align*}
a^0_k  & :=a_k(0)\sim \fN(0,1), \\
\vw_k^0   & :=\vw_k(0)\sim \fN(\vzero,\mI_d).                             \end{align*}
For all $i\in[n]$,  we denote hereafter that  
\[
e_i :=e_i(\vtheta) :=  f_{\vtheta}(\vx_i) - y_i,   
\]
and  
\[
\ve:=\ve(\vtheta) := {[e_1(\vtheta), e_2(\vtheta), \ldots, e_n(\vtheta)]}^{\T}.
\]
Then the normalized flow reads: For all $k\in[m]$,
\begin{equation}\label{eq...text...Prelim...ProbSetup...NormalizedDynamicsRepeat}
\begin{aligned}
\frac{\D {a}_k}{\D {t}}& = -\frac{\eps}{\nu} \frac{1}{n}\sum_{i=1}^ne_i  \frac{\sigma(\eps{\vw}_k^\T\vx_i)}{\eps}, \\
\frac{\D  {\vw}_k}{\D {t}}& =- \frac{\nu}{\eps}\frac{1}{n}\sum_{i=1}^ne_i    {a}_k\sigma^{(1)}(\eps {\vw}_k^\T\vx_i)\vx_i.
\end{aligned}     
\end{equation}
\subsection{Linear Regime}
We define the normalized kernels as follows
\begin{equation}
\begin{aligned}
 k^{[a]}(\vx,\vx')&:=\frac{1}{\eps^2}\Exp_{\vw}{\sigma(\eps\vw^\T\vx)\sigma(\eps\vw^\T\vx')},                        \\
 k^{[\vw]}(\vx,\vx')&:=\Exp_{(a,\vw)}a^2\sigma^{(1)}(\eps\vw^\T\vx)\sigma^{(1)}(\eps\vw^\T\vx')\left<\vx, \vx'\right>, \\
\end{aligned}
\end{equation}
thus, the components of the Gram matrices $\mK^{[a]}$ and $\mK^{[\vw]}$ of at infinite width  respectively reads: For any $i,j\in[n]$,
\begin{equation}
\begin{aligned}
\mK^{[a]}&:=\left[K_{ij}^{[a]}\right]_{n\times n},\\
K^{[a]}_{ij}&:=k^{[a]}(\vx_i,\vx_j),            \\
\mK^{[\vw]}&:=\left[K_{ij}^{[\vw]}\right]_{n\times n},\\
K^{[\vw]}_{ij}&:=k^{[\vw]}(\vx_i,\vx_j), 
\end{aligned}
\end{equation}
we conclude that under \Cref{Assumption....ActivationFunctions...NTK} and \Cref{Assump...Unparallel},  $\mK^{[a]}$ and $\mK^{[\vw]}$ are strictly positive definite, and both of  their least eigenvalues are of order one~(\Cref{append...thm...LeastEigenGram}). 

We define the normalized Gram matrices $\mG^{[a]}(\vtheta)$, $\mG^{[\vw]}(\vtheta)$, and $\mG(\vtheta)$ for a finite width two-layer network as follows: For any $i,j\in[n]$,
\begin{equation}
\begin{aligned}
\mG^{[a]}(\vtheta)&:=\left[G_{ij}^{[a]}(\vtheta)\right]_{n\times n},\\
G^{[a]}_{ij}(\vtheta)&:=\frac{1}{m}\sum_{k=1}^m\left<\nabla_{a_k} f_{\vtheta}(\vx_i), \frac{\eps}{\nu}\nabla_{a_k} f_{\vtheta}(\vx_j)\right>\\
&=\frac{\nu^2}{ m} \frac{\eps}{\nu}\sum_{k=1}^m\sigma(\eps\vw_k^\T\vx_i)\sigma(\eps\vw_k^\T\vx_j)\\
&=\frac{\nu\eps^3}{ m}\sum_{k=1}^m\frac{1}{\eps^2}\sigma(\eps\vw_k^\T\vx_i)\sigma(\eps\vw_k^\T\vx_j),\\
\mG^{[\vw]}(\vtheta)&:=\left[G_{ij}^{[\vw]}(\vtheta)\right]_{n\times n},\\
G^{[\vw]}_{ij}(\vtheta)&:=\frac{1}{m}\sum_{k=1}^m\left<\nabla_{\vw_k} f_{\vtheta}(\vx_i), \frac{\nu}{\eps}\nabla_{\vw_k}  f_{\vtheta}(\vx_j)\right>\\
&=\frac{\nu^3\eps}{m}\sum_{k=1}^m a_k^2\sigma^{(1)}(\eps\vw_k^\T\vx_i)\sigma^{(1)}(\eps\vw_k^\T\vx_j)\left<\vx_i, \vx_j\right>, 
\end{aligned}
\end{equation}
and 
\begin{equation}
           \mG(\vtheta):=\mG^{[a]}(\vtheta)+\mG^{[\vw]}(\vtheta). 
\end{equation}
\begin{remark}
We conclude  that 
\begin{equation}
\lambda_{\min}\left(\mG^{[a]}\left(\vtheta^0\right)  \right)\sim \Omega(\nu\eps^3),\quad
        \lambda_{\min}\left(\mG^{[\vw]}\left(\vtheta^0\right)  \right)\sim \Omega(\nu^3\eps),
\end{equation}
and it has been rigorously achieved by \Cref{append...prop...LeastEigenvalueatInitial} located in \Cref{append...subsection....GramMatrices}.
\end{remark}
\noindent Finally, we obtain that
\begin{equation*}
\begin{aligned}
\frac{\D}{\D t}\RS(\vtheta)& =-\left(\sum_{k=1}^m\frac{\eps}{\nu}\left<\nabla_{a_k}\RS(\vtheta), \nabla_{a_k}\RS(\vtheta)\right>+\sum_{k=1}^m\frac{\nu}{\eps}\left<\nabla_{\vw_k}\RS(\vtheta), \nabla_{\vw_k}\RS(\vtheta)\right>\right) \\
& =-\frac{m}{n^2}\ve^{\T}\left(\mG^{[a]}(\vtheta)+\mG^{[\vw]}(\vtheta)\right)\ve. 
\end{aligned}
\end{equation*}

In the case where $\gamma<1$~($\vtheta$-lazy regime),   the following holds for all $t>0$:
\begin{align*}
 \lambda_{\min}\left(\mG^{[a]}(\vtheta(t))  \right)\geq \frac{1}{2}\nu\eps^3\lambda,\quad  \lambda_{\min}\left(\mG^{[\vw]}(\vtheta(t))  \right)\geq \frac{1}{2}\nu^3\eps\lambda,
\end{align*}
for some universal constant $\lambda>0$.
Hence,  we obtain that 
\begin{align*}
\frac{\D}{\D t}\RS(\vtheta(t))& =-\frac{m}{n^2}\ve^{\T}\left(\mG^{[a]}(\vtheta(t))+\mG^{[\vw]}(\vtheta(t))\right)\ve \\
&\leq -\frac{2m}{n}\lambda_{\min}\left(\mG(\vtheta(t))\right)\RS(\vtheta(t))\\
& \leq -\frac{m}{n}\nu^2\eps^2\left(\frac{\eps}{\nu}\lambda+\frac{\nu}{\eps}\lambda\right)\RS(\vtheta(t)),
\end{align*}
then
\begin{equation}\label{eq...text...Techinique...LossDecay}
\RS(\vtheta(t))  \leq \exp\left(-\frac{m}{n}\nu^2\eps^2\left(\frac{\eps}{\nu}\lambda+\frac{\nu}{\eps}\lambda\right)t\right)\RS(\vtheta(0)).
\end{equation}
The following relation
\begin{equation}\label{eq...text...Techinique...Theta}
\lim_{m\to\infty}\sup\limits_{t\in[0,+\infty)}\frac{\norm{\vtheta(t)-\vtheta(0)}_2}{\norm{\vtheta(0)}_2}=0,
\end{equation}    
is illustrated through an intuitive scaling analysis. Since 
\begin{equation*}
\frac{\D}{\D t}\RS(\vtheta(t))=-\frac{\eps}{\nu}\Norm{\nabla_{\vtheta_a}\RS(\vtheta(t))}_2^2 -\frac{\nu}{\eps}\Norm{\nabla_{\vtheta_{\vw}}\RS(\vtheta(t))}_2^2\sim -\frac{m}{n}\nu^2\eps^2\left(\frac{\eps}{\nu}\lambda+\frac{\nu}{\eps}\lambda\right)\RS(\vtheta(t)),
\end{equation*}
then we have that  
\begin{equation*}
 \RS(\vtheta(t))  \sim  \exp\left(-{\frac{m}{n}\nu^2\eps^2\left(\frac{\eps}{\nu}\lambda+\frac{\nu}{\eps}\lambda\right)}t\right){\RS(\vtheta(0))},
\end{equation*}
and 
\begin{align*}
\Norm{\nabla_{\vtheta_a}\RS(\vtheta(t))}_2  &
 \sim  \sqrt{\frac{m}{n}\nu^3\eps\left(\frac{\eps}{\nu}\lambda+\frac{\nu}{\eps}\lambda\right)}\sqrt{\RS(\vtheta(t))}\\
&\sim  \sqrt{\frac{m}{n}\nu^3\eps\left(\frac{\eps}{\nu}\lambda+\frac{\nu}{\eps}\lambda\right)}\exp\left(-{\frac{m}{2n}\nu^2\eps^2\left(\frac{\eps}{\nu}\lambda+\frac{\nu}{\eps}\lambda\right)}t\right)\sqrt{\RS(\vtheta(0))},\\
\Norm{\nabla_{\vtheta_{\vw}}\RS(\vtheta(t))}_2  &
 \sim  \sqrt{\frac{m}{n}\nu\eps^3\left(\frac{\eps}{\nu}\lambda+\frac{\nu}{\eps}\lambda\right)}\sqrt{\RS(\vtheta(t))}\\
&\sim  \sqrt{\frac{m}{n}\nu\eps^3\left(\frac{\eps}{\nu}\lambda+\frac{\nu}{\eps}\lambda\right)}\exp\left(-{\frac{m}{2n}\nu^2\eps^2\left(\frac{\eps}{\nu}\lambda+\frac{\nu}{\eps}\lambda\right)}t\right)\sqrt{\RS(\vtheta(0))},
\end{align*}
both holds, hence 
\begin{align*}
{\norm{\vtheta(t)-\vtheta(0)}_2}&\leq {\norm{\vtheta_a(t)-\vtheta_a(0)}_2}+{\norm{\vtheta_{\vw}(t)-\vtheta_{\vw}(0)}_2}\\
&\leq    
\frac{\eps}{\nu}\int_{0}^t\Norm{\nabla_{\vtheta_a}\RS(\vtheta(s))}_2\D s +\frac{\nu}{\eps}\int_{0}^t\Norm{\nabla_{\vtheta_{\vw}}\RS(\vtheta(s))}_2\D s\\
 &\leq   
\frac{\eps}{\nu}\int_{0}^{\infty}\Norm{\nabla_{\vtheta_a}\RS(\vtheta(s))}_2\D s +\frac{\nu}{\eps}\int_{0}^{\infty}\Norm{\nabla_{\vtheta_{\vw}}\RS(\vtheta(s))}_2\D s\\
&\lesssim    \left(\sqrt{\frac{\eps}{\nu}}+\sqrt{\frac{\nu}{\eps}}\right)\sqrt{\frac{n}{ {m} \nu^2\eps^2\left(\frac{\eps}{\nu}\lambda+\frac{\nu}{\eps}\lambda\right)}}\sqrt{\RS(\vtheta(0))},
\end{align*}
and 
\[
\Norm{\vtheta(0)}_2\sim \sqrt{m},
\]
hence 
\begin{equation}
\begin{aligned}
\frac{{\norm{\vtheta(t)-\vtheta(0)}_2}}{\Norm{\vtheta(0)}_2} &\lesssim \left(\sqrt{\frac{\eps}{\nu}}+\sqrt{\frac{\nu}{\eps}}\right)  \sqrt{\frac{n}{ {m}^2 \nu^2\eps^2\left(\frac{\eps}{\nu}\lambda+\frac{\nu}{\eps}\lambda\right)}}\sqrt{\RS(\vtheta(0))}\\
&\lesssim  \sqrt{\frac{n}{  {m}^2 \nu^2\eps^2 }}\sqrt{\RS(\vtheta(0))}.
\end{aligned}
\end{equation}
The  rigorous statements  of    relations \eqref{eq...text...Techinique...LossDecay} and \eqref{eq...text...Techinique...Theta} are given in  \Cref{append...thm..ThetaLazyRegime}.

In the case where    $\gamma\geq  1$ and $\gamma'> \gamma-1$~($\vw$-lazy regime),  the following holds for all $t>0$:
\[\lambda_{\min}\left(\mG^{[a]}(\vtheta(t))  \right)\geq \frac{1}{2}\nu\eps^3\lambda,\]
for some universal constant $\lambda>0$. Hence, we have
\begin{align*}
\frac{\D}{\D t}\RS(\vtheta(t))& =-\frac{m}{n^2}\ve^{\T}\left(\mG^{[a]}(\vtheta(t))+\mG^{[\vw]}(\vtheta(t))\right)\ve \\
&\leq -\frac{2m}{n}\lambda_{\min}\left(\mG^{[a]}(\vtheta(t))\right)\RS(\vtheta(t))\\
& \leq -\frac{m}{n}\nu^2\eps^2\frac{\eps}{\nu}\lambda\RS(\vtheta(t))\\
&= -\frac{m}{n}\nu \eps^3 \lambda\RS(\vtheta(t)),
\end{align*}
thus the following holds
\[
\RS(\vtheta(t)) \leq \exp\left(-\frac{m\nu\eps^3\lambda t}{n}\right)\RS(\vtheta(0)),
\]
and \eqref{eq...text...Techinique...Theta} does not hold anymore.
However, we still have  
\begin{equation}\label{eq...text...Techinique...ThetaVW}
\lim_{m\to\infty}\sup\limits_{t\in[0,+\infty)}\frac{\norm{\vtheta_{\vw}(t)-\vtheta_{\vw}(0)}_2}{\norm{\vtheta_{\vw}(0)}_2}=0,
\end{equation}    
and it can also be illustrated through an intuitive scaling analysis. Since 
\begin{equation*}
\frac{\D}{\D t}\RS(\vtheta(t))=-\frac{\eps}{\nu}\Norm{\nabla_{\vtheta_a}\RS(\vtheta(t))}_2^2 -\frac{\nu}{\eps}\Norm{\nabla_{\vtheta_{\vw}}\RS(\vtheta(t))}_2^2\sim -\frac{m}{n}\nu \eps^3 \lambda\RS(\vtheta(t)),
\end{equation*}
then  
\begin{align*}
\Norm{\nabla_{\vtheta_{\vw}}\RS(\vtheta(t))}_2&\sim \sqrt{\frac{m}{n} \eps^4 \lambda} \sqrt{\RS(\vtheta(t))}\\ 
&\sim  \sqrt{\frac{m}{n} \eps^4 \lambda}\exp\left(-\frac{m\nu\eps^3\lambda t}{2n}\right)\sqrt{\RS(\vtheta(0))}, \end{align*}
hence
\begin{align*}
  {\norm{\vtheta_{\vw}(t)-\vtheta_{\vw}(0)}_2} &\leq     \frac{\nu}{\eps}\int_{0}^t\Norm{\nabla_{\vtheta_{\vw}}\RS(\vtheta(s))}_2\D s\\
 &\leq  
 \frac{\nu}{\eps}\int_{0}^{\infty}\Norm{\nabla_{\vtheta_{\vw}}\RS(\vtheta(s))}_2\D s\\
&\lesssim     \sqrt{\frac{n}{ {m} \eps^4 \lambda }}\sqrt{\RS(\vtheta(0))},
\end{align*}
and as
\[
\Norm{\vtheta_{\vw}(0)}_2\sim \sqrt{m},
\]
then 
\begin{equation}
\begin{aligned}
\frac{{\norm{\vtheta_{\vw}(t)-\vtheta_{\vw}(0)}_2}}{\Norm{\vtheta_{\vw}(0)}_2} &\lesssim  \sqrt{\frac{n}{ {m}^2 \eps^4 \lambda }}\sqrt{\RS(\vtheta(0))}\\
&\lesssim  \sqrt{\frac{n}{ {m}^2 \eps^4  }}\sqrt{\RS(\vtheta(0))}.
\end{aligned}
\end{equation}
The    rigorous statements  of relation  \eqref{eq...text...Techinique...ThetaVW} are given in  \Cref{append...thm..ThetaLazyRegime...WLazy}. To end this part, we provide a sketch of the proofs for \Cref{thm..linear}, see \Cref{fig:sketch1}.

\begin{figure}[ht]
    \includegraphics[width=\textwidth]{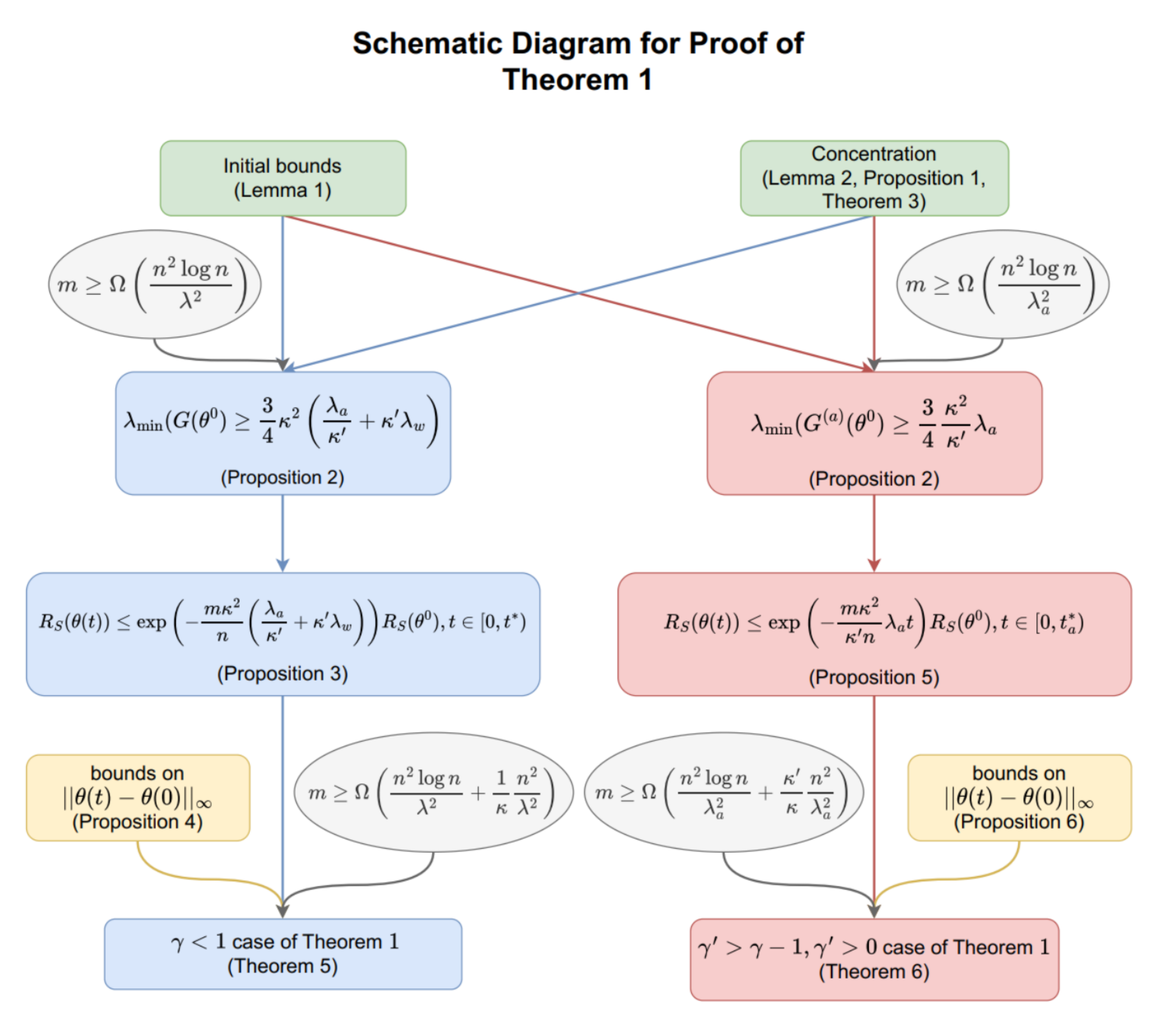}
    \caption{Sketch of proof for \Cref{thm..linear}.}
    \label{fig:sketch1}
\end{figure}
\subsection{Condensed Regime}\label{subsec...Technique...condense}
We remark that the $\{a_k, \vw_k\}_{k=1}^m$ dynamics
\begin{equation}\label{eq...text..techinique...pairs}
\begin{aligned}
\frac{\D {a}_k}{\D {t}}& = -\frac{\eps}{\nu} \frac{1}{n}\sum_{i=1}^ne_i  \frac{\sigma(\eps{\vw}_k^\T\vx_i)}{\eps}, \\
\frac{\D  {\vw}_k}{\D {t}}& =- \frac{\nu}{\eps}\frac{1}{n}\sum_{i=1}^ne_i    {a}_k\sigma^{(1)}(\eps {\vw}_k^\T\vx_i)\vx_i,     
\end{aligned}
\end{equation}
is  coupled   in the sense that the solution of at least one of
the equations in the system depends on knowing one of the other solutions
in the system, and a coupled system is usually hard to solve.

However, in the condense regime, as $\eps\ll 1$ and $\eps\nu\ll \frac{1}{m}$, the evolution of  $\{e_i\}_{i=1}^n$ is slow enough so that it remains close to $\{-y_i\}_{i=1}^n$ over a period of time $T>0$ at the initial stage, hence \eqref{eq...text..techinique...pairs} approximately reads
\begin{equation}\label{eq...text..techinique...pairsreduce}
\begin{aligned}
    \frac{\D a_k}{\D t}
     & \approx  \frac{\eps}{\nu} \frac{1}{n}\sum_{i=1}^n y_i\vw_k^{\T}\vx_i= \frac{\eps}{\nu}\vw_k^\T\vz,\\
    \frac{\D \vw_k}{\D t}
     & \approx \frac{\nu}{\eps}\frac{1}{n}\sum_{i=1}^n y_i a_k\sigma^{(1)}(0)\vx_i=\frac{\nu}{\eps}a_k\vz,
\end{aligned} 
\end{equation}
and the coupled dynamics is reduced to linear dynamics.

We are able to solve out the linear dynamics \eqref{eq...text..techinique...pairsreduce}~(\Cref{append...prop...CondensedRegime...LinearODESolution}), whose solutions read: For each $k\in[m]$, under the initial condition $\left[\nu a_k(0), \eps\vw_k^\T(0)\right]^\T=\left[\nu a_k^0, (\eps\vw_k^0)^\T\right]^\T$, we obtain that
\begin{equation}\label{eq...text...LinearSolution}
\begin{aligned}
\nu a_k(t)&=\nu\left(\frac{1}{ {2}}\exp(\norm{\vz}_2t)+\frac{1}{ {2}}\exp(-\norm{\vz}_2t)\right) a_k^0\\
&~~+ {\eps} \left(\frac{1}{ {2}}\exp(\norm{\vz}_2t)-\frac{1}{ {2}}\exp(-\norm{\vz}_2t)\right)\left<\vw_k^0, \hat{\vz}\right>,\\
\eps\vw_k(t)&= {\nu}\left(\frac{1}{ {2}}\exp(\norm{\vz}_2t)-\frac{1}{ {2}}\exp(-\norm{\vz}_2t)\right)a_k^0\hat{\vz}\\
&~~+{\eps}\left(\frac{1}{ {2}}\exp(\norm{\vz}_2t)+\frac{1}{ {2}}\exp(-\norm{\vz}_2t)\right)\left<\vw_k^0, \hat{\vz}\right>\hat{\vz}\\
&~~-{\eps}\left<\vw_k^0, \hat{\vz}\right>\hat{\vz}+{\eps}\vw_k^0.
\end{aligned}
\end{equation}
We remark that $\{a_k, \vw_k\}_{k=1}^m$  are   the normalized parameters, then $\{\nu a_k, \eps\vw_k\}_{k=1}^m$  corresponds to the original parameters in \eqref{eq...text...OriginalDynamics}.

In the case where $\gamma> 1$ and $0\leq\gamma'<\gamma-1$~($\vw$-lag regime),  as $\eps\gg \nu$, then the magnitude  of $\{\eps\vw_k\}_{k=1}^m$ is much larger than that of  $\{\nu a_k\}_{k=1}^m$ at $t=0$. 
Based on \eqref{eq...text...LinearSolution},  it can be observed that $\{\eps\vw_k\}_{k=1}^m$ remains dormant until $\{\nu a_k\}_{k=1}^m$   attain a magnitude that is commensurate with that of  $\{\eps\vw_k\}_{k=1}^m$, and only then do the magnitudes of $\{\eps\vw_k\}_{k=1}^m$ and $\{\nu a_k\}_{k=1}^m$  experience exponential growth   simultaneously.
In order for the initial condensation of $\vtheta_{\vw}$ to be observed, one has to wait for  some growth in the magnitude of $\{\eps\vw_k\}_{k=1}^m$, hence $T\sim\Omega(\log(m))$. More importantly, compared with the $\vw$-lazy regime, the condition $\gamma'<\gamma-1$  enforces  $\eps\ll\frac{1}{\sqrt{m}}$, thus providing  enough room for $\{\eps\vw_k\}_{k=1}^m$  to grow in the $\vz$-direction before  $\{e_i\}_{i=1}^n$ deviates away from $\{-y_i\}_{i=1}^m$.

In the case where $\gamma> 1$ and $\gamma'<0$~($a$-lag regime),  as $\eps\ll \nu$, then the initial magnitudes of $\{\eps\vw_k\}_{k=1}^m$ is much smaller than that of  $\{\nu a_k\}_{k=1}^m$. 
Based on \eqref{eq...text...LinearSolution},  
it takes only $T\sim\Omega(1)$ for $\{\eps\vw_k\}_{k=1}^m$  to attain a magnitude comparable to  $\{\nu a_k\}_{k=1}^m$,
and this rapid growth leads to the observation of initial condensation towards  the $\vz$-direction, where $\gamma>1$ impose $\{e_i\}_{i=1}^n$ to a small neighbourhood of $\{-y_i\}_{i=1}^m$ for a period of time which is at least of order one.  To end this
part, we provide a sketch of the proofs for \Cref{thm..CondensedRegime}, see \Cref{fig:sketch2}.

\begin{figure}[ht]
\centering
    \includegraphics[width=0.8\textwidth]{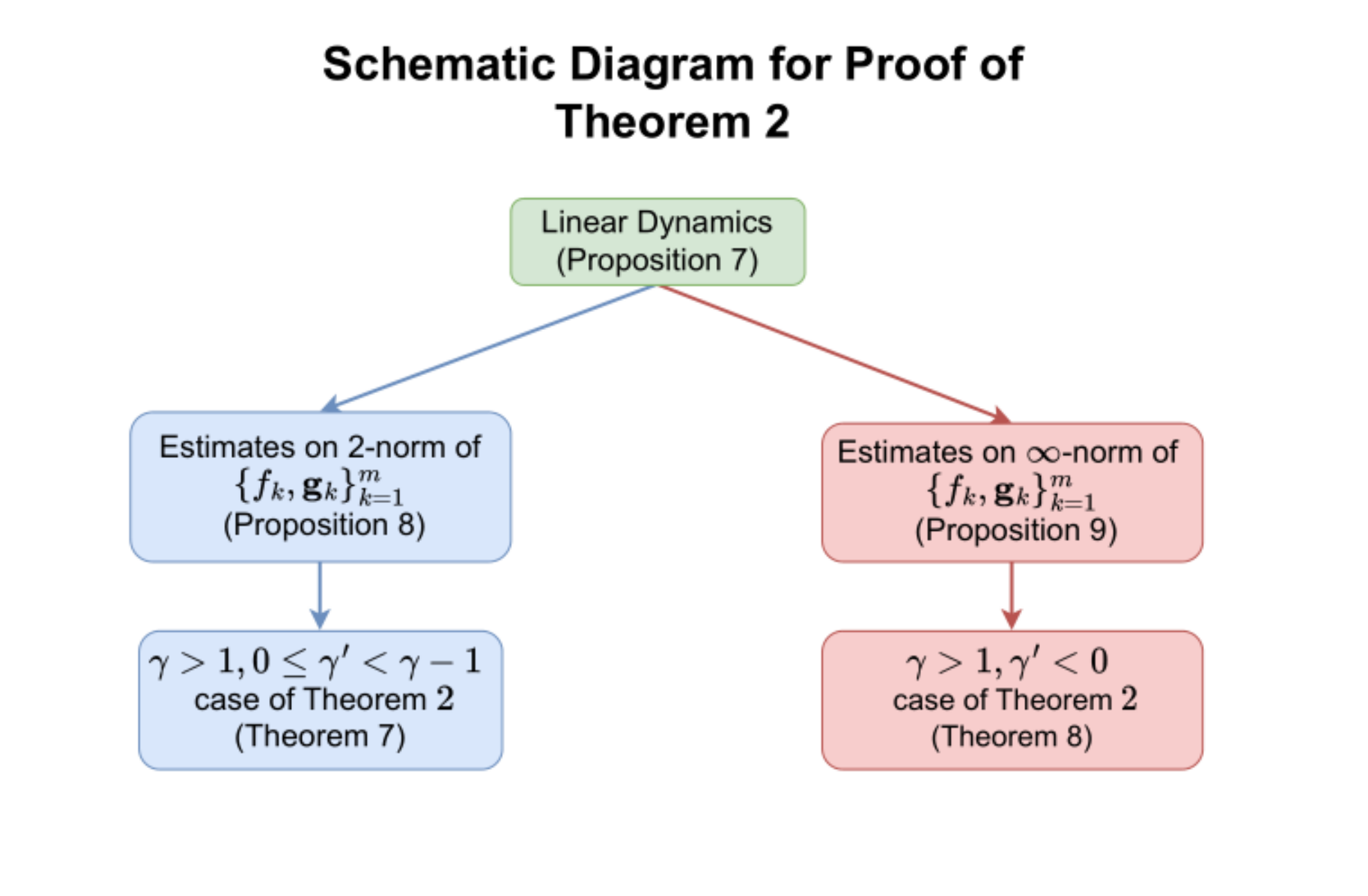}
    \caption{Sketch of proof for \Cref{thm..CondensedRegime}.}
    \label{fig:sketch2}
\end{figure}
\section{Conclusions}\label{section....Conclusion}
In this paper, we present the phase diagram of initial condensation for two-layer NNs with a wide class of smooth activation functions. We demonstrate the distinct features exhibited by NNs in the linear regime area and condensed regime area,  and we provide a complete and detailed analysis for the transition   across the boundary~(critical regime) in the phase diagram. Moreover, in comparison with the work of Luo et al.~\cite{luo2021phase},
we identify the direction towards which the weight parameters condense, and we draw estimates on the time required for initial condensation to occur in contrast to the work of Zhou et al.~\cite{zhou2021towards}. The phase diagram at initial stage is crucial in that it is a valuable tool for understanding  the implicit regularization effect provided by weight initialization schemes, and it serves as a  cornerstone upon which future works can be done to provide thorough characterization of  the dynamical behavior of general NNs  at each of the identified regime.  

In future, we endeavor to  establish a framework 
for the analysis of  initial condensation   by a series of papers. In our upcoming publication, we   plan to extend this formalism to Convolutional Neural Network~(CNN)~\cite{o2015introduction},
and apply it to investigate  the phenomenon of  condensation for a  wide range of NN architectures, including fully-connected deep network~(DNN) and Residual Network (ResNet)~\cite{He2016Deep}.  
\section*{Acknowledgments}
This work is sponsored by the National Key R\&D Program of China  Grant No. 2022YFA1008200 (Z. X., T. L.), the Shanghai Sailing Program, the Natural Science Foundation of Shanghai Grant No. 20ZR1429000  (Z. X.), the National Natural Science Foundation of China Grant No. 62002221 (Z. X.), the National Natural Science Foundation of China Grant No. 12101401 (T. L.), Shanghai Municipal Science and Technology Key Project No. 22JC1401500 (T. L.), Shanghai Municipal of Science and Technology Major Project No. 2021SHZDZX0102, and the HPC of School of Mathematical Sciences and the Student Innovation Center, and the Siyuan-1 cluster supported by the Center for High Performance Computing at Shanghai Jiao Tong University.
\bibliographystyle{siam}
\bibliography{Ref.bib}
\newpage
\appendix
 \section{Several Estimates on the Initialization}
We   begin this part by an estimate on   standard Gaussian vectors.
\begin{lemma}[Bounds on initial parameters]\label{append...lemma..Initialization}
    Given any $\delta\in(0,1)$, we have with probability at least $1-\delta$ over the choice of $\vtheta^0$,
    \begin{equation}\label{eqn:lem1}
        \max\limits_{k\in[m]}\left\{\abs{a_k^0},\;\norm{\vw_k^0}_{\infty}\right \}\leq \sqrt{2\log\frac{2m(d+1)}{\delta}},
    \end{equation}
\end{lemma}
\begin{proof}
    If $\rX \sim \fN(0, 1)$, then   for any $\eta > 0$,
    \[\Prob(\abs{\rX} > \eta) \leq 2\exp\left(-\frac{1}{2}\eta^2\right).\] 
 Since  for $k\in[m]$,
    \[	{a}_k^0\sim \fN(0,1), \quad {\vw}_k^0\sim \fN(\vzero,\mI_d),\]
 where
    \[	{\vw}_k^0:=\left[\left({w}_k^0\right)_{1}, \left({w}_k^0\right)_{2}, \cdots, \left({w}_k^0\right)_{d} \right]^\T,\]
 then for any $j\in[d]$,
    \[	\left({w}_k^0\right)_{j}\sim \fN(0,1),\]
and they are all independent with each other. 
As we set
    \begin{equation*}
        \eta = \sqrt{2\log\frac{2m(d+1)}{\delta}},
    \end{equation*}
we obtain that
\begin{align*}
&~~\Prob\left(\max\limits_{k\in[m]}\left\{\abs{a_k^0},\; \norm{\vw_k^0}_{\infty}\right \}>\eta\right)\\
& = \Prob\left(\max\limits_{k\in[m],j\in[d]}\left\{\abs{a_k^0},\; \abs{{(w^0_k)}_{j}}\right \}>\eta\right)  \\
&=\Prob\left(\bigcup\limits_{k=1}^m\left[\left(\abs{a_k^0}>\eta\right)\bigcup\left(\bigcup\limits_{j=1}^d\left(\abs{{(w_k^0)}_{j}}>\eta\right)\right)\right]\right) \\
& \leq \sum_{k=1}^m \Prob\left(\abs{a_k^0}>\eta\right) + \sum_{k=1}^m\sum_{j=1}^d \Prob\left(\abs{{(w^0_k)}_{j}}>\eta\right)\\
& \leq 2m \exp\left( -\frac{1}{2}\eta^2\right) + 2md \exp\left( -\frac{1}{2}\eta^2 \right)\\
& = 2m(d+1)\exp\left({-\frac{1}{2}\eta^2}\right) = \delta.
\end{align*}
\end{proof}
Next we would like to introduce the sub-exponential norm~\cite{Vershynin2010Introduction} of a random variable and  Bernstein's Inequality.
\begin{definition}[Sub-exponential norm]
The sub-exponential norm of a random variable $\rX$ is defined as
    \begin{equation}
        \norm{\rX}_{\psi_1} := \inf\left\{s>0 \mid \Exp_{\rX}\left[\exp\left(\frac{\abs{\rX}}{s}\right)\right]\leq 2\right\}.
    \end{equation}
    In particular, we denote $\rX:=\chi^2(d)$ as   a   chi-square distribution with  $d$ degrees of freedom~\cite{Laurent2000AdaptiveEstimationQuadratic}, and its 
   sub-exponential norm by \[C_{\psi,d}:=\norm{\rX}_{\psi_1}.\] 
\end{definition}

\begin{remark}
As the  probability density function of  $\rX=\chi^2(d)$ reads
    \begin{equation*}
        f_{\rX}(z):=\frac{1}{2^{\frac{d}{2}}\Gamma(\frac{d}{2})}z^{\frac{d}{2}-1}\exp\left({-\frac{z}{2}}\right),\quad z\geq 0,
    \end{equation*}
    we note that
\begin{align*}
        \Exp_{\rX\sim\chi^2(1)}\exp\left(\frac{\abs{\rX}}{s}\right)
         & =\int_0^{+\infty}\frac{1}{2^{\frac{1}{2}}\Gamma(\frac{1}{2})}z^{-\frac{1}{2}}\exp\left(-\left({\frac{1}{2}-\frac{1}{s}}\right)z\right)\diff{z}=\frac{1}{\sqrt{1-\frac{2}{s}}},
\end{align*}
Then we obtain that \[\frac{8}{3}\leq C_{\psi,1}<3.\] Moreover, we notice that 
\begin{align*}
        \Exp_{\rX\sim\chi^2(d)}\exp\left(\frac{\abs{\rX}}{s}\right)
         & =\left(\Exp_{\rY\sim\chi^2(1)}\exp\left(\frac{\abs{\rY}}{s}\right)\right)^d,
\end{align*}
as we set 
\begin{align*}
       \frac{1}{\sqrt{1-\frac{2}{s}}}=2^{\frac{1}{d}},
\end{align*}
then \[s=\frac{2}{1-2^{-\frac{2}{d}}},\]
hence 
\[\frac{2}{1-2^{-\frac{2}{d}}}\leq C_{\psi,d}<3,\]
and 
\[ C_{\psi,d}\geq C_{\psi,1},\]
for $d\geq 1$.
\end{remark}
\begin{lemma}\label{append...lem..chi2SubExpNorm}
 Given $\Norm{\vx}_2\leq 1$ and  $\Norm{\vy}_2\leq 1$ 
 equipped with   $\sigma(\cdot):\sR\to\sR$, with  $\sigma\in\fC^1(\sR)$, satisfying
 \[\sigma(0)=0,\]
 and 
	\begin{equation*}
		\Norm{\sigma^{(1)}(\cdot)}_{\infty}\leq 1.
	\end{equation*}
Under the condition that 
  \[	{a}\sim \fN(0,1), \quad {\vw}\sim \fN(0, \mI_d),\] 
then for any $\eps>0$,
\begin{itemize}
\item if \[\rX_1:=\frac{1}{\eps^2}\sigma(\eps\vw^\T\vx)\sigma(\eps\vw^\T\vy),\] 
then \[\norm{\rX_1}_{\psi_1}\leq  C_{\psi,d},\]
\item and if \[\rX_2:=a^2\sigma^{(1)}(\eps\vw^\T\vx)\sigma^{(1)}(\eps\vw^\T\vy)\left<\vx, \vy\right>,\] 
then \[\norm{\rX_2}_{\psi_1} \leq C_{\psi,1}.\]
\end{itemize}
\end{lemma}
\begin{proof}
Let
    \[\rZ:=\norm{\vw}^2_2\sim\chi^2(d).\]  
Since  
    \[\Abs{\rX_1}\leq \Abs{\vw^\T\vx}\Abs{\vw^\T\vy}=\norm{\vw}^2_2,\] 
then by definition,
\begin{align*}
\norm{\rX_1}_{\psi_1}
& =\inf\left\{s>0\mid\Exp_{\rX_1}\exp\left(\frac{\Abs{\rX_1}}{s}\right)\leq 2\right\} \\
& =\inf\left\{s>0\mid\Exp_{\vw}\exp\left(\frac{\Abs{\sigma(\vw^\T\vx)\sigma(\vw^\T\vy)}}{s}\right)\leq 2\right\}  \\
&\leq \inf\left\{s>0\mid\Exp_{\vw}\exp\left(\frac{\norm{\vw}^2_2}{s}\right)\leq 2\right\},         
\end{align*}
hence 
\begin{equation}
  \norm{\rX_1}_{\psi_1}\leq C_{\psi,d}.
\end{equation}

By similar reasoning, as 
    \[\Abs{\rX_2}\leq a^2,\] 
hence 
\begin{equation}
  \norm{\rX_2}_{\psi_1}\leq C_{\psi,1}.
\end{equation}
 
\end{proof}
We state an important theorem without proof, details of which can be found in~\cite{Vershynin2010Introduction}.
\begin{theorem}[Bernstein's inequality]\label{append...thm...BernsteinInequality}
Let $\{\rX_k\}_{k=1}^m$ be i.i.d.\ sub-exponential random variables satisfying \[\Exp\rX_1=\mu,\] then for any $\eta\geq 0$, we have
\begin{equation*}
\Prob\left(\Abs{\frac{1}{m}\sum_{k=1}^m\rX_k-\mu}\geq \eta\right)\leq 2\exp\left(-C_0 m \min\left(\frac{\eta^2}{\norm{\rX_1}^2_{\psi_1}},\frac{\eta}{\norm{\rX_1}_{\psi_1}}\right)\right),
\end{equation*}
for some absolute constant $C_0$.
\end{theorem}
\begin{proposition}[Upper and lower bounds of initial parameters]\label{append...prop..UpperBoundandLowerBoundInitial}
Given any $\delta\in(0,1)$,   if
    \[m=\Omega\left(\log \frac{4}{\delta} \right),\]
then with probability at least $1-\delta$ over the choice of $\vtheta^0$,
\begin{align}
\sqrt{\frac{m}{2}}& \leq \norm{\vtheta_{a}^0}_2\leq \sqrt{\frac{3m}{2}},\label{append...eq..InitialANorm}\\
\sqrt{\frac{md}{2}} & \leq \norm{\vtheta_{\vw}^0}_2\leq \sqrt{\frac{3md}{2}},\label{append...eqgroup..InitialWNorm}     \end{align}
and 
\begin{equation} \label{append...eqgroup..InitialThetaNorm} 
\sqrt{\frac{m(d+1)}{2}}\leq \norm{\vtheta^0}_2\leq \sqrt{\frac{3m(d+1)}{2}}.                                          
\end{equation}
\end{proposition}
\begin{proof}
Since 
\[\left(a_1^0\right)^2, \left(a_2^0\right)^2, \cdots, \left(a_m^0\right)^2\sim\chi^2(1)\] 
are i.i.d.\ sub-exponential  
random variables   with 
\[\Exp \left(a_1^0\right)^2=1.\] 
By application of  \Cref{append...thm...BernsteinInequality}, we have
\begin{equation*}
        \Prob\left(\Abs{\frac{1}{m}\sum_{k=1}^{m}\left(a_k^0\right)^2-1}\geq \eta\right)\leq 2\exp\left(-C_0 m \min\left(\frac{\eta^2}{C^2_{\psi,1}},\frac{\eta}{C_{\psi,1}}\right)\right),
\end{equation*}
since $C_{\psi,1}\geq \frac{8}{3}>2$, then  for any $0\leq \eta\leq 2$, it is obvious that
\[
\min\left(\frac{\eta^2}{C^2_{\psi,1}},\frac{\eta}{C_{\psi,1}}\right)=\frac{\eta^2}{C^2_{\psi,1}}.
\]
We  set 
\[
2\exp\left(-C_0 m \frac{\eta^2}{C^2_{\psi,1}}\right)=\frac{\delta}{2},
\]
and consequently,
\[
\eta=\sqrt{\frac{C^2_{\psi,1}}{m C_0 }\log\frac{4}{\delta}},
\]
then  with probability at least $1-\frac{\delta}{2}$ over the choice of $\vtheta^0$,
\begin{equation}
   m\left(1-\sqrt{\frac{C^2_{\psi,1}}{m C_0 }\log\frac{4}{\delta}}\right) \leq \norm{\vtheta_{a}^0}^2_2\leq m\left(1+\sqrt{\frac{C^2_{\psi,1}}{m C_0 }\log\frac{4}{\delta}}\right),
\end{equation}
by choosing 
\[
m\geq{\frac{4C^2_{\psi,1}}{ C_0 }\log\frac{4}{\delta}},
\]
we obtain that 
\[ \sqrt{\frac{m}{2}}
           \leq \norm{\vtheta_{a}^0}_2\leq \sqrt{\frac{3m}{2}}.\]

As for $\vtheta_{\vw}^0$, 
by application of \Cref{append...thm...BernsteinInequality},
    \begin{equation*}
        \Prob\left(\Abs{\frac{1}{md}\sum_{k=1}^{md}\left(w_k^0\right)^2-1}\geq \eta\right)\leq 2\exp\left(-C_0 md \min\left(\frac{\eta^2}{C^2_{\psi,1}},\frac{\eta}{C_{\psi,1}}\right)\right).
    \end{equation*}
We  set 
\[
2\exp\left(-C_0 md \frac{\eta^2}{C^2_{\psi,1}}\right)=\frac{\delta}{2},
\]
and consequently,
\[
\eta=\sqrt{\frac{C^2_{\psi,1}}{md C_0 }\log\frac{4}{\delta}},
\]
then  with probability at least $1-\frac{\delta}{2}$ over the choice of $\vtheta^0$,
\begin{equation}
   md\left(1-\sqrt{\frac{C^2_{\psi,1}}{md C_0 }\log\frac{4}{\delta}}\right) \leq \norm{\vtheta_{\vw}^0}^2_2\leq md\left(1+\sqrt{\frac{C^2_{\psi,1}}{md C_0 }\log\frac{4}{\delta}}\right),
\end{equation}
by choosing 
\[
m\geq{\frac{4C^2_{\psi,1}}{ dC_0 }\log\frac{4}{\delta}},
\]
we obtain that 
\[ \sqrt{\frac{md}{2}}
\leq \norm{\vtheta_{\vw}^0}_2\leq \sqrt{\frac{3md}{2}}.\]

Finally, as 
\[
\norm{\vtheta^0}_2^2=\norm{\vtheta_{a}^0}^2_2+\norm{\vtheta_{\vw}^0}^2_2,
\]
then with probability at least $1-\delta$,
\[
  \sqrt{\frac{m(d+1)}{2}}
          \leq \norm{\vtheta^0}_2\leq \sqrt{\frac{3m(d+1)}{2}}.
\]
\end{proof}
\section{Linear Regime}\label{appen...LinearRegime}
\subsection{Full Rankness of the Gram matrices}\label{append...subsection....GramMatrices}
We shall state two lemmas concerning full rankness of the Gram matrices, which have been stated as Lemma $\mathrm{F.1.}$ and Lemma $\mathrm{F.2.}$ in Du et al.~\cite{Du2018Gradient}.
\begin{lemma}\label{append...lemma...GramNoDerivative}
Assume $\sigma(\cdot)$ is analytic and not a polynomial function. Consider input data set as 
$\fZ=\{\vz_1,\vz_2,\dots,\vz_n\}$, and  non-parallel with each other, i.e., for any $j\neq k$,
\[\vz_j\notin \mathrm{span}\left(\vz_k\right),\] 
 we  define
\begin{equation}
\begin{aligned}
&\mG(\fZ):=\left[\mG(\fZ)_{ij}\right],
\\
    \mG(\fZ)_{ij}&:=\Exp_{\vw\sim \fN\left(\vzero,  \mI \right)}\left[ \sigma(\vw^{\T}\vz_i)\sigma(\vw^{\T}\vz_j)\right],
    \end{aligned}
\end{equation}
then $\lambda_{\min}\left(\mG(\fZ)\right)>0.$
\end{lemma}
Similar to   \Cref{append...lemma...GramNoDerivative}, we have another  Lemma.
\begin{lemma}\label{append...lemma...GramWithDerivative}
Assume $\sigma(\cdot)$ is analytic and not a polynomial function. Consider input data set as 
$\fZ=\{\vz_1,\vz_2,\dots,\vz_n\}$, and  non-parallel with each other,   we  define
\begin{equation}
\begin{aligned}
 \mG(\fZ)&:=\left[\mG(\fZ)_{ij}\right],
\\  \mG(\fZ)_{ij}&:=\Exp_{\vw\sim \fN\left(\vzero, \mI \right)}\left[ \sigma^{(1)}(\vw^{\T}\vz_i)\sigma^{(1)}(\vw^{\T}\vz_j)\left<\vz_i, \vz_j\right>\right],
 \end{aligned}
\end{equation}
then $\lambda_{\min}\left(\mG(\fZ)\right)>0.$
\end{lemma}

We state an important theorem concerning the least eigenvalue of the normalized Gram matrices $\mK^{[a]}$ and $\mK^{[\vw]}$   at infinite width limit. Recall that 
\begin{equation}
\begin{aligned}
K_{ij}^{[a]}&=\frac{1}{\eps^2}\Exp_{\vw\sim \fN\left(\vzero,  \mI_d \right)}{\sigma(\eps\vw^\T\vx_i)\sigma(\eps\vw^\T\vx_j)},\\
K_{ij}^{[\vw]}&=\Exp_{(a,\vw)\sim \fN\left(\vzero,  \mI_{d+1} \right)}a^2\sigma^{(1)}(\eps\vw^\T\vx_i)\sigma^{(1)}(\eps\vw^\T\vx_j)\left<\vx_i, \vx_j\right>.
\end{aligned}
\end{equation}
\begin{theorem}[Least eigenvalue of Gram matrices at infinite width]\label{append...thm...LeastEigenGram}
Under \Cref{Assumption....ActivationFunctions...NTK}  and \Cref{Assump...Unparallel},
the normalized Gram matrices $\mK^{[a]}$ and $\mK^{[\vw]}$ are strictly positive definite, and both of their least eigenvalues are of order $1$.
   
In other words, if we denote
\begin{equation*}
        \lambda:=\min \{\lambda_a,\lambda_{\vw}\}>0,
\end{equation*}
where
\begin{equation}
        \lambda_a := \lambda_{\min}\left(\mK^{[a]}\right),\quad \lambda_{\vw} := \lambda_{\min}\left(\mK^{[\vw]}\right),
\end{equation}
then
\[
     \lambda>0, 
\]
and 
\[
      \lambda\sim\Omega(1).
\]
\end{theorem}
\begin{proof}
Since 	the following limit exists
\begin{equation*} 
{\gamma}_2=\lim_{m\to\infty} -\frac{\log \eps}{\log m},
\end{equation*}
then we conclude that 
\[
\lim_{m\to\infty} \eps=0,
\]
or 
\[
\lim_{m\to\infty} \eps=1,
\]
or
\[
\lim_{m\to\infty} \eps=+\infty,
\]
and 	the normalized matrices $	\mK^{[a]}=\left[K_{ij}^{[a]}\right]_{n\times n}$ and $\mK^{[\vw]}=\left[K_{ij}^{[\vw]}\right]_{n\times n}$ read
\begin{align*}
K_{ij}^{[a]}&=\frac{1}{\eps^2}\Exp_{\vw\sim \fN\left(\vzero,  \mI_d \right)}{\sigma(\eps\vw^\T\vx_i)\sigma(\eps\vw^\T\vx_j)},\\
K_{ij}^{[\vw]}&=\Exp_{(a,\vw)\sim \fN\left(\vzero,  \mI_{d+1} \right)}a^2\sigma^{(1)}(\eps\vw^\T\vx_i)\sigma^{(1)}(\eps\vw^\T\vx_j)\left<\vx_i, \vx_j\right>\\
&=\Exp_{\vw\sim \fN\left(\vzero,  \mI_d \right)}\sigma^{(1)}(\eps\vw^\T\vx_i)\sigma^{(1)}(\eps\vw^\T\vx_j)\left<\vx_i, \vx_j\right>.
\end{align*}

For the case where $\lim_{m\to\infty} \eps=1$, by direct application of \Cref{append...lemma...GramNoDerivative} and \Cref{append...lemma...GramWithDerivative}, $\lambda \sim \Omega(1)$ can be easily achieved.
 
For the case where $\lim_{m\to\infty} \eps=0$,  the normalized matrices   read
\begin{align*}
K_{ij}^{[a]}&=\lim_{m\to\infty} \frac{1}{\eps^2}\Exp_{\vw\sim \fN\left(\vzero,  \mI_d \right)}{\sigma(\eps\vw^\T\vx_i)\sigma(\eps\vw^\T\vx_j)}\\
&= \Exp_{\vw\sim \fN\left(\vzero,  \mI_d \right)}(\vw^\T\vx_i)(\vw^\T\vx_j),  \\
K_{ij}^{[\vw]}&=\lim_{m\to\infty}\Exp_{\vw\sim \fN\left(\vzero,  \mI_d \right)}\sigma^{(1)}(\eps\vw^\T\vx_i)\sigma^{(1)}(\eps\vw^\T\vx_j)\left<\vx_i, \vx_j\right>\\
&=\Exp_{\vw\sim \fN\left(\vzero,  \mI_d \right)}\left[\sigma^{(1)}(0)^2\right]\left<\vx_i, \vx_j\right>\\
&=\left<\vx_i, \vx_j\right>.
\end{align*}
hence both  $\mK^{[a]}$ and $\mK^{[\vw]}$ are independent of $\vtheta_{\vw}$ and $\vtheta_{\vw}$, and depend merely on the input sample $\fS$. Consequently,  $\lambda \sim \Omega(1)$.

  For the case where $\lim_{m\to\infty} \eps=+\infty$,  the normalized matrices   read
\begin{align*}
K_{ij}^{[a]} &=\lim_{m\to\infty} \frac{1}{\eps^2}\Exp_{\vw\sim \fN\left(\vzero,  \mI_d \right)}{\sigma(\eps\vw^\T\vx_i)\sigma(\eps\vw^\T\vx_j)} \\
&= \lim_{m\to\infty}\Exp_{\vw\sim \fN\left(\vzero,  \mI_d \right)}\sigma^{(1)}(\eps\vw^\T\vx_i)\sigma^{(1)}(\eps\vw^\T\vx_j),  \\
K_{ij}^{[\vw]} 
&=\lim_{m\to\infty}\Exp_{\vw\sim \fN\left(\vzero,  \mI_d \right)}\sigma^{(1)}(\eps\vw^\T\vx_i)\sigma^{(1)}(\eps\vw^\T\vx_j)\left<\vx_i, \vx_j\right>,
\end{align*}
where the entries of the matrix $\mK^{[a]}$ exhibit    leaky-ReLU-like behaviors. It has been proven in Du et al.~\cite{Du2018Gradienta} that under the unit data and nonparallel assumption, $\mK^{[a]}$ is positive definite. Moreover, its expression  can be explicitly computed, see~\cite{radhakrishnan2022lecture}. As for the matrix $\mK^{[\vw]}$, we   define a function $\mG:\sR^{n\times n} \to \sR^{n\times n},$ such that 
\begin{align*}
\mG(\mK)_{ij}:={\mK}_{ij}\Exp_{(u,v)^{\T}\sim \fN\left(\vzero, \begin{pmatrix}{\mK}_{ii}&{\mK}_{ij}\\
{\mK}_{ji}&{\mK}_{jj}\end{pmatrix} \right)}\sigma^{(1)}(u)\sigma^{(1)}(v).
\end{align*}
We denote that $\mA\succeq \mB$ if and only if $\mA-\mB$ is a semi-positive definite matrix, and $\mA\succ \mB$ if and only if $\mA-\mB$ is a strictly positive definite matrix.
Consequently, a scalar function $g(t)$ is defined as
\begin{align*}
    g(t):= \min_{\mK: \mK\succ 0,~\frac{1}{c}\leq \mK_{ii}\leq c,~ \lambda_{\min}\left(\mK\right)\geq t} \lambda_{\min}\left(\mG(\mK)\right).
\end{align*}
Then \Cref{append...lemma...GramWithDerivative} guarantees that 
\begin{equation*}
     g(\lambda_0)>0,
\end{equation*}
and  $\lambda \sim \Omega(1)$.
\end{proof}
Finally, we are hereby to show that the Gram matrix $\mG$ at $t=0$ is also positive definite. Recall that 
\begin{equation}
\begin{aligned}
\mG^{[a]}(\vtheta)&=\frac{\nu\eps^3}{ m}\sum_{k=1}^m\frac{1}{\eps^2}\sigma(\eps\vw_k^\T\vx_i)\sigma(\eps\vw_k^\T\vx_j),\\
\mG^{[\vw]}(\vtheta)&=\frac{\nu^3\eps}{m}\sum_{k=1}^m a_k^2\sigma^{(1)}(\eps\vw_k^\T\vx_i)\sigma^{(1)}(\eps\vw_k^\T\vx_j)\left<\vx_i, \vx_j\right>, 
\end{aligned}
\end{equation}
\begin{proposition}[Least eigenvalue of initial Gram matrices]\label{append...prop...LeastEigenvalueatInitial}
Given any $\delta\in(0,1)$, under \Cref{Assumption....ActivationFunctions...NTK}, \Cref{Assump...Unparallel} and \Cref{assump...LimitExistence}, if 
\[m=\Omega\left(\frac{n^2}{\lambda^2}\log\frac{4n^2}{\delta}\right),\]
where $\lambda=\min \{\lambda_a,\lambda_{\vw}\}$, whose definition can be found in \Cref{append...thm...LeastEigenGram}.
Then with probability at least $1-\delta$ over the choice of $\vtheta^0$, we have
\begin{equation}
\lambda_{\min}\left(\mG\left(\vtheta^0\right)\right)\geq\frac{3}{4}\nu^2\eps^2\left(\frac{\eps}{\nu}\lambda_a+\frac{\nu}{\eps}\lambda_{\vw}\right).
\end{equation}
\end{proposition}
\begin{proof}
For any $\eta > 0$ and all $i, j\in [n]$,  we define the events
\begin{equation*}
\begin{aligned}
\Omega_{ij}^{[a]}&:=\left\{\vtheta^0 \mid \left\lvert \frac{1}{\nu\eps^3}G^{[a]}_{ij}\left(\vtheta^0\right) - K^{[a]}_{ij}\right\rvert \leq \frac{\eta}{n} \right\}, \\
\Omega_{ij}^{[\vw]} & := \left\{\vtheta^0 \mid \left\lvert \frac{1}{\nu^3\eps}G^{[\vw]}_{ij}\left(\vtheta^0\right) - K^{[\vw]}_{ij}\right\rvert \leq \frac{\eta}{n} \right\}.   
\end{aligned}
\end{equation*}
By application of \Cref{append...lem..chi2SubExpNorm} and \Cref{append...thm...BernsteinInequality},  we obtain that for sufficiently small $\eta>0$, the following holds
\begin{equation*}
\begin{aligned}
\Prob(\Omega^{[a]}_{ij})   & \geq 1-2\exp\left(-\frac{mC_0\eta^2}{n^2C_{\psi,d}^2}\right) \\
\Prob(\Omega^{[\vw]}_{ij}) & \geq 1-2\exp\left(-\frac{mC_0\eta^2}{n^2C_{\psi,1}^2}\right)\geq 1-2\exp\left(-\frac{mC_0\eta^2}{n^2C_{\psi,d}^2}\right),
\end{aligned}
\end{equation*}
hence with probability at least $1-4n^2\exp\left(-\frac{mC_0\eta^2}{n^2C_{\psi,d}^2}\right)$ over the choice of $\vtheta^0$, we have
\begin{equation*}
\begin{aligned}
\Norm{\frac{1}{\nu\eps^3}\mG^{[a]}\left(\vtheta^0\right) - \mK^{[a]}}_\mathrm{F} &\leq n \Norm{\frac{1}{\nu\eps^3}\mG^{[a]}\left(\vtheta^0\right) -\mK^{[a]}}_{\infty}\leq \eta,\\
\Norm{\frac{1}{\nu^3\eps}\mG^{[\vw]}\left(\vtheta^0\right) - \mK^{[\vw]}}_\mathrm{F} &\leq n \Norm{\frac{1}{\nu^3\eps}\mG^{[\vw]}\left(\vtheta^0\right) - \mK^{[\vw]}}_{\infty}\leq \eta.
\end{aligned}
\end{equation*}
By taking $\eta=\frac{\lambda}{4}$,
\[\delta=4n^2\exp\left(-\frac{mC_0\lambda^2}{16n^2C_{\psi,d}^2}\right),\]
and we conclude that 
\begin{equation*}
\begin{aligned}
\lambda_{\min}\left(\mG\left(\vtheta^0\right)\right)& \geq\lambda_{\min}\left(\mG^{[a]}\left(\vtheta^0\right)\right) + \lambda_{\min}\left(\mG^{[\vw]}\left(\vtheta^0\right)\right)\\
& \geq{\nu\eps^3}\lambda_a-{\nu\eps^3}\Norm{\frac{1}{\nu\eps^3}\mG^{[a]}\left(\vtheta^0\right) - \mK^{[a]}}_\mathrm{F} \\
& \quad + {\nu^3\eps}\lambda_{\vw} - {\nu^3\eps}\Norm{\frac{1}{\nu^3\eps}\mG^{[\vw]}\left(\vtheta^0\right) - \mK^{[\vw]}}_\mathrm{F}\\
& \geq{\nu\eps^3}(\lambda_a-\eta)+{\nu^3\eps}(\lambda_{\vw}-\eta)\\
&\geq\frac{3}{4}\nu^2\eps^2\left(\frac{\eps}{\nu}\lambda_a+\frac{\nu}{\eps}\lambda_{\vw}\right).
\end{aligned}
\end{equation*}
\end{proof}
\subsection{$\vtheta$-lazy Regime}\label{subsec...thetalazy}
In the rest of the paper, 
we define   two quantities
\begin{equation}
 \alpha(t):=\max\limits_{k\in[m],s\in[0,t]}\abs{a_k(s)}, \quad \omega(t):=\max\limits_{k\in[m],s\in[0,t]}\norm{\vw_k(s)}_{\infty},
\end{equation}
and we denote
\begin{equation}
    t^\ast = \inf\{t \mid \vtheta(t)\notin \mathcal{N}\left(\vtheta^0\right)\},
\end{equation}
where the event is defined as 
\begin{equation}
\mathcal{N}\left(\vtheta^0\right) := \left\{\vtheta \mid \norm{\mG(\vtheta) - \mG\left(\vtheta^0\right)}_\mathrm{F}\leq \frac{1}{4}\nu^2\eps^2\left(\frac{\eps}{\nu}\lambda_a+\frac{\nu}{\eps}\lambda_{\vw}\right)\right\}.
\end{equation}
We observe immediately that 
the event $\mathcal{N}\left(\vtheta^0\right)\neq \varnothing$, since $\vtheta^0\in\mathcal{N}\left(\vtheta^0\right)$.
Recall that 
$ 
\lambda=\min \{\lambda_a,\lambda_{\vw}\}
$,
where
\begin{equation*}
        \lambda_a = \lambda_{\min}\left(\mK^{[a]}\right),\quad \lambda_{\vw} = \lambda_{\min}\left(\mK^{[\vw]}\right),
\end{equation*}
whose definition can be found in \Cref{append...thm...LeastEigenGram}.
Then we have the following lemma.
\begin{proposition}\label{append...prop...InitialDecay}
Given  any $\delta\in(0,1)$,  under \Cref{Assumption....ActivationFunctions...NTK}, \Cref{Assump...Unparallel} and \Cref{assump...LimitExistence}, if 
\[m=\Omega\left(\frac{n^2}{\lambda^2}\log\frac{4n^2}{\delta}\right),\]
then with probability at least $1-\delta$ over the choice of $\vtheta^0$,  we have for any time $t\in[0, t^\ast)$,
\begin{equation}
\RS(\vtheta(t)) \leq \exp\left(-\frac{m}{n}\nu^2\eps^2\left(\frac{\eps}{\nu}\lambda_a+\frac{\nu}{\eps}\lambda_{\vw}\right)t\right)\RS\left(\vtheta^0\right).
\end{equation}
\end{proposition}
\begin{proof}
\Cref{append...prop...LeastEigenvalueatInitial} implies that for any $\delta\in(0,1)$, with probability at least $1-\delta$ over the choice of $\vtheta^0$ and for any $\vtheta\in\mathcal{N}\left(\vtheta^0\right)$, we have
\begin{equation*}
\begin{aligned}
\lambda_{\min}\left(\mG(\vtheta)\right)& \geq \lambda_{\min}\left(\mG\left(\vtheta^0\right)\right) - \norm{\mG(\vtheta) - \mG\left(\vtheta^0\right)}_\mathrm{F}\\
& \geq \frac{3}{4}\nu^2\eps^2\left(\frac{\eps}{\nu}\lambda_a+\frac{\nu}{\eps}\lambda_{\vw}\right) - \frac{1}{4}\nu^2\eps^2\left(\frac{\eps}{\nu}\lambda_a+\frac{\nu}{\eps}\lambda_{\vw}\right) \\
& = \frac{1}{2}\nu^2\eps^2\left(\frac{\eps}{\nu}\lambda_a+\frac{\nu}{\eps}\lambda_{\vw}\right).
\end{aligned}\end{equation*}
Note that
\begin{align*}
G_{ij}(\vtheta) &= G_{ij}^{[a]}(\vtheta)+G^{[\vw]}_{ij}(\vtheta)\\
& =\frac{\nu\eps^3}{m}\sum_{k=1}^m\frac{1}{\eps^2}\sigma(\eps\vw_k^\T\vx_i)\sigma(\eps\vw_k^\T\vx_j)+\frac{\nu^3\eps}{m}\sum_{k=1}^m a_k^2\sigma^{(1)}(\eps\vw_k^\T\vx_i)\sigma^{(1)}(\eps\vw_k^\T\vx_j)\left<\vx_i, \vx_j\right>,
\end{align*}
and the normalized flow
\begin{equation*}
\begin{aligned}
\nabla_{a_k}\RS(\vtheta)& ={\nu\eps}\frac{1}{n}\sum_{i=1}^n e_i \frac{\sigma(\eps{\vw}_k^\T\vx_i)}{\eps},\\
\nabla_{\vw_k}\RS(\vtheta) & = {\nu\eps}\frac{1}{n}\sum_{i=1}^ne_i   {a}_k\sigma^{(1)}(\eps {\vw}_k^\T\vx_i)\vx_i.
\end{aligned}
\end{equation*}
Finally, we obtain that
\begin{equation*}
\begin{aligned}
\frac{\D}{\D t}\RS(\vtheta)& =-\left(\sum_{k=1}^m\frac{\eps}{\nu}\left<\nabla_{a_k}\RS(\vtheta), \nabla_{a_k}\RS(\vtheta)\right>+\sum_{k=1}^m\frac{\nu}{\eps}\left<\nabla_{\vw_k}\RS(\vtheta), \nabla_{\vw_k}\RS(\vtheta)\right>\right) \\
& =-\frac{m}{n^2}\ve^{\T}\mG(\vtheta)\ve\\
& \leq -\frac{2m}{n}\lambda_{\min}\left(\mG(\vtheta)\right)\RS(\vtheta)\\
& \leq -\frac{m}{n}\nu^2\eps^2\left(\frac{\eps}{\nu}\lambda_a+\frac{\nu}{\eps}\lambda_{\vw}\right)\RS(\vtheta),
\end{aligned}\end{equation*}
and immediate integration yields the result.
\end{proof}
\begin{proposition}\label{append...prop...BoundofTraining}
Given any $\delta\in(0,1)$, under \Cref{Assumption....ActivationFunctions...NTK}, \Cref{Assump...Unparallel} and \Cref{assump...LimitExistence}, if $\gamma<1$, and
\[m=\max\left(\Omega\left(\frac{n^2}{\lambda^2}\log\frac{8n^2}{\delta}\right), \Omega\left(\left(\frac{n\sqrt{\RS\left(\vtheta^0\right)}}{\lambda}\right)^{\frac{1}{1-\gamma}}\right)\right),\]
then with probability at least $1-\delta$ over the choice of $\vtheta^0$, for any time $t\in[0, t^\ast)$ and for any $k\in [m]$, both
\begin{equation}
\begin{aligned}
\max\limits_{k\in[m]}\abs{a_k(t) - a_k(0)}& \leq 2\max\left\{\frac{\eps}{\nu}, 1\right\}\sqrt{2\log\frac{4m(d+1)}{\delta}}p, \\
\max\limits_{k\in[m]}\norm{\vw_k(t) - \vw_k(0)}_{\infty}
& \leq 2\max\left\{\frac{\nu}{\eps}, 1\right\}\sqrt{2\log\frac{4m(d+1)}{\delta}}p,
\end{aligned}\end{equation}
and
\begin{equation}
\max\limits_{k\in[m]}\{\abs{a_k(0)},\;\norm{\vw_k(0)}_{\infty}\} \leq \sqrt{2\log\frac{4m(d+1)}{\delta}},
\end{equation}
hold, where 
\[p := \frac{2\sqrt{2}dn\sqrt{\RS\left(\vtheta^0\right)}}{m\nu\eps\left(\frac{\eps}{\nu}\lambda_a+\frac{\nu}{\eps}\lambda_{\vw}\right)}.\]
\end{proposition}
\begin{proof}
Since
\begin{equation*}
\alpha(t)=\max\limits_{k\in[m],s\in[0,t]}\Abs{a_k(s)}, \quad \omega(t)=\max\limits_{k\in[m],s\in[0,t]}\Norm{\vw_k(s)}_{\infty},
\end{equation*}
we obtain
\begin{equation*}
\begin{aligned}
\Abs{\nabla_{a_k}\RS}^2 & =\left\lvert\frac{1}{n}\sum_{i=1}^n e_i\nu\sigma(\eps\vw_k^\T\vx_i)\right\rvert^2 \leq 2\norm{\eps\vw_k}^2_1\nu^2\RS(\vtheta) \leq 2d^2(\omega(t))^2\nu^2\eps^2\RS(\vtheta),\\
\norm{\nabla_{\vw_k}\RS}^2 & =\left\lVert\frac{1}{n}\sum_{i=1}^n e_i\nu\eps a_k\sigma^{(1)}(\eps\vw_k^\T\vx_i)\vx_i\right\rVert^2_{\infty}\leq 2\abs{a_k}^2\nu^2\eps^2\RS(\vtheta)\leq 2(\alpha(t))^2\nu^2\eps^2\RS(\vtheta).
\end{aligned}
\end{equation*}
By \Cref{append...prop...LeastEigenvalueatInitial} and \Cref{append...prop...InitialDecay}, we have if 
\[m\geq \frac{16n^2C_{\psi,d}^2}{\lambda^2C_0}\log\frac{8n^2}{\delta},\]
then with probability at least $1-\frac{\delta}{2}$ over the choice of $\vtheta^0$, we have that
\begin{equation*}
\begin{aligned}
\Abs{a_k(t) - a_k(0)}& \leq\frac{\eps}{\nu}\int_0^t\abs{\nabla_{a_k}\RS(\vtheta(s))}\diff{s} \\
& \leq{\sqrt{2}d\eps^2}\int_{0}^{t} \omega(s)\sqrt{\RS(\vtheta(s))}\diff{s}\\
& \leq{\sqrt{2}d\eps^2}\omega(t)\int_{0}^{t}\sqrt{\RS\left(\vtheta^0\right)}\exp\left(-\frac{m}{2n}\nu^2\eps^2\left(\frac{\eps}{\nu}\lambda_a+\frac{\nu}{\eps}\lambda_{\vw}\right)t\right)\diff{s}\\
&\leq{\sqrt{2}d\eps^2}\omega(t)\int_{0}^{\infty}\sqrt{\RS\left(\vtheta^0\right)}\exp\left(-\frac{m}{2n}\nu^2\eps^2\left(\frac{\eps}{\nu}\lambda_a+\frac{\nu}{\eps}\lambda_{\vw}\right)t\right)\diff{s} \\
& = \frac{2\sqrt{2}dn\sqrt{\RS\left(\vtheta^0\right)}}{m\nu^2\left(\frac{\eps}{\nu}\lambda_a+\frac{\nu}{\eps}\lambda_{\vw}\right)}\omega(t)=\frac{\eps}{\nu}{p}\omega(t),
\end{aligned}\end{equation*}
and 
\begin{equation*}
\begin{aligned}
\Norm{\vw_k(t) - \vw_k(0)}_{\infty} & \leq \frac{\nu}{\eps}\int_{0}^{t} \Norm{\nabla_{\vw_k}\RS(\vtheta(s))}_{\infty}\diff{s}\\
& \leq \sqrt{2}\nu^2 \int_{0}^t \alpha(s)\sqrt{\RS(\vtheta(s))} \diff{s}\\
& \leq \sqrt{2}\nu^2 \alpha(t) \int_{0}^{t}\sqrt{\RS\left(\vtheta^0\right)}\exp\left(-\frac{m}{2n}\nu^2\eps^2\left(\frac{\eps}{\nu}\lambda_a+\frac{\nu}{\eps}\lambda_{\vw}\right)t\right)\diff{s} \\
& \leq \sqrt{2}\nu^2 \alpha(t) \int_{0}^{\infty}\sqrt{\RS\left(\vtheta^0\right)}\exp\left(-\frac{m}{2n}\nu^2\eps^2\left(\frac{\eps}{\nu}\lambda_a+\frac{\nu}{\eps}\lambda_{\vw}\right)t\right)\diff{s}\\
& = \frac{2\sqrt{2}n\sqrt{\RS\left(\vtheta^0\right)}}{m\eps^2\left(\frac{\eps}{\nu}\lambda_a+\frac{\nu}{\eps}\lambda_{\vw}\right)} \alpha(t)  \leq \frac{2\sqrt{2}dn\sqrt{\RS\left(\vtheta^0\right)}}{m\eps^2\left(\frac{\eps}{\nu}\lambda_a+\frac{\nu}{\eps}\lambda_{\vw}\right)} \alpha(t) =\frac{\nu}{\eps}p\alpha(t).
\end{aligned}\end{equation*}
Thus
\begin{equation*}
\begin{aligned}
\alpha(t) & \leq\alpha(0)+\frac{\eps}{\nu}{p}\omega(t), \\
\omega(t) & \leq\omega(0)+\frac{\nu}{\eps}p\alpha(t),
 \end{aligned}
 \end{equation*}
moreover, by \Cref{append...lemma..Initialization},  we have with probability at least $1 - \frac{\delta}{2}$ over the choice of $\vtheta^0$, as
\begin{equation*}
        \max\limits_{k\in[m]}\{\abs{a_k(0)},\;\norm{\vw_k(0)}_{\infty}\}\leq \sqrt{2\log\frac{4m(d+1)}{\delta}},
\end{equation*}
then 
\begin{equation*}
        \alpha(0)\leq \sqrt{2\log\frac{4m(d+1)}{\delta}},\quad \omega(0)\leq \sqrt{2\log\frac{4m(d+1)}{\delta}}
\end{equation*}
Hence if
\begin{equation*}
m\geq  \left(\frac{2\sqrt{2}dn\sqrt{\RS\left(\vtheta^0\right)}}{\lambda}\right)^{\frac{1}{1-\gamma}},
\end{equation*}
we have
\begin{equation*}
 p=\frac{2\sqrt{2}dn\sqrt{\RS\left(\vtheta^0\right)}}{m\nu\eps\left(\frac{\eps}{\nu}\lambda_a+\frac{\nu}{\eps}\lambda_{\vw}\right)}\leq \frac{\sqrt{2}dn\sqrt{\RS\left(\vtheta^0\right)}}{m\nu\eps\lambda}\leq \frac{1}{2}.
\end{equation*}
Thus
\begin{equation*}
\alpha(t)  \leq\alpha(0)+\frac{\eps}{\nu}p\omega(0)+p^2\alpha(t),      
\end{equation*}
hence 
\[
\alpha(t) \leq\frac{4}{3}\alpha(0)+\frac{2}{3}\frac{\eps}{\nu}\omega(0).
\]
Therefore
\begin{equation*}
        \alpha(t)\leq 2\max\left\{\frac{\eps}{\nu}, 1\right\}\sqrt{2\log\frac{4m(d+1)}{\delta}}.
\end{equation*}
 Similarly, one can obtain the estimate of $\omega(t)$ as
\begin{equation*}
        \omega(t)\leq 2\max\left\{\frac{\nu}{\eps}, 1\right\}\sqrt{2\log\frac{4m(d+1)}{\delta}}.
    \end{equation*}
To sum up, for any $t\in[0, t^\ast)$, with probability at least $1-\delta$ over the choice of $\vtheta^0$,
 \begin{equation*}
\begin{aligned}
\max\limits_{k\in[m]}\Abs{a_k(t) - a_k(0)}
& \leq 2\max\left\{\frac{\eps}{\nu}, 1\right\}\sqrt{2\log\frac{4m(d+1)}{\delta}}p,\\
\max\limits_{k\in[m]}\Norm{\vw_k(t) - \vw_k(0)}_{\infty}
& \leq 2\max\left\{\frac{\nu}{\eps}, 1\right\}\sqrt{2\log\frac{4m(d+1)}{\delta}}p.
\end{aligned}\end{equation*}
\end{proof}
\begin{theorem}\label{append...thm..ThetaLazyRegime}
Given any $\delta\in(0,1)$, under \Cref{Assumption....ActivationFunctions...NTK}, \Cref{Assump...Unparallel} and \Cref{assump...LimitExistence}, if $\gamma<1$ and
\begin{equation}\label{eq...thm...ThetaLazy...WidthRequire}
\begin{aligned}
m=\max\Bigg(&\Omega\left(\frac{n^2}{\lambda^2}\log\frac{16n^2}{\delta}\right), \Omega\left(\left(\frac{n\sqrt{\RS\left(\vtheta^0\right)}}{\lambda}\right)^{\frac{1}{1-\gamma}}\right),\\
&\Omega\left(\left(\frac{n^2\sqrt{\RS\left(\vtheta^0\right)}}{\lambda^2}\right)^{\frac{1}{1-\gamma}}\right), \Omega\left(\log \frac{8}{\delta} \right)\Bigg),
\end{aligned}
\end{equation}
then with probability at least $1-\frac{\delta}{2}$ over the choice of $\vtheta^0$, we have for all time $t>0$,
\begin{equation}\label{eq...thm...ThetaLazy...ExponentialDecay}
    \RS(\vtheta(t))\leq \exp\left(-\frac{2m\nu^2\eps^2\lambda t}{n}\right)\RS\left(\vtheta^0\right),    
\end{equation}  
and with probability at least $1-\delta$ over the choice of $\vtheta^0$,
\begin{align}
&\lim_{m\to\infty}\sup\limits_{t\in[0,+\infty)}\frac{\norm{\vtheta_{\vw}(t)-\vtheta_{\vw}(0)}_2}{\norm{\vtheta_{\vw}(0)}_2}=0.\label{eq...thm...ThetaLazy...RelativeChange...InW}
\end{align}  
\end{theorem}
\begin{remark}
In this scenario, we also obtain that with probability at least $1-\delta$ over the choice of $\vtheta^0$,
\begin{equation}
\lim_{m\to\infty}\sup\limits_{t\in[0,+\infty)}\frac{\norm{\vtheta(t)-\vtheta(0)}_2}{\norm{\vtheta(0)}_2}=0.\label{eq...thm...ThetaLazy...RelativeChange...InTheta}   
\end{equation}    
\end{remark}
\begin{proof}
According to \Cref{append...prop...BoundofTraining}, it suffices to show that $t^\ast=\infty$.

\noindent (i). Firstly, from \Cref{append...prop...InitialDecay}, we have with probability at least $1-\frac{\delta}{2}$ over the choice of $\vtheta^0$, for any $t\in[0, t^\ast)$, the following holds
\begin{align*}
\Abs{a_k(t) - a_k(0)}&\leq \frac{\eps}{\nu}p\omega(t)\leq 2\max\left\{\frac{\eps}{\nu}, 1\right\}\xi p,\\
\Norm{\vw_k(t)-\vw_k(0)}_{\infty} & \leq \frac{\nu}{\eps}p\alpha(t)\leq 2\max\left\{\frac{\nu}{\eps},1\right\}\xi p,   
\end{align*}  
where
\[p= \frac{2\sqrt{2}dn\sqrt{\RS\left(\vtheta^0\right)}}{m\nu\eps\left(\frac{\eps}{\nu}\lambda_a+\frac{\nu}{\eps}\lambda_{\vw}\right)},\]
and 
\[
\xi:=\sqrt{2\log\frac{8m(d+1)}{\delta}}.
\]
For $m$ large enough, i.e.,
\begin{equation*}
m\geq  \left(\frac{2\sqrt{2}dn\sqrt{\RS\left(\vtheta^0\right)}}{\lambda}\right)^{\frac{1}{1-\gamma}},
\end{equation*}
then we have
\begin{equation*}
p=\frac{2\sqrt{2}dn\sqrt{\RS\left(\vtheta^0\right)}}{m\nu\eps\left(\frac{\eps}{\nu}\lambda_a+\frac{\nu}{\eps}\lambda_{\vw}\right)}\leq \frac{\sqrt{2}dn\sqrt{\RS\left(\vtheta^0\right)}}{m\nu\eps\lambda}\leq \frac{1}{2}.
\end{equation*}
It is noteworthy to emphasize that the demonstration of   \eqref{eq...thm...ThetaLazy...ExponentialDecay} requires solely the utilization of the aforementioned relations.

\noindent (ii).  Let
\begin{equation*}
g^{[a]}_{ij}(\vw) := \frac{1}{\eps^2}\sigma(\eps\vw^\T\vx_i)\sigma(\eps\vw^\T\vx_j),
\end{equation*}
then
\begin{equation*}
\Abs{G_{ij}^{[a]}(\vtheta(t)) - G_{ij}^{[a]}(\vtheta(0))} \leq \frac{\nu\eps^3}{m}\sum_{k=1}^m \Abs{g^{[a]}_{ij}(\vw_k(t)) - g^{[a]}_{ij}(\vw_k(0))}.
\end{equation*}
By mean value theorem, for some $c\in(0,1)$,
 \begin{equation*}
\Abs{g^{[a]}_{ij}(\vw_k(t)) - g^{[a]}_{ij}(\vw_k(0))} \leq \Norm{\nabla_{\vw} g^{[a]}_{ij}\left(c\vw_k(t) + (1-c)\vw_k(0)\right)}_{\infty}\norm{\vw_k(t) - \vw_k(0)}_1,
\end{equation*}
where
\begin{equation*}
\nabla_{\vw} g_{ij}^{[a]}(\vw)=\frac{1}{\eps}\sigma^{(1)}(\eps\vw^\T\vx_i)\sigma(\eps\vw^\T\vx_j)\vx_i+\frac{1}{\eps}\sigma(\eps\vw^\T\vx_i)\sigma^{(1)}(\eps\vw^\T\vx_j)\vx_j,
\end{equation*}
and
\begin{equation*}
\Norm{\nabla_{\vw} g_{ij}^{[a]}(\vw)}_{\infty}\leq 2\norm{\vw}_1.
\end{equation*}
then
\begin{equation*}
\Norm{\vw_k(t)-\vw_k(0)}_1\leq 2d\max\left\{\frac{\nu}{\eps},1\right\}\xi p,
\end{equation*}
thus, we have
\begin{equation*}
\begin{aligned}
\Norm{c\vw_k(t) + (1-c)\vw_k(0)}_1
& \leq d\left(\norm{\vw_k(0)}_{\infty} + \norm{\vw_k(t) - \vw_k(0)}_{\infty}\right) \\
& \leq d\left(\xi+ 2\max\left\{\frac{\nu}{\eps},1\right\}\xi p\right)\\
&\leq d\left(\xi+ \max\left\{\frac{\nu}{\eps},1\right\}\xi \right)\\
&\leq 2\max\left\{\frac{\nu}{\eps},1\right\}d\xi,
\end{aligned}
\end{equation*}
hence
\begin{align*}
\Abs{G_{ij}^{[a]}(\vtheta(t)) - G_{ij}^{[a]}(\vtheta(0))} &\leq \frac{\nu\eps^3}{m}\sum_{k=1}^m \Abs{g^{[a]}_{ij}(\vw_k(t)) - g^{[a]}_{ij}(\vw_k(0))}\\
&\leq  8d^2{\nu^2\eps^2}\xi^2\max\left\{\frac{\nu}{\eps},\frac{\eps}{\nu}\right\}p,
\end{align*}
then by using the same technique in \Cref{append...prop...LeastEigenvalueatInitial},
\begin{equation*}
\begin{aligned}
&\Norm{\mG^{[a]}(\vtheta(t)) - \mG^{[a]}(\vtheta(0))}_\mathrm{F}\\
&~~\leq 8d^2n{\nu^2\eps^2}\max\left\{\frac{\nu}{\eps},\frac{\eps}{\nu}\right\}\left(2\log\frac{8m(d+1)}{\delta}\right)\frac{2\sqrt{2}dn\sqrt{\RS\left(\vtheta^0\right)}}{m\nu\eps\left(\frac{\eps}{\nu}\lambda_a+\frac{\nu}{\eps}\lambda_{\vw}\right)} \\
& \leq\nu\eps\max\left\{\frac{\nu}{\eps},\frac{\eps}{\nu}\right\}\frac{32\sqrt{2}d^3n^2\left(\log\frac{8m(d+1)}{\delta}\right)\sqrt{\RS\left(\vtheta^0\right)}}{m\left(\frac{\eps}{\nu}\lambda_a+\frac{\nu}{\eps}\lambda_{\vw}\right)}.
\end{aligned}
\end{equation*}
We shall notice that as
\begin{align*}
\frac{1}{\lambda^2}
& \geq 16\left(\frac{1}{27}\right)^{\frac{1}{2}}\frac{1}{\left(\left(\frac{\eps}{\nu}\right)^{\frac{3}{2}}\lambda_a+\left(\frac{\nu}{\eps}\right)^{\frac{1}{2}}
\lambda_{\vw}\right)^2}\\
& = 16\left(\frac{1}{27}\right)^{\frac{1}{2}}\frac{\nu}{\eps}\frac{1}{\left(\frac{\eps}{\nu}\lambda_a+\frac{\nu}{\eps}\lambda_{\vw}\right)^2},
\end{align*} 
and 
\begin{align*}
\frac{1}{\lambda^2}
& \geq 16\left(\frac{1}{27}\right)^{\frac{1}{2}}\frac{1}{\left(\left(\frac{\eps}{\nu}\right)^{\frac{1}{2}}\lambda_a+\left(\frac{\nu}{\eps}\right)^{\frac{3}{2}}
\lambda_{\vw}\right)^2}\\
& = 16\left(\frac{1}{27}\right)^{\frac{1}{2}}\frac{\eps}{\nu}\frac{1}{\left(\frac{\eps}{\nu}\lambda_a+\frac{\nu}{\eps}\lambda_{\vw}\right)^2}.
\end{align*} 
As $\gamma<1$, then we may choose $m$ large enough, such that
\begin{align*}
m\nu\eps&\geq \frac{96\sqrt{2}d^3n^2\left(\log\frac{8m(d+1)}{\delta}\right)\sqrt{\RS\left(\vtheta^0\right)}}{\lambda^2}\\
&\geq \frac{256\sqrt{2}d^3n^2\left(\log\frac{8m(d+1)}{\delta}\right)\sqrt{\RS\left(\vtheta^0\right)}}{16\left(\frac{1}{27}\right)^{\frac{1}{2}}\lambda^2}\\
&\geq\max\left\{\frac{\nu}{\eps},\frac{\eps}{\nu}\right\}\frac{256\sqrt{2}d^3n^2\left(\log\frac{8m(d+1)}{\delta}\right)\sqrt{\RS\left(\vtheta^0\right)}}{\left(\frac{\eps}{\nu}\lambda_a+\frac{\nu}{\eps}\lambda_{\vw}\right)^2}.
\end{align*}
Then for any $t\in[0, t^\ast)$,
\begin{equation}\label{eqproof....thm...ThetaLazy...PartOne}
\Norm{\mG^{[a]}(\vtheta(t)) - \mG^{[a]}(\vtheta(0))}_\mathrm{F}\leq\frac{1}{8}\nu^2\eps^2\left(\frac{\eps}{\nu}\lambda_a+\frac{\nu}{\eps}\lambda_{\vw}\right).
\end{equation}

\noindent (iii). As for  $\mG^{[\vw]}(\vtheta(t))$, we observe that directly from \Cref{append...prop...InitialDecay},  for any $t\in[0, t^\ast)$, the following holds
\begin{align*}
&\Abs{G^{[\vw]}_{ij}(\vtheta(t))-G^{[\vw]}_{ij}(\vtheta(0))}\\
& \leq \frac{\nu^3\eps\Abs{\left<\vx_i, \vx_j\right>}}{m}\Bigg(\sum_{k=1}^m\Big|a_k^2(t)\sigma^{(1)}\left(\left<\vw_k(t), \vx_i\right>\right)\sigma^{(1)}\left(\left<\vw_k(t), \vx_j\right>\right)\\
&~~~~~~~~~~~~~~~~~~~~~~~~~~~~~~-a_k^2(0)\sigma^{(1)}\left(\left<\vw_k(0), \vx_i\right>\right)\sigma^{(1)}\left(\left<\vw_k(0), \vx_j\right>\right)\Big|\Bigg),
\end{align*}
we define a new quantity
\begin{equation*}
  D_{k,i,j}(t):=  \sigma^{(1)}\left(\left<\vw_k(t), \vx_i\right>\right)\sigma^{(1)}\left(\left<\vw_k(t), \vx_j\right>\right)-\sigma^{(1)}\left(\left<\vw_k(0), \vx_i\right>\right)\sigma^{(1)}\left(\left<\vw_k(0), \vx_j\right>\right),
\end{equation*}
then we obtain that 
\begin{align*}
&\Abs{G^{[\vw]}_{ij}(\vtheta(t))-G^{[\vw]}_{ij}(\vtheta(0))}\\
& \leq \frac{\nu^3\eps}{m}\sum_{k=1}^m\left(a_k^2(t)D_{k,i,j}(t)+\Abs{a_k^2(t) -a_k^2(0)}\right).
\end{align*}
We shall make an estimate on the quantity $D_{k,i,j}(t)$:
\begin{align*}
D_{k,i,j}(t)&=\sigma^{(1)}\left(\left<\vw_k(t), \vx_i\right>\right)\sigma^{(1)}\left(\left<\vw_k(t), \vx_j\right>\right)-\sigma^{(1)}\left(\left<\vw_k(t), \vx_i\right>\right)\sigma^{(1)}\left(\left<\vw_k(0), \vx_j\right>\right)\\
&~~+\sigma^{(1)}\left(\left<\vw_k(t), \vx_i\right>\right)\sigma^{(1)}\left(\left<\vw_k(0), \vx_j\right>\right)-\sigma^{(1)}\left(\left<\vw_k(0), \vx_i\right>\right)\sigma^{(1)}\left(\left<\vw_k(0), \vx_j\right>\right)\\
&\leq \Abs{\left<\vw_k(t), \vx_i\right>-\left<\vw_k(0), \vx_i\right>}+\Abs{\left<\vw_k(t), \vx_j\right>-\left<\vw_k(0), \vx_j\right>}\\
&\leq 2\Norm{\vw_k(t)-\vw_k(0)}_1\leq 2d\Norm{\vw_k(t)-\vw_k(0)}_{\infty}\\
&\leq 4d\max\left\{\frac{\nu}{\eps},1\right\}\xi p.
\end{align*}
Hence, we observe that
\begin{align*}
\Abs{a_k^2(t) -a_k^2(0)}&\leq \Abs{a_k(t)-a_k(0)}^2+2\Abs{a_k(0)}\Abs{a_k(t)-a_k(0)}\\
&\leq\left( 2\max\left\{\frac{\eps}{\nu}, 1\right\}\xi p\right)^2+2\xi\left( 2\max\left\{\frac{\eps}{\nu}, 1\right\}\xi p\right)\\
&\leq 6\max\left\{\left(\frac{\eps}{\nu}\right)^2, 1\right\}\xi^2p,
\end{align*}
and consequently
\begin{align*}
a_k(t)^2&\leq\Abs{a_k^2(t) -a_k^2(0)}+a_k^2(0)\\
&\leq 6\xi^2\max\left\{\left(\frac{\eps}{\nu}\right)^2, 1\right\}p+\xi^2\\
&\leq 3\xi^2\max\left\{\left(\frac{\eps}{\nu}\right)^2, 1\right\}+\xi^2\\
&\leq 4\max\left\{\left(\frac{\eps}{\nu}\right)^2, 1\right\}\xi^2.
\end{align*}
Thus, we have
\begin{align*}
&\Abs{G^{[\vw]}_{ij}(\vtheta(t))-G^{[\vw]}_{ij}(\vtheta(0))}\\
& \leq \frac{\nu^3\eps}{m}\sum_{k=1}^m\left(a_k^2(t)D_{k,i,j}(t)+\Abs{a_k^2(t) -a_k^2(0)}\right)\\
& \leq \nu^3\eps\left(16d\max\left\{\left(\frac{\eps}{\nu}\right)^2, \frac{\nu}{\eps}\right\}\xi^3p+ 6\max\left\{\left(\frac{\eps}{\nu}\right)^2, 1\right\}\xi^2p\right)\\
& \leq\nu^3\eps\left(16d\max\left\{\left(\frac{\eps}{\nu}\right)^2, \frac{\nu}{\eps}\right\}\xi^3p+ 6d\max\left\{\left(\frac{\eps}{\nu}\right)^2, 1\right\}\xi^2p\right)\\
& \leq 24\nu^2\eps^2 d \left(2\log\frac{8m(d+1)}{\delta}\right)^{\frac{3}{2}}\max\left\{\left(\frac{\nu}{\eps}\right)^2, \frac{\eps}{\nu}\right\}p,
\end{align*}
then by using the same technique in \Cref{append...prop...LeastEigenvalueatInitial},
\begin{align*}
&\Norm{G^{[\vw]}(\vtheta(t))-G^{[\vw]}(\vtheta(0))}_\mathrm{F}\\
&\leq n \Norm{G^{[\vw]}_{ij}(\vtheta(t))-G^{[\vw]}_{ij}(\vtheta(0))}_{\infty}\\       
& \leq 24\nu^2\eps^2 d \left(2\log\frac{8m(d+1)}{\delta}\right)^{\frac{3}{2}}\max\left\{\left(\frac{\nu}{\eps}\right)^2, \frac{\eps}{\nu}\right\}np\\
&=\nu\eps\max\left\{\left(\frac{\nu}{\eps}\right)^2, \frac{\eps}{\nu}\right\}\frac{96d^2n^2\left(\log\frac{8m(d+1)}{\delta}\right)^{\frac{3}{2}}\sqrt{\RS\left(\vtheta^0\right)}}{m\left(\frac{\eps}{\nu}\lambda_a+\frac{\nu}{\eps}\lambda_{\vw}\right)},
\end{align*}
We   notice that as
\begin{align*}
\frac{1}{\lambda^2}
& \geq \frac{\left(\frac{\nu}{\eps}\right)^2}{\left(\frac{\nu}{\eps}
\lambda_{\vw}\right)^2} = \left(\frac{\nu}{\eps}\right)^2\frac{1}{\left(\frac{\eps}{\nu}\lambda_a+\frac{\nu}{\eps}\lambda_{\vw}\right)^2},
\end{align*} 
and 
\begin{align*}
\frac{1}{\lambda^2}
& \geq 16\left(\frac{1}{27}\right)^{\frac{1}{2}}\frac{1}{\left(\left(\frac{\eps}{\nu}\right)^{\frac{1}{2}}\lambda_a+\left(\frac{\nu}{\eps}\right)^{\frac{3}{2}}
\lambda_{\vw}\right)^2}\\
& = 16\left(\frac{1}{27}\right)^{\frac{1}{2}}\frac{\eps}{\nu}\frac{1}{\left(\frac{\eps}{\nu}\lambda_a+\frac{\nu}{\eps}\lambda_{\vw}\right)^2}.
\end{align*} 
As $\gamma<1$, then we may choose $m$ large enough, such that
\begin{align*}
m\nu\eps&\geq \frac{384d^2n^2\left(\log\frac{8m(d+1)}{\delta}\right)^{\frac{3}{2}}\sqrt{\RS\left(\vtheta^0\right)}}{\lambda^2}\\
&\geq \max\left\{\frac{384d^2n^2\left(\log\frac{8m(d+1)}{\delta}\right)^{\frac{3}{2}}\sqrt{\RS\left(\vtheta^0\right)}}{\lambda^2}, \frac{384d^2n^2\left(\log\frac{8m(d+1)}{\delta}\right)^{\frac{3}{2}}\sqrt{\RS\left(\vtheta^0\right)}}{16\left(\frac{1}{27}\right)^{\frac{1}{2}}\lambda^2}\right\}\\
&\geq\max\left\{\left(\frac{\nu}{\eps}\right)^2, \frac{\eps}{\nu}\right\}\frac{d^2n^2\left(\log\frac{8m(d+1)}{\delta}\right)^{\frac{3}{2}}\sqrt{\RS\left(\vtheta^0\right)}}{\left(\frac{\eps}{\nu}\lambda_a+\frac{\nu}{\eps}\lambda_{\vw}\right)^2}.
\end{align*}
Then for any $t\in[0, t^\ast)$,
\begin{equation}\label{eqproof....thm...ThetaLazy...PartTwo}
\Norm{\mG^{[\vw]}(\vtheta(t)) - \mG^{[\vw]}(\vtheta(0))}_\mathrm{F}\leq\frac{1}{8}\nu^2\eps^2\left(\frac{\eps}{\nu}\lambda_a+\frac{\nu}{\eps}\lambda_{\vw}\right).
\end{equation}
To sum up, for $t\in[0,t^\ast)$, the following holds
\begin{equation*}
\RS(\vtheta(t)) \leq \exp\left(-\frac{m}{n}\nu^2\eps^2\left(\frac{\eps}{\nu}\lambda_a+\frac{\nu}{\eps}\lambda_{\vw}\right)t\right)\RS\left(\vtheta^0\right).
\end{equation*}
Suppose we have $t^\ast<+\infty$, then one can take the limit $t\to t^\ast$ in~\eqref{eqproof....thm...ThetaLazy...PartOne} and~\eqref{eqproof....thm...ThetaLazy...PartTwo}, which  leads to a contradiction with the definition of $t^\ast$. Therefore $t^\ast=+\infty$.   

Directly from \Cref{append...prop..UpperBoundandLowerBoundInitial}, we have with probability at least $1-\frac{\delta}{2}$ over the choice of $\vtheta^0$,
 \begin{equation*}
\Norm{\vtheta_{\vw}^0}_2\geq \sqrt{\frac{md}{2}},
\end{equation*}
thus we have
\begin{align*}
\sup\limits_{t\in[0,+\infty)}\frac{\norm{\vtheta_{\vw}(t)-\vtheta_{\vw}(0)}_2}{\norm{\vtheta_{\vw}(0)}_2}
& \leq \sqrt{\frac{2}{md}}\sup\limits_{t\in[0,+\infty)}\Norm{\vtheta_{\vw}(t)-\vtheta_{\vw}(0)}_2,
\end{align*}
via \Cref{append...prop...BoundofTraining},  with probability at least $1-\frac{\delta}{2}$ over the choice of $\vtheta^0$,
\begin{equation*}
\begin{aligned}       
\max\limits_{k\in[m]}\Norm{\vw_k(t) - \vw_k(0)}_{\infty}
& \leq 2\max\left\{\frac{\nu}{\eps}, 1\right\}\xi  p,
\end{aligned}\end{equation*}
then 
\begin{align*}
\Norm{\vtheta_{\vw}(t)-\vtheta_{\vw}(0)}_2&\leq \left[\sum_{k=1}^m\left(\Norm{\vw_k(t) - \vw_k(0)}_{\infty}^2\right)\right]^{\frac{1}{2}}\\
&\leq 2\max\left\{\frac{\nu}{\eps}, 1\right\}\sqrt{md}\xi p\\
&\leq  2\max\left\{\frac{\nu}{\eps}, 1\right\}\sqrt{md}\sqrt{2\log\frac{8m(d+1)}{\delta}}\frac{2\sqrt{2}dn\sqrt{\RS\left(\vtheta^0\right)}}{m\nu\eps\left(\frac{\eps}{\nu}\lambda_a+\frac{\nu}{\eps}\lambda_{\vw}\right)}\\
&\leq 8\max\left\{\frac{\nu}{\eps}, 1\right\}\sqrt{\log\frac{8m(d+1)}{\delta}}\frac{d^{\frac{3}{2}}n\sqrt{\RS\left(\vtheta^0\right)}}{\sqrt{m}\nu\eps\left(\frac{\eps}{\nu}\lambda_a+\frac{\nu}{\eps}\lambda_{\vw}\right)},
\end{align*}
hence
\begin{align*}
\sup\limits_{t\in[0,+\infty)}\frac{\norm{\vtheta_{\vw}(t)-\vtheta_{\vw}(0)}_2}{\norm{\vtheta_{\vw}(0)}_2}
& \leq \sqrt{\frac{2}{md}}\sup\limits_{t\in[0,+\infty)}\Norm{\vtheta_{\vw}(t)-\vtheta_{\vw}(0)}_2\\
&\leq 8\sqrt{2}\max\left\{\frac{\nu}{\eps}, 1\right\}\sqrt{\log\frac{8m(d+1)}{\delta}}\frac{dn\sqrt{\RS\left(\vtheta^0\right)}}{{m}\nu\eps\left(\frac{\eps}{\nu}\lambda_a+\frac{\nu}{\eps}\lambda_{\vw}\right)}\\
&\leq 4\sqrt{2}\sqrt{\log\frac{8m(d+1)}{\delta}}\frac{dn\sqrt{\RS\left(\vtheta^0\right)}}{{m}\nu\eps\lambda},
\end{align*}
as $\gamma<1$, then we obtain that
\begin{equation}
 \sup\limits_{t\in[0,+\infty)}\frac{\norm{\vtheta_{\vw}(t)-\vtheta_{\vw}(0)}_2}{\norm{\vtheta_{\vw}(0)}_2}\lesssim   \frac{\sqrt{\log\frac{8m(d+1)}{\delta}}}{m^{1-\gamma}},
\end{equation}
hence
\begin{equation}
\lim_{m\to\infty}\sup\limits_{t\in[0,+\infty)}\frac{\norm{\vtheta_{\vw}(t)-\vtheta_{\vw}(0)}_2}{\norm{\vtheta_{\vw}(0)}_2}=0.
\end{equation}
Similarly, 
directly from \Cref{append...prop..UpperBoundandLowerBoundInitial}, we have with probability at least $1-\frac{\delta}{2}$ over the choice of $\vtheta^0$,
 \begin{equation*}
\Norm{\vtheta^0}_2\geq \sqrt{\frac{m(d+1)}{2}},
\end{equation*}
thus we have
\begin{align*}
\sup\limits_{t\in[0,+\infty)}\frac{\norm{\vtheta(t)-\vtheta(0)}_2}{\norm{\vtheta(0)}_2}
& \leq \sqrt{\frac{2}{m(d+1)}}\sup\limits_{t\in[0,+\infty)}\Norm{\vtheta(t)-\vtheta(0)}_2,
\end{align*}
via \Cref{append...prop...BoundofTraining},  with probability at least $1-\frac{\delta}{2}$ over the choice of $\vtheta^0$,
\begin{equation*}
\begin{aligned}
\max\limits_{k\in[m]}\Abs{a_k(t) - a_k(0)}& \leq 2\max\left\{\frac{\eps}{\nu}, 1\right\}\xi p, \\
\max\limits_{k\in[m]}\Norm{\vw_k(t) - \vw_k(0)}_{\infty}
& \leq 2\max\left\{\frac{\nu}{\eps}, 1\right\}\xi  p,
\end{aligned}\end{equation*}
then 
\begin{align*}
\Norm{\vtheta(t)-\vtheta(0)}_2&\leq \left[\sum_{k=1}^m\left(\Abs{a_k(t) - a_k(0)}^2+\Norm{\vw_k(t) - \vw_k(0)}_{\infty}^2\right)\right]^{\frac{1}{2}}\\
&\leq 2\max\left\{\frac{\eps}{\nu}, \frac{\nu}{\eps}\right\}\sqrt{m(d+1)}\xi p\\
&\leq  2\max\left\{\frac{\eps}{\nu}, \frac{\nu}{\eps}\right\}\sqrt{m(d+1)}\sqrt{2\log\frac{8m(d+1)}{\delta}}\frac{2\sqrt{2}dn\sqrt{\RS\left(\vtheta^0\right)}}{m\nu\eps\left(\frac{\eps}{\nu}\lambda_a+\frac{\nu}{\eps}\lambda_{\vw}\right)}\\
&\leq 8\max\left\{\frac{\eps}{\nu}, \frac{\nu}{\eps}\right\}\sqrt{\log\frac{8m(d+1)}{\delta}}\frac{\sqrt{d+1}dn\sqrt{\RS\left(\vtheta^0\right)}}{\sqrt{m}\nu\eps\left(\frac{\eps}{\nu}\lambda_a+\frac{\nu}{\eps}\lambda_{\vw}\right)},
\end{align*}
hence
\begin{align*}
\sup\limits_{t\in[0,+\infty)}\frac{\norm{\vtheta(t)-\vtheta(0)}_2}{\norm{\vtheta(0)}_2}
& \leq \sqrt{\frac{2}{m(d+1)}}\sup\limits_{t\in[0,+\infty)}\Norm{\vtheta(t)-\vtheta(0)}_2\\
&\leq 8\sqrt{2}\max\left\{\frac{\eps}{\nu}, \frac{\nu}{\eps}\right\}\sqrt{\log\frac{8m(d+1)}{\delta}}\frac{dn\sqrt{\RS\left(\vtheta^0\right)}}{{m}\nu\eps\left(\frac{\eps}{\nu}\lambda_a+\frac{\nu}{\eps}\lambda_{\vw}\right)}\\
&\leq 8\sqrt{2}\sqrt{\log\frac{8m(d+1)}{\delta}}\frac{dn\sqrt{\RS\left(\vtheta^0\right)}}{{m}\nu\eps\lambda},
\end{align*}
as $\gamma<1$, then we obtain that
\begin{equation}
 \sup\limits_{t\in[0,+\infty)}\frac{\norm{\vtheta(t)-\vtheta(0)}_2}{\norm{\vtheta(0)}_2}\lesssim   \frac{\sqrt{\log\frac{8m(d+1)}{\delta}}}{m^{1-\gamma}},
\end{equation}
moreover,
\begin{equation}
\lim_{m\to\infty}\sup\limits_{t\in[0,+\infty)}\frac{\norm{\vtheta(t)-\vtheta(0)}_2}{\norm{\vtheta(0)}_2}=0.
\end{equation}

\end{proof}
\subsection{$\vw$-lazy Regime}\label{subsec...theta_w}
We denote
\begin{equation}
    t_a^\ast = \inf\{t \mid \vtheta(t)\notin \mathcal{N}\left(\vtheta^0\right)\},
\end{equation}
where the event is defined as 
\begin{equation}
    \mathcal{N}_a\left(\vtheta^0\right) := \left\{\vtheta \mid \norm{\mG^{[a]}(\vtheta) - \mG^{[a]}\left(\vtheta^0\right)}_\mathrm{F}\leq \frac{1}{4}\nu\eps^3\lambda_a\right\}.
\end{equation}
We observe immediately that 
the event $\mathcal{N}_a\left(\vtheta^0\right)\neq \varnothing$, since $\vtheta^0\in\mathcal{N}_a\left(\vtheta^0\right)$.
Recall that 
\begin{equation*}
        \lambda_a = \lambda_{\min}\left(\mK^{[a]}\right),
\end{equation*}
whose definition can be found in \Cref{append...thm...LeastEigenGram}.
\begin{proposition}\label{append...prop...InitialDecay...WLazy}
Given  any $\delta\in(0,1)$,  under \Cref{Assumption....ActivationFunctions...NTK}, \Cref{Assump...Unparallel} and \Cref{assump...LimitExistence}, if 
\[m=\Omega\left(\frac{n^2}{\lambda_a^2}\log\frac{4n^2}{\delta}\right),\]
then with probability at least $1-\delta$ over the choice of $\vtheta^0$,  we have for any time $t\in[0, t_a^\ast)$,
\begin{equation}
\RS(\vtheta(t)) \leq \exp\left(-\frac{m}{n}\nu\eps^3\lambda_a t\right)\RS\left(\vtheta^0\right).
\end{equation}
\end{proposition}
\begin{proof}
Similar to the proof in \Cref{append...prop...InitialDecay},
 with probability at least $1-\delta$ over the choice of $\vtheta^0$ and for any $\vtheta\in\mathcal{N}\left(\vtheta^0\right)$, we have
\begin{equation*}
\begin{aligned}
\lambda_{\min}\left(\mG(\vtheta)\right)& \geq \lambda_{\min}\left(\mG^{[a]}(\vtheta)\right)\\
&\geq\lambda_{\min}\left(\mG^{[a]}\left(\vtheta^0\right)\right) - \norm{\mG^{[a]}(\vtheta) - \mG^{[a]}\left(\vtheta^0\right)}_\mathrm{F}\\
& \geq \frac{1}{2}\nu^3\eps\lambda_a.
\end{aligned}\end{equation*}
Finally, we obtain that
\begin{equation*}
\begin{aligned}
\frac{\D}{\D t}\RS(\vtheta(t))&  =-\frac{m}{n^2}\ve^{\T}\mG(\vtheta(t))\ve\\
& \leq -\frac{2m}{n}\lambda_{\min}\left(\mG(\vtheta(t))\right)\RS(\vtheta(t))\\
& \leq -\frac{m}{n}\nu^3\eps \lambda_a \RS(\vtheta(t)),
\end{aligned}\end{equation*}
    and immediate integration yields the result.
\end{proof}
\begin{proposition}\label{append...prop...BoundofTraining...WLazy}
Given any $\delta\in(0,1)$, under \Cref{Assumption....ActivationFunctions...NTK}, \Cref{Assump...Unparallel} and \Cref{assump...LimitExistence}, if $\gamma> 1$, $\gamma'> \gamma-1$, and
\[m=\max\left(\Omega\left(\frac{n^2}{\lambda_a^2}\log\frac{8n^2}{\delta}\right), \Omega\left(\frac{n\sqrt{\RS\left(\vtheta^0\right)}}{\lambda_a}\right)^{\frac{1}{1-\gamma+\gamma'}}\right),\]
then with probability at least $1-\delta$ over the choice of $\vtheta^0$, for any time $t\in[0, t_a^\ast)$ and for any $k\in [m]$, both
\begin{equation}
\begin{aligned}
\max\limits_{k\in[m]}\abs{a_k(t) - a_k(0)}& \leq 2\frac{\eps}{\nu}\sqrt{2\log\frac{4m(d+1)}{\delta}}p_a, \\
\max\limits_{k\in[m]}\norm{\vw_k(t) - \vw_k(0)}_{\infty}
& \leq 2\sqrt{2\log\frac{4m(d+1)}{\delta}}p_a,
\end{aligned}\end{equation}
and
\begin{equation}
\max\limits_{k\in[m]}\{\abs{a_k(0)},\;\norm{\vw_k(0)}_{\infty}\} \leq \sqrt{2\log\frac{4m(d+1)}{\delta}},
\end{equation}
hold, where 
\[p_a := \frac{2\sqrt{2}dn\sqrt{\RS\left(\vtheta^0\right)}}{m\eps^2\lambda_a}.\]
\end{proposition}
\begin{proof}
Since
\begin{equation*}
        \alpha(t)=\max\limits_{k\in[m],s\in[0,t]}\Abs{a_k(s)}, \quad \omega(t)=\max\limits_{k\in[m],s\in[0,t]}\Norm{\vw_k(s)}_{\infty},
\end{equation*}
we obtain
\begin{equation*}
\begin{aligned}
\Abs{\nabla_{a_k}\RS}^2 & =\left\lvert\frac{1}{n}\sum_{i=1}^n e_i\nu\sigma(\eps\vw_k^\T\vx_i)\right\rvert^2 \leq 2\norm{\eps\vw_k}^2_1\nu^2\RS(\vtheta) \leq 2d^2(\omega(t))^2\nu^2\eps^2\RS(\vtheta),\\
\norm{\nabla_{\vw_k}\RS}^2 & =\left\lVert\frac{1}{n}\sum_{i=1}^n e_i\nu\eps a_k\sigma^{(1)}(\eps\vw_k^\T\vx_i)\vx_i\right\rVert^2_{\infty}\leq 2\abs{a_k}^2\nu^2\eps^2\RS(\vtheta)\leq 2(\alpha(t))^2\nu^2\eps^2\RS(\vtheta).
\end{aligned}
\end{equation*}
By \Cref{append...prop...InitialDecay...WLazy}, we have if 
\[m\geq \frac{16n^2d^2C_{\psi,d}^2}{\lambda^2C_0}\log\frac{8n^2}{\delta},\]
then with probability at least $1-\frac{\delta}{2}$ over the choice of $\vtheta^0$, we have that
\begin{align*}
\Abs{a_k(t) - a_k(0)}& \leq\frac{\eps}{\nu}\int_0^t\abs{\nabla_{a_k}\RS(\vtheta(s))}\diff{s} \\
& \leq{\sqrt{2}d\eps^2}\int_{0}^{t} \omega(s)\sqrt{\RS(\vtheta(s))}\diff{s}\\
&\leq{\sqrt{2}d\eps^2}\omega(t)\int_{0}^{t}\sqrt{\RS\left(\vtheta^0\right)}\exp\left(-\frac{m}{2n}\nu^2\eps^2\left(\frac{\eps}{\nu}\lambda_a+\frac{\nu}{\eps}\lambda_{\vw}\right)t\right)\diff{s}\\
&\leq{\sqrt{2}d\eps^2}\omega(t)\int_{0}^{\infty}\sqrt{\RS\left(\vtheta^0\right)}\exp\left(-\frac{m}{2n}\nu^2\eps^2\left(\frac{\eps}{\nu}\lambda_a+\frac{\nu}{\eps}\lambda_{\vw}\right)t\right)\diff{s} \\
& = \frac{2\sqrt{2}dn\sqrt{\RS\left(\vtheta^0\right)}}{m\nu^2\left(\frac{\eps}{\nu}\lambda_a+\frac{\nu}{\eps}\lambda_{\vw}\right)}\omega(t)\\
&\leq \frac{2\sqrt{2}dn\sqrt{\RS\left(\vtheta^0\right)}}{m\nu\eps\lambda_a}\omega(t)\\
&\leq\frac{\eps}{\nu}p_a\omega(t),
\end{align*}
and similarly
\begin{equation*}
\begin{aligned}
\Norm{\vw_k(t) - \vw_k(0)}_{\infty} & \leq \frac{\nu}{\eps}\int_{0}^{t} \Norm{\nabla_{\vw_k}\RS(\vtheta(s))}_{\infty}\diff{s}\\
& \leq \sqrt{2}\nu^2 \int_{0}^t \alpha(s)\sqrt{\RS(\vtheta(s))} \diff{s}\\
& \leq \sqrt{2}\nu^2 \alpha(t) \int_{0}^{t}\sqrt{\RS\left(\vtheta^0\right)}\exp\left(-\frac{m}{2n}\nu^2\eps^2\left(\frac{\eps}{\nu}\lambda_a+\frac{\nu}{\eps}\lambda_{\vw}\right)t\right)\diff{s} \\
& \leq \sqrt{2}\nu^2 \alpha(t) \int_{0}^{\infty}\sqrt{\RS\left(\vtheta^0\right)}\exp\left(-\frac{m}{2n}\nu^2\eps^2\left(\frac{\eps}{\nu}\lambda_a+\frac{\nu}{\eps}\lambda_{\vw}\right)t\right)\diff{s}\\
& = \frac{2\sqrt{2}n\sqrt{\RS\left(\vtheta^0\right)}}{m\eps^2\left(\frac{\eps}{\nu}\lambda_a+\frac{\nu}{\eps}\lambda_{\vw}\right)} \alpha(t) \\
&\leq \frac{2\sqrt{2}dn\sqrt{\RS\left(\vtheta^0\right)}}{m\frac{\eps^3}{\nu}\lambda_a}\alpha(t)\\
&\leq\frac{\nu}{\eps}p_a\alpha(t).
\end{aligned}\end{equation*}
Thus
\begin{equation*}
\begin{aligned}
\alpha(t) & \leq\alpha(0)+\frac{\eps}{\nu}p_a\omega(t), \\
\omega(t) & \leq\omega(0)+\frac{\nu}{\eps}p_a\alpha(t),
 \end{aligned}
 \end{equation*}
moreover, by \Cref{append...lemma..Initialization}, with probability at least $1 - \frac{\delta}{2}$ over the choice of $\vtheta^0$,
\begin{equation*}
\max\limits_{k\in[m]}\{\abs{a_k(0)},\;\norm{\vw_k(0)}_{\infty}\}\leq \sqrt{2\log\frac{4m(d+1)}{\delta}}.
\end{equation*}
Then, if
\begin{equation*}
m\geq  \left(\frac{4\sqrt{2}dn\sqrt{\RS\left(\vtheta^0\right)}}{\lambda_a}\right)^{\frac{1}{1-\gamma+\gamma'}},
\end{equation*}
we have
\begin{equation*}
 p_a= \frac{2\sqrt{2}dn\sqrt{\RS\left(\vtheta^0\right)}}{m\eps^2\lambda_{a}} \leq\frac{1}{2}.
\end{equation*}
Therefore
\begin{equation*}
        \alpha(t)\leq 2\max\left\{\frac{\eps}{\nu}, 1\right\}\sqrt{2\log\frac{4m(d+1)}{\delta}}=2\frac{\eps}{\nu}\sqrt{2\log\frac{4m(d+1)}{\delta}}.
\end{equation*}
 Similarly, one can obtain the estimate of $\omega(t)$ as
\begin{equation*}
        \omega(t)\leq 2\max\left\{\frac{\nu}{\eps}, 1\right\}\sqrt{2\log\frac{4m(d+1)}{\delta}}=2\sqrt{2\log\frac{4m(d+1)}{\delta}}.
    \end{equation*}
Hence, for any $t\in[0, t_a^\ast)$, with probability at least $1-\delta$ over the choice of $\vtheta^0$,
 \begin{equation*}
\begin{aligned}
\max\limits_{k\in[m]}\Abs{a_k(t) - a_k(0)}
& \leq 2\frac{\eps}{\nu}\sqrt{2\log\frac{4m(d+1)}{\delta}}p_a,\\
\max\limits_{k\in[m]}\Norm{\vw_k(t) - \vw_k(0)}_{\infty}
& \leq 2\sqrt{2\log\frac{4m(d+1)}{\delta}}p_a.
\end{aligned}\end{equation*}
\end{proof}
\begin{theorem}\label{append...thm..ThetaLazyRegime...WLazy}
Given any $\delta\in(0,1)$, under \Cref{Assumption....ActivationFunctions...NTK}, \Cref{Assump...Unparallel} and \Cref{assump...LimitExistence}, if $\gamma\geq 1$, $\gamma'> \gamma-1$, and
\begin{equation}
\begin{aligned}
m=\max\Bigg(&\Omega\left(\frac{n^2}{\lambda_a^2}\log\frac{16n^2}{\delta}\right), \Omega\left(\left(\frac{n\sqrt{\RS\left(\vtheta^0\right)}}{\lambda_a}\right)^{\frac{1}{1-\gamma+\gamma'}}\right),\\
&\Omega\left(\left(\frac{n^2\sqrt{\RS\left(\vtheta^0\right)}}{\lambda_a^2}\right)^{\frac{1}{1-\gamma+\gamma'}}\right),  \Omega\left(\log \frac{8}{\delta} \right)\Bigg),
\end{aligned}
\end{equation}
then with probability at least $1-\frac{\delta}{2}$ over the choice of $\vtheta^0$, we have for all time $t>0$,
\begin{equation}\label{eq...thm...WLazy...ExponentialDecay}
    \RS(\vtheta(t))\leq \exp\left(-\frac{m\nu\eps^3\lambda_a t}{n}\right)\RS\left(\vtheta^0\right),    
\end{equation}  
and with probability at least $1-\delta$ over the choice of $\vtheta^0$,
\begin{align}
&\lim_{m\to\infty}\sup\limits_{t\in[0,+\infty)}\frac{\norm{\vtheta_{\vw}(t)-\vtheta_{\vw}(0)}_2}{\norm{\vtheta_{\vw}(0)}_2}=0.\label{eq...thm...WLazy...RelativeChange...InW}
\end{align}  
\end{theorem}
\begin{proof}
According to \Cref{append...prop...BoundofTraining...WLazy}, it suffices to show that $t_a^\ast=\infty$. 

(i). Firstly, from \Cref{append...prop...InitialDecay...WLazy}, we have with probability at least $1-\frac{\delta}{2}$ over the choice of $\vtheta^0$, for any $t\in[0, t_a^\ast)$, the following holds
\begin{align*}
\Abs{a_k(t) - a_k(0)}&\leq  2\frac{\eps}{\nu}\xi p_a,\\
\Norm{\vw_k(t)-\vw_k(0)}_{\infty} & \leq  2\xi p_a,   
\end{align*}  
where
\[p_a= \frac{2\sqrt{2}dn\sqrt{\RS\left(\vtheta^0\right)}}{m\eps^2\lambda_a},\]
and 
\[
\xi=\sqrt{2\log\frac{8m(d+1)}{\delta}}.
\]
For $m$ large enough, i.e.,
\begin{equation*}
m\geq  \left(\frac{4\sqrt{2}dn\sqrt{\RS\left(\vtheta^0\right)}}{\lambda_a}\right)^{\frac{1}{1-\gamma+\gamma'}},
\end{equation*}
we have
\begin{equation*}
 p_a= \frac{2\sqrt{2}dn\sqrt{\RS\left(\vtheta^0\right)}}{m\eps^2\lambda_{a}} \leq\frac{1}{2}.
\end{equation*}
We  inherit the proof in \Cref{append...thm..ThetaLazyRegime} and obtain that
\begin{align*}
\Abs{G_{ij}^{[a]}(\vtheta(t)) - G_{ij}^{[a]}(\vtheta(0))} &\leq \frac{\nu\eps^3}{m}\sum_{k=1}^m \Abs{g^{[a]}_{ij}(\vw_k(t)) - g^{[a]}_{ij}(\vw_k(0))}\\
&\leq  8d^2{\nu^2\eps^2}\xi^2\max\left\{\frac{\nu}{\eps},\frac{\eps}{\nu}\right\}p_a\\
&\leq 8d^2\nu\eps^3\xi^2p_a,
\end{align*}
by using the same technique in \Cref{append...prop...LeastEigenvalueatInitial},
\begin{equation*}
\begin{aligned}
&\Norm{\mG^{[a]}(\vtheta(t)) - \mG^{[a]}(\vtheta(0))}_\mathrm{F}\\
&~~\leq 8d^2n\nu\eps^3\left(2\log\frac{8m(d+1)}{\delta}\right)\frac{2\sqrt{2}dn\sqrt{\RS\left(\vtheta^0\right)}}{m\eps^2\lambda_a} \\
& \leq\nu\eps^3\frac{32\sqrt{2}d^3n^2\left(\log\frac{8m(d+1)}{\delta}\right)\sqrt{\RS\left(\vtheta^0\right)}}{m\eps^2\lambda_a}.
\end{aligned}
\end{equation*}
As $\gamma\geq 1$, and $\gamma'>\gamma-1$, then we may choose $m$ large enough, such that
\begin{align*}
m\eps^2&\geq \frac{128\sqrt{2}d^3n^2\left(\log\frac{8m(d+1)}{\delta}\right)\sqrt{\RS\left(\vtheta^0\right)}}{\lambda_a^2},
\end{align*}
then for any $t\in[0, t_a^\ast)$,
\begin{equation}\label{eqproof....thm...WLazy...PartOne}
\Norm{\mG^{[a]}(\vtheta(t)) - \mG^{[a]}(\vtheta(0))}_\mathrm{F}\leq\frac{1}{4}\nu\eps^3\lambda_a.
\end{equation}
Hence, for $t\in[0,t_a^\ast)$, the following holds
\begin{equation*}
\RS(\vtheta(t)) \leq \exp\left(-\frac{m}{n}\nu\eps^3\lambda_at\right)\RS\left(\vtheta^0\right).
\end{equation*}
Suppose we have $t_a^\ast<+\infty$, then one can take the limit $t\to t_a^\ast$ in~\eqref{eqproof....thm...WLazy...PartOne}, which  leads to a contradiction with the definition of $t_a^\ast$. Therefore $t_a^\ast=+\infty$.   

Directly from \Cref{append...prop..UpperBoundandLowerBoundInitial}, we have with probability at least $1-\frac{\delta}{2}$ over the choice of $\vtheta^0$,
 \begin{equation*}
\Norm{\vtheta_{\vw}^0}_2\geq \sqrt{\frac{md}{2}},
\end{equation*}
thus we have
\begin{align*}
\sup\limits_{t\in[0,+\infty)}\frac{\norm{\vtheta_{\vw}(t)-\vtheta_{\vw}(0)}_2}{\norm{\vtheta_{\vw}(0)}_2}
& \leq \sqrt{\frac{2}{md}}\sup\limits_{t\in[0,+\infty)}\Norm{\vtheta_{\vw}(t)-\vtheta_{\vw}(0)}_2,
\end{align*}
via \Cref{append...prop...BoundofTraining...WLazy},  with probability at least $1-\frac{\delta}{2}$ over the choice of $\vtheta^0$,
\begin{equation*}
\begin{aligned}       
\max\limits_{k\in[m]}\Norm{\vw_k(t) - \vw_k(0)}_{\infty}
& \leq 2\xi  p_a,
\end{aligned}\end{equation*}
then 
\begin{align*}
\Norm{\vtheta_{\vw}(t)-\vtheta_{\vw}(0)}_2&\leq \left[\sum_{k=1}^m\left(\Norm{\vw_k(t) - \vw_k(0)}_{\infty}^2\right)\right]^{\frac{1}{2}}\\
&\leq 2\sqrt{md}\xi p_a\\
&\leq  2\sqrt{md}\sqrt{2\log\frac{8m(d+1)}{\delta}}\frac{2\sqrt{2}dn\sqrt{\RS\left(\vtheta^0\right)}}{m\eps^2\lambda_a}\\
&\leq 8\sqrt{\log\frac{8m(d+1)}{\delta}}\frac{d^{\frac{3}{2}}n\sqrt{\RS\left(\vtheta^0\right)}}{\sqrt{m}\eps^2\lambda_a},
\end{align*}
hence
\begin{align*}
\sup\limits_{t\in[0,+\infty)}\frac{\norm{\vtheta_{\vw}(t)-\vtheta_{\vw}(0)}_2}{\norm{\vtheta_{\vw}(0)}_2}
& \leq \sqrt{\frac{2}{md}}\sup\limits_{t\in[0,+\infty)}\Norm{\vtheta_{\vw}(t)-\vtheta_{\vw}(0)}_2\\
&\leq 8\sqrt{2}\sqrt{\log\frac{8m(d+1)}{\delta}}\frac{dn\sqrt{\RS\left(\vtheta^0\right)}}{{m}\eps^2\lambda_a},
\end{align*}
as $\gamma\geq 1$, and $\gamma'>\gamma-1$, then we obtain that
\begin{equation}
 \sup\limits_{t\in[0,+\infty)}\frac{\norm{\vtheta_{\vw}(t)-\vtheta_{\vw}(0)}_2}{\norm{\vtheta_{\vw}(0)}_2}\lesssim   \frac{\sqrt{\log\frac{8m(d+1)}{\delta}}}{m^{1-\gamma+\gamma'}},
\end{equation}
hence
\begin{equation}
\lim_{m\to\infty}\sup\limits_{t\in[0,+\infty)}\frac{\norm{\vtheta_{\vw}(t)-\vtheta_{\vw}(0)}_2}{\norm{\vtheta_{\vw}(0)}_2}=0.
\end{equation}
\end{proof}
\section{Condensed Regime}
\subsection{Effective Linear Dynamics}
As  the normalized flow reads
\begin{equation}\label{append...eq...text...Condensed...ChZero...OriginDynamics}
\begin{aligned}
\frac{\D {a}_k}{\D {t}}& = \frac{\eps}{\nu}\left(-\frac{1}{n}\sum_{i=1}^n e_i \frac{\sigma(\eps{\vw}_k^\T\vx_i)}{\eps}\right),\\
\frac{\D  {\vw}_k}{\D {t}}& = \frac{\nu}{\eps} \left(-\frac{1}{n}\sum_{i=1}^ne_i   {a}_k\sigma^{(1)}(\eps {\vw}_k^\T\vx_i)\vx_i\right),
\end{aligned}
\end{equation}
since $e_i\approx -y_i$, and   by means of perturbation
expansion with respect to $\eps$ and keep the order $1$ term, we obtain that
\begin{equation*}
\begin{aligned}
\frac{\D {a}_k}{\D {t}}& \approx \frac{\eps}{\nu}\frac{1}{n}\sum_{i=1}^n y_i  \sigma^{(1)}(0){\vw}_k^\T\vx_i= \frac{\eps}{\nu}\frac{1}{n}\sum_{i=1}^n y_i{\vw}_k^\T\vx_i, \\
\frac{\D  {\vw}_k}{\D {t}}& \approx \frac{\nu}{\eps}\frac{1}{n}\sum_{i=1}^n y_i     {a}_k\sigma^{(1)}(0)\vx_i=\frac{\nu}{\eps}\frac{1}{n}\sum_{i=1}^n y_i     {a}_k\vx_i,
\end{aligned}
\end{equation*}
so the normalized flow approximately reads
\begin{equation}\label{append...eq...text...Condensed...ChZero...ModifiedDynamics}
\begin{aligned}
\frac{\D {a}_k}{\D {t}}& = \frac{\eps}{\nu}\frac{1}{n}\sum_{i=1}^n y_i{\vw}_k^\T\vx_i=\frac{\eps}{\nu}{\vw}_k^\T\vz, \\
\frac{\D  {\vw}_k}{\D {t}}&  =\frac{\nu}{\eps}\frac{1}{n}\sum_{i=1}^n y_i     {a}_k\vx_i=\frac{\nu}{\eps}{a}_k\vz.
\end{aligned}
\end{equation}
We observe that \eqref{append...eq...text...Condensed...ChZero...ModifiedDynamics} reveals that the training  dynamics of two-layer NNs at initial stage has a close relationship to power iteration  of a matrix that only depends on the input sample $\fS$. We denote by
\begin{equation}
    \mA:=\left[\begin{array}{cc}
0 & \vz^\T \\
\vz & \mzero_{d\times d}
\end{array}\right], 
\end{equation}
where
\begin{equation*} 
\vz=\frac{1}{n}\sum_{i=1}^n y_i {\vx}_i,
\end{equation*}
whose definition can be traced back to \Cref{assumption...GenericData}, and \eqref{append...eq...text...Condensed...ChZero...ModifiedDynamics} can be written into
\begin{equation}\label{append...eq...text...Condensed...ChZero...ModifiedinMatrixForm}
\begin{aligned}
\frac{\D \vq_k}{\D {t}}&=\mA\vq_k, 
\end{aligned}
\end{equation}
where
\begin{equation*}
\vq_k:=\left[\frac{\nu}{\eps}a_k, \vw_k\right]^\T.
\end{equation*}
Moreover, simple linear algebra shows that  $\mA$ has two nonzero eigenvalues $\lambda_1=\Norm{\vz}_2$ and $\lambda_2=-\Norm{\vz}_2$. Moreover, the unit eigenvector for $\lambda_1=\Norm{\vz}_2$ is
\begin{align*}
\vu_1:=\left[\frac{1}{\sqrt{2}}, \frac{1}{\sqrt{2}}\hat{\vz}^\T\right]^\T, 
\end{align*}
where 
\begin{equation*}
\hat{\vz}=\frac{ \sum_{i=1}^n y_i {\vx}_i}{\Norm{ \sum_{i=1}^n y_i {\vx}_i}_2}.
\end{equation*}
whose definition can be traced back to \Cref{assumption...GenericData}. We obtain further that   
the unit eigenvector for $\lambda_2=-\Norm{\vz}_2$ is
\begin{align*}
 \vu_2:=\left[-\frac{1}{\sqrt{2}}, \frac{1}{\sqrt{2}}\hat{\vz}^\T\right]^\T, 
\end{align*}
and $\left<\vu_1, \vu_2\right>=0$.  

We also observe that the rest of the eigenvalues for $\mA$ is all zero, i.e., 
\[
\lambda_3=\lambda_4=\dots=\lambda_{d+1}=0,
\]
and their eigenvectors $\{\vu_k\}_{k=3}^{d+1}$ read,
\[
\vu_k:=\left[0, \vb_k^\T\right]^\T,\quad \vb_k\in \vz^{\perp},
\]
and 
\[
\mathrm{span}\left\{\vb_k\right\}_{k=3}^{d+1}= \vz^{\perp},
\]
whose first component is zero, and the rest of 
the  components spans the   orthogonal complement of $\vz$, i.e.,
\[\vz^{\perp}=\{\vx\in \sR^d \mid \left< \vz, \vx\right>=0,~~ \forall \vx \in \sR^d\},\]
since 
\[\mathrm{dim}\left(\mathrm{span}\left\{\vb_k\right\}_{k=3}^{d+1}\right)=\mathrm{dim}\left(\vz^{\perp}\right)=d-1.\]
We  hereafter denote that: For $t\geq 0$,
\[r(t):=\exp\left(\frac{1}{2}\Norm{\vz}_2t\right),\]
and since 
\[\frac{1}{c}\leq\Norm{\vz}_2\leq c,\]
for some universal constant $c>0$, then we obtain that for some universal constants $c_1>0$ and $c_2>0$,
\[\exp(c_1t)\leq r(t)\leq \exp(c_2t).\]
\begin{proposition}\label{append...prop...CondensedRegime...LinearODESolution}
The solution to the linear differential equation
\begin{equation}\label{append...eq...proposition...Condensed...LinearDynamics}
\frac{\D \vq}{\D {t}}=\mA\vq,\quad\vq(0)=\vq^0,
\end{equation}
where
\[\vq(t):=\left[a(t), \left(\vw(t)\right)^\T\right]^\T,\quad\vq^0:=\left[a^0, \left(\vw^0\right)^\T\right]^\T,\]
reads
\begin{equation}
\begin{aligned}
a(t)&=\left(\frac{1}{ {2}}r^2(t)+\frac{1}{ {2}}r^{-2}(t)\right) a^0+\left(\frac{1}{ {2}}r^2(t)-\frac{1}{ {2}}r^{-2}(t)\right)\left<\vw^0, \hat{\vz}\right>,\\
\vw(t)&=\left(\frac{1}{ {2}}r^2(t)-\frac{1}{ {2}}r^{-2}(t)\right)a^0\hat{\vz}+\left(\frac{1}{ {2}}r^2(t)+\frac{1}{ {2}}r^{-2}(t)\right)\left<\vw^0, \hat{\vz}\right>\hat{\vz}\\
&~~-\left<\vw^0, \hat{\vz}\right>\hat{\vz}+\vw^0,
\end{aligned}    
\end{equation}
where $r(t)=\exp\left(\frac{1}{2}\Norm{\vz}_2t\right).$
\end{proposition}
\begin{proof}
We only need to solve out  the matrix exponential for $\mA$, as $\mA$ is symmetric, then it   can be diagonalized by an orthogonal matrix $\mP$, where
\[
\mP:=\left[\vu_1, \vu_2, \dots, \vu_{d+1}\right],
\]
and 
\[
\mA=\mP\mJ\mP^\T,
\]
where 
\[
\mJ:=\begin{bmatrix}\norm{\vz}_2 & & \\
&-\norm{\vz}_2 &  \\ & & \ddots & \\ & & & 0\end{bmatrix},
\]
since 
\[
\exp\left(t\mA\right)=\mP\exp\left(t\mJ\right)\mP^\T,
\]
then
\begin{align*}
\begin{pmatrix}
a(t)\\
\vw(t) \end{pmatrix}
&=
\exp\left(t\mA\right)
\begin{pmatrix}a^0\\
\vw^0 \end{pmatrix},
\end{align*}
thus, we obtain that
\begin{align*}
a(t)&=\left(\frac{1}{ {2}}\exp(\norm{\vz}_2t)+\frac{1}{ {2}}\exp(-\norm{\vz}_2t)\right) a^0\\
&~~+\left(\frac{1}{ {2}}\exp(\norm{\vz}_2t)-\frac{1}{ {2}}\exp(-\norm{\vz}_2t)\right)\left<\vw^0, \hat{\vz}\right>,\\
\vw(t)&=\left(\frac{1}{ {2}}\exp(\norm{\vz}_2t)-\frac{1}{ {2}}\exp(-\norm{\vz}_2t)\right)a^0\hat{\vz}\\
&~~+\left(\frac{1}{ {2}}\exp(\norm{\vz}_2t)+\frac{1}{ {2}}\exp(-\norm{\vz}_2t)\right)\left<\vw^0, \hat{\vz}\right>\hat{\vz}\\
&~~-\left<\vw^0, \hat{\vz}\right>\hat{\vz}+\vw^0.
\end{align*}
\end{proof}
\begin{remark}
It is noteworthy that $\vw$ has  two components, one is 
the projection of  $\vw^0$ into the direction of  $\vz$:
\[
\left(\frac{1}{ {2}}r^2(t)-\frac{1}{ {2}}r^{-2}(t)\right)a^0\hat{\vz} +\left(\frac{1}{ {2}}r^2(t)+\frac{1}{ {2}}r^{-2}(t)\right)\left<\vw^0, \hat{\vz}\right>\hat{\vz},
\]
which evolves with respect to   $t$.  The other is  the projection of  $\vw^0$ onto $\vz^{\perp}$:
\[
 \vw^0-\left<\vw^0, \hat{\vz}\right>\hat{\vz},
\]
which remains frozen as $t$ evolves.
\end{remark}
\subsection{Difference between Real  and Linear Dynamics}\label{append...subsec...Real+LinearDynamics}
We observe that the real dynamics can be written into 
\begin{equation}\label{append...eq...text...Condense...Difference...W-lag} 
\begin{aligned}
\frac{\nu}{\eps}\frac{\D {a}_k}{\D {t}}& =  \left(-\frac{1}{n}\sum_{i=1}^n e_i \frac{\sigma(\eps{\vw}_k^\T\vx_i)}{\eps}-\frac{1}{n}\sum_{i=1}^n y_i   {\vw}_k^\T\vx_i+\vw_k^\T\vz\right), \\
\frac{\D  {\vw}_k}{\D {t}}& = \frac{\nu}{\eps}\left(-\frac{1}{n}\sum_{i=1}^ne_i    {a}_k\sigma^{(1)}(\eps {\vw}_k^\T\vx_i)\vx_i-\frac{1}{n}\sum_{i=1}^n y_i     {a}_k\vx_i+a_k\vz\right),
\end{aligned}
\end{equation} 
hence the difference between the real and linear dynamics is characterized by $\{f_k, \vg_k\}_{k=1}^m$, where
\begin{align*}
f_k&:=\frac{1}{n}\sum_{i=1}^n\left( e_i \frac{\sigma(\eps{\vw}_k^\T\vx_i)}{\eps}+ y_i   {\vw}_k^\T\vx_i\right),\\
\vg_k&:=\frac{1}{n}\sum_{i=1}^n\left(e_i    {a}_k\sigma^{(1)}(\eps {\vw}_k^\T\vx_i)\vx_i+y_ia_k\vx_i\right).
\end{align*}
In other words, the real dynamics can be written into
\begin{equation}\label{append...eq...text...Condense...Difference...RealDynamicsRewritten}
\begin{aligned}
 \frac{\D \begin{pmatrix}\frac{\nu}{\eps}a_k\\
\vw_k \end{pmatrix}}{\D t}
=\mA\begin{pmatrix}\frac{\nu}{\eps}a_k\\
\vw_k \end{pmatrix}
+
\begin{pmatrix}f_k\\
\frac{\nu}{\eps}\vg_k \end{pmatrix},\quad \begin{pmatrix}\frac{\nu}{\eps}a_k(0)\\
\vw_k(0) \end{pmatrix}=\begin{pmatrix}\frac{\nu}{\eps}a_k^0\\
\vw_k^0 \end{pmatrix},
\end{aligned}    
\end{equation}
and its solution reads
\begin{equation}
\label{append...eq...text...Condense...Difference...RealDynamicsSolution}
\begin{aligned}
\begin{pmatrix}\frac{\nu}{\eps}a_k\\
\vw_k \end{pmatrix}
=\exp\left(t\mA\right)\begin{pmatrix}\frac{\nu}{\eps}a_k^0\\
\vw_k^0 \end{pmatrix}
+
\int_{0}^t\exp\left((t-s)\mA\right)\begin{pmatrix}f_k(s)\\
\frac{\nu}{\eps}\vg_k(s) \end{pmatrix}\D s.
\end{aligned}        
\end{equation}
\subsection{$\vw$-lag regime}
\begin{definition}[Neuron $2$-energy, $\vw$-lag regime]
In real dynamics, we define the  $2$-energy at time $t$ for each single neuron, i.e., for each $k\in[m]$, 
\begin{equation}
q_k(t):=\left(\frac{\nu^2}{\eps^2}\Abs{a_k(t)}^2 + \Norm{\vw_k(t)}_2^2\right)^{\frac{1}{2}}.
\end{equation}
\end{definition}
\noindent  We denote 
\begin{equation}
q_{\max}(t):=\max_{k\in[m]} q_k(t).
\end{equation}
For simplicity, we hereafter drop the $(t)$s for all $q_k(t)$ and $q_{\max}(t)$.  
Then the estimates on  $\{f_k, \vg_k\}_{k=1}^m$ read
\begin{proposition}\label{append...prop...CellProblemEstimates...Wlag}
For any time $t>0$,
\begin{equation}
\begin{aligned}
\Abs{f_k}&\leq m\eps^2 q_{\max}^2\Norm{{\vw}_k}_2+\eps \Norm{{\vw}_k}_2^2,\\
\Norm{\vg_k}_2&\leq m\eps^2 q_{\max}^2\Abs{a_k}+\eps \Norm{{\vw}_k}_2\Abs{a_k}.
\end{aligned}    
\end{equation}
\end{proposition}
\begin{proof}
We obtain that 
\begin{align*}
\Abs{f_k}&=\Abs{\frac{1}{n}\sum_{i=1}^n \left((e_i+y_i)\frac{\sigma(\eps{\vw}_k^\T\vx_i)}{\eps}+y_i   {\vw}_k^\T\vx_i-y_i\frac{\sigma(\eps{\vw}_k^\T\vx_i)}{\eps}\right)}\\
&=\Abs{\frac{1}{n}\sum_{i=1}^n \left(\left(\sum_{k=1}^m \nu a_k \sigma(\eps{\vw}_k^\T\vx_i)\right)\frac{\sigma(\eps{\vw}_k^\T\vx_i)}{\eps}+y_i   {\vw}_k^\T\vx_i-y_i\frac{\sigma(\eps{\vw}_k^\T\vx_i)}{\eps}\right)}\\
&\leq \frac{1}{n}\sum_{i=1}^n \left(\left(\sum_{k=1}^m \nu\eps \Abs{a_k} \Norm{{\vw}_k}_2 \right)\Norm{{\vw}_k}_2+  \eps ({\vw}_k^\T\vx_i)^2 \right)\\
&\leq\frac{1}{2n}\sum_{i=1}^n \left(\nu\eps\left(\sum_{k=1}^m \left(\frac{\nu}{\eps} \Abs{a_k}^2+\frac{\eps}{\nu} \Norm{{\vw}_k}_2^2 \right)\right)\Norm{{\vw}_k}_2+  \eps ({\vw}_k^\T\vx_i)^2 \right)\\
&\leq \frac{1}{2n}\sum_{i=1}^n \left(\left(\sum_{k=1}^m \eps^2 q_k^2 \right)\Norm{{\vw}_k}_2+  \eps ({\vw}_k^\T\vx_i)^2 \right)\\
&\leq \frac{1}{2n}\sum_{i=1}^n \left(\left(\sum_{k=1}^m \eps^2 q_k^2 \right)\Norm{{\vw}_k}_2+  \eps \Norm{{\vw}_k}_2^2 \right)\\
&\leq m\eps^2 q_{\max}^2\Norm{{\vw}_k}_2+\eps \Norm{{\vw}_k}_2^2,
\end{align*}
and 
\begin{align*}
\Norm{\vg_k}_2&=\Norm{\frac{1}{n}\sum_{i=1}^n \left((e_i+y_i){a}_k\sigma^{(1)}(\eps {\vw}_k^\T\vx_i)\vx_i+y_i   {a}_k\vx_i-y_i{a}_k\sigma^{(1)}(\eps {\vw}_k^\T\vx_i)\vx_i\right)}_2\\
&=\Norm{\frac{1}{n}\sum_{i=1}^n \left(\left(\sum_{k=1}^m \nu a_k \sigma(\eps{\vw}_k^\T\vx_i)\right){a}_k\sigma^{(1)}(\eps {\vw}_k^\T\vx_i)\vx_i+y_i   {a}_k\vx_i-y_i{a}_k\sigma^{(1)}(\eps {\vw}_k^\T\vx_i)\vx_i\right)}_2\\
&\leq \frac{1}{n}\sum_{i=1}^n \left(\left(\sum_{k=1}^m \nu\eps \Abs{a_k} \Norm{{\vw}_k}_2 \right)\Abs{a_k}+  \eps \Norm{{\vw}_k}_2\Abs{a_k} \right)\\
&\leq \frac{1}{2n}\sum_{i=1}^n \left( \left(\sum_{k=1}^m \eps^2 q_k^2 \right)\Abs{a_k}+  \eps\Norm{{\vw}_k}_2\Abs{a_k}\right)\\
 &\leq m\eps^2 q_{\max}^2\Abs{a_k}+\eps \Norm{{\vw}_k}_2\Abs{a_k}.
\end{align*}
\end{proof}
We denote a useful quantity
\begin{equation} 
\phi(t):= \sup_{0\leq s\leq t} {q}_{\max}(s).
\end{equation}
Then directly from \Cref{append...lemma..Initialization}, we have with probability at least $1-\delta$ over the choice of $\vtheta^0$,
\begin{equation}
\max\limits_{k\in[m]}\left\{\abs{a_k^0},\;\norm{\vw_k^0}_{\infty}\right \}\leq \sqrt{2\log\frac{2m(d+1)}{\delta}},
\end{equation}
hence
\begin{equation} \label{apend...eq...text...PHI}
 \phi(0)\leq\sqrt{2(d+1)\log\frac{2m(d+1)}{\delta}}.
\end{equation}
We define
\begin{equation}
T_{\mathrm{eff}}: = \inf\left\{t>0 \mid m\eps^2\phi^3(t)>m^{-\tau},\quad \tau=\frac{\gamma-\gamma'-1}{4}\right\},
\end{equation}
then for $m$ large enough, as $\gamma-\gamma'>1$, based on \eqref{apend...eq...text...PHI},
\begin{equation}
m\eps^2\phi^3(0)\leq  m\eps^2\left(2(d+1)\log\frac{2m(d+1)}{\delta}\right)^{\frac{3}{2}} \leq m^{-\frac{\gamma-\gamma'-1}{2}},
\end{equation}
hence $T_{\mathrm{eff}}\geq 0$.

We observe further that  by taking the $2$-norm on both sides of \eqref{append...eq...text...Condense...Difference...RealDynamicsSolution}, the following holds 
\begin{align*}
q_k(t)&\leq\exp\left(t\Norm{\mA}_{2\to 2}\right)q_k(0)+\int_{0}^t\exp\left((t-s)\Norm{\mA}_{2\to 2}\right)\left(m\eps^2q_{\max}^2(s)+\eps q_{\max}(s)\right)q_k(s)\D s,     
\end{align*}
and by taking supreme over the  index $k$ and  time $0\leq t\leq T_{\mathrm{eff}}$   on both sides,  and  for  large enough $m$, the following holds  
\begin{equation}
\begin{aligned}
\phi(t)&\leq \phi(0)\exp(t)+ 2m^{-\min\left\{\frac{1}{2},\tau\right\}}\int_0^t\exp(t-s)\D s \\ 
&\leq \phi(0)\exp(t)+2m^{-\min\left\{\frac{1}{2},\tau\right\}}\left(\exp(t)-1\right)\\
&\leq \phi(0)\exp(t)+2m^{-\min\left\{\frac{1}{2},\tau\right\}}\exp(t),
\end{aligned}
\end{equation}
then based on \eqref{apend...eq...text...PHI}, we have with probability $1-\delta$ over the choice of $\vtheta^0$, for sufficiently large $m$,
\begin{equation}
\begin{aligned}
\phi(t)  
&\leq \phi(0)\exp(t)+2m^{-\min\left\{\frac{1}{2},\tau\right\}}\exp(t)\\
&\leq 2\phi(0)\exp(t)\leq 2\sqrt{2(d+1)\log\frac{2m(d+1)}{\delta}}\exp(t),  
\end{aligned}  
\end{equation}
we set $t_0$ as the time satisfying
\begin{equation}
2\sqrt{2(d+1)\log\frac{2m(d+1)}{\delta}} \exp(t_0)  =\frac{1}{2} m^{\frac{\gamma-\gamma'-1}{4}},  
\end{equation}
then we obtain that, for any $\eta_0>\frac{\gamma-\gamma'-1}{100}>0$,
\begin{equation}
T_{\mathrm{eff}}\geq t_0> \log\left(\frac{1}{4}\right)+\left({\frac{\gamma-\gamma'-1}{4}}-\eta_0\right)\log(m).
\end{equation}
\begin{theorem}[Condensed regime, $\vw$-lag regime]\label{append...thm..CondensedRegime...W-Einent...PartOne}
Given any $\delta\in(0,1)$, under \Cref{Assumption....ActivationFunctions}, \Cref{assumption...GenericData} and \Cref{assump...LimitExistence}, if $\gamma> 1$ and $0\leq\gamma'<\gamma-1$, 
then	with probability at least $1-\delta$ over the choice of $\vtheta^0$,
we have  
\begin{equation}\label{append...eq...thm...Condense...WLag...PartOne}
\lim_{m\to+\infty} \sup\limits_{t\in[0,T_{\mathrm{eff}}]}\mathrm{RD}(\vtheta_{\vw}(t))=+\infty,
\end{equation}
and
\begin{equation}\label{append...eq...thm...Condense...WLag...PartTwo}
\lim_{m\to+\infty} \sup\limits_{t\in[0,T_{\mathrm{eff}}]}\frac{\Norm{\vtheta_{\vw, \vz}(t) }_2}{\Norm{\vtheta_{\vw}(t)}_2} =1.
\end{equation}
\end{theorem}
\begin{proof}
Since we have 
\begin{equation*}
\begin{aligned}
\frac{\D \begin{pmatrix}\frac{\nu}{\eps}a_k\\
\vw_k \end{pmatrix}}{\D t}
=\mA\begin{pmatrix}\frac{\nu}{\eps}a_k\\
\vw_k \end{pmatrix}
+
\begin{pmatrix}f_k\\
\frac{\nu}{\eps}\vg_k \end{pmatrix},\quad \begin{pmatrix}\frac{\nu}{\eps}a_k(0)\\
\vw_k(0) \end{pmatrix}=\begin{pmatrix}\frac{\nu}{\eps}a_k^0\\
\vw_k^0 \end{pmatrix},
\end{aligned}    
\end{equation*}
and its solution reads
\begin{equation*}
\begin{aligned}
\begin{pmatrix}\frac{\nu}{\eps}a_k\\
			\vw_k \end{pmatrix}
&=\exp\left(t\mA\right)\begin{pmatrix}\frac{\nu}{\eps}a_k^0\\
			\vw_k^0 \end{pmatrix}
+\int_{0}^t\exp\left((t-s)\mA\right)
\begin{pmatrix}f_k\\
			\frac{\nu}{\eps}\vg_k \end{pmatrix}\D s.
\end{aligned}
\end{equation*}
As we notice that for any $k\in[m]$, $\begin{pmatrix}\frac{\nu}{\eps}a_k\\
			\vw_k \end{pmatrix}$ 
can be written into two parts,  the first one is the linear part, the second one is the residual part. For simplicity of proof, we need to introduce some further notations.

As we already identify the parameters $\vtheta_a=\mathrm{vec}(\{a_k\}_{k=1}^{m})$  and $\vtheta_{\vw}=\mathrm{vec}(\{\vw_k\}_{k=1}^{m})$, with some slight misuse of notations, we denote $\vtheta_a:=\mathrm{vec}(\left\{\frac{\nu}{\eps}a_k\right\}_{k=1}^{m})$.
From the observations above, 
\[
\vtheta:=\mathrm{vec}\left(\vtheta_a, \vtheta_{\vw}\right)=\mathrm{vec}\left(\left\{\frac{\nu}{\eps}a_k\right\}_{k=1}^{m}, \{\vw_k\}_{k=1}^{m}\right),
\]
and
$\vtheta_a$ and $\vtheta_{\vw}$ can be written into
\begin{align*}
 \vtheta_a(t)&=    \Bar{\vtheta}_a(t)+\widetilde{\vtheta}_a(t),\\
  \vtheta_{\vw}(t)&=    \Bar{\vtheta}_{\vw}(t)+\widetilde{\vtheta}_{\vw}(t),
\end{align*}
where the $k$-th component of $\Bar{\vtheta}_a$ and  $\Bar{\vtheta}_{\vw}$ respectively reads
\begin{align*}
\left(\Bar{\vtheta}_a\right)_k&:=   \frac{\nu}{\eps} \left(\frac{1}{ {2}}r^2(t)+\frac{1}{ {2}}r^{-2}(t)\right) a_k^0+\left(\frac{1}{ {2}}r^2(t)-\frac{1}{ {2}}r^{-2}(t)\right)\left<\vw_k^0, \hat{\vz}\right>,\\
\left(\Bar{\vtheta}_{\vw}\right)_{k}&:=\frac{\nu}{\eps}\left(\frac{1}{ {2}}r^2(t)-\frac{1}{ {2}}r^{-2}(t)\right)a_k^0\hat{\vz} +\left(\frac{1}{ {2}}r^2(t)+\frac{1}{ {2}}r^{-2}(t)\right)\left<\vw_k^0, \hat{\vz}\right>\hat{\vz}\\
&~~-\left<\vw_k^0, \hat{\vz}\right>\hat{\vz}+\vw_k^0,
\end{align*}
 and the $k$-th component of $\widetilde{\vtheta}_a$ and  $\widetilde{\vtheta}_{\vw}$ altogether reads
\begin{align*}
\begin{pmatrix}\left(\widetilde{\vtheta}_a(t)\right)_k\\
			\left(\widetilde{\vtheta}_{\vw}(t)\right)_{k}\end{pmatrix}
&:=  
\int_{0}^{t}\exp\left((t-s)\mA\right)\begin{pmatrix}f_k\\
\frac{\nu}{\eps}\vg_k \end{pmatrix}\D s.
\end{align*}
Moreover, we observe that  ${\vtheta}_{\vw}$ can be decomposed into two parts,  one is 
the projection of  $\vw^0$ into the direction of  $\vz$, i.e., ${\vtheta}_{\vw, \vz}$ , and the other is  the projection of  $\vw^0$ onto $\vz^{\perp},$ i.e., ${\vtheta}_{\vw, \vz^{\perp}}$. As  $\Bar{\vtheta}_{\vw}$  and $\widetilde{\vtheta}_{\vw}$ inherits the same structure as ${\vtheta}_{\vw}$, we may apply the same decomposition to $\Bar{\vtheta}_{\vw}$  and $\widetilde{\vtheta}_{\vw}$. Hence, we obtain that 
\begin{align*}
{\vtheta}_{\vw, \vz}(t)&= \Bar{\vtheta}_{\vw, \vz}(t)+\widetilde{\vtheta}_{\vw, \vz}(t),\\   
{\vtheta}_{\vw, \vz^{\perp}}(t)&= \Bar{\vtheta}_{\vw, \vz^{\perp}}(t)+\widetilde{\vtheta}_{\vw, \vz^{\perp}}(t),
\end{align*}
and these relations concerning $2$-norm hold simultaneously for any $t\geq 0$,
\begin{align*}
\Norm{{\vtheta}_{\vw}(t)}_2^2&=\Norm{{\vtheta}_{\vw, \vz}(t)}_2^2+\Norm{{\vtheta}_{\vw, \vz^{\perp}}(t)}_2^2,   \\
\Norm{\Bar{\vtheta}_{\vw}(t)}_2^2&=\Norm{\Bar{\vtheta}_{\vw, \vz}(t)}_2^2+\Norm{\Bar{\vtheta}_{\vw, \vz^{\perp}}(t)}_2^2, \\ 
\Norm{\widetilde{\vtheta}_{\vw}(t)}_2^2&=\Norm{\widetilde{\vtheta}_{\vw, \vz}(t)}_2^2+\Norm{\widetilde{\vtheta}_{\vw, \vz^{\perp}}(t)}_2^2,  
\end{align*}
and the $k$-th component of $\Bar{\vtheta}_{\vw, \vz}$ and  $\Bar{\vtheta}_{\vw, \vz^{\perp}}$ altogether reads
\begin{align*}
\left(\Bar{\vtheta}_{\vw, \vz}\right)_k &:=   \frac{\nu}{\eps}\left(\frac{1}{ {2}}r^2(t)-\frac{1}{ {2}}r^{-2}(t)\right)a_k^0\hat{\vz} +\left(\frac{1}{ {2}}r^2(t)+\frac{1}{ {2}}r^{-2}(t)\right)\left<\vw_k^0, \hat{\vz}\right>\hat{\vz},\\
\left(\Bar{\vtheta}_{\vw, \vz^{\perp}}\right)_k &:=\vw_k^0 -\left<\vw_k^0, \hat{\vz}\right>\hat{\vz},
\end{align*}
and finally, based on the  relations concerning $2$-norm above, for any $t\geq 0$,
\begin{align*}
\Norm{\widetilde{\vtheta}_{\vw, \vz}(t)}_2 &\leq\Norm{\widetilde{\vtheta}_{\vw}(t)}_2,\\
\Norm{\widetilde{\vtheta}_{\vw, \vz^{\perp}}(t)}_2 &\leq\Norm{\widetilde{\vtheta}_{\vw}(t)}_2.
\end{align*}

We are hereby to prove \eqref{append...eq...thm...Condense...WLag...PartOne}.  Firstly, we observe that 
\[
{\vtheta}_{\vw}(0)=\Bar{\vtheta}_{\vw}(0),
\]
hence
\begin{align*}
&\Norm{\vtheta_{\vw}(t)-\vtheta_{\vw}(0)}_2^2\\
=&\Norm{\vtheta_{\vw}(t)-\Bar{\vtheta}_{\vw}(0)}^2_2\\     
=&\Norm{\vtheta_{\vw, \vz}(t)-\Bar{\vtheta}_{\vw, \vz}(0)+\vtheta_{\vw, \vz^{\perp}}(t)-\Bar{\vtheta}_{\vw, \vz^{\perp}}(0)}_2\\
=&\Norm{\Bar{\vtheta}_{\vw, \vz}(t)-\Bar{\vtheta}_{\vw, \vz}(0)+\widetilde{\vtheta}_{\vw, \vz}(t)}^2_2+\Norm{\Bar{\vtheta}_{\vw, \vz^{\perp}}(t)-\Bar{\vtheta}_{\vw, \vz^{\perp}}(0)+\widetilde{\vtheta}_{\vw, \vz^{\perp}}(t)}^2_2,
\end{align*} 
 by choosing $\eta_0=\frac{\gamma-\gamma'-1}{8}$, then  for time $0\leq t\leq \bar{t}_0:=\left({\frac{\gamma-\gamma'-1}{8}}\right)\log(m)-\log(2)$,
\begin{align*}
 &\Norm{\int_{0}^{t}\exp\left((t-s)\mA\right)
 \begin{pmatrix}f_k\\
 \frac{\nu}{\eps}\vg_k \end{pmatrix}
 \D s}_2\\
&\leq \left(m\eps^2\phi^3(t)+  \eps\phi^2(t)\right) \int_{0}^{t}\exp((t-s)\Norm{\vz}_2)  \D s\\
&\leq 2m^{-\min\left\{\tau, \frac{1}{2}\right\}}\int_{0}^{t}\exp((t-s)) \D s\\
&\leq 2m^{-\frac{\gamma-\gamma'-1}{4}}  \exp(t)\leq 2m^{-\frac{\gamma-\gamma'-1}{4}}  \exp(\bar{t}_0)= m^{-\frac{\gamma-\gamma'-1}{8}}.
\end{align*}
We  conclude that for $t\leq \bar{t}_0$, the following holds
\begin{align*}
\Norm{\widetilde{\vtheta}_{\vw}(t)}_2&\leq \sqrt{m}\Norm{\int_{0}^{t}\exp\left((t-s)\mA\right)
 \begin{pmatrix}f_k\\
 \frac{\nu}{\eps}\vg_k \end{pmatrix}
 \D s}_2\leq\sqrt{m} m^{-\frac{\gamma-\gamma'-1}{8}},
\end{align*}
since the $k$-th component of   $\Bar{\vtheta}_{\vw, \vz^{\perp}}$  reads
\begin{align*}
\Bar{\vtheta}_{\vw, \vz^{\perp}}(t)&=   \vw_k^0 -\left<\vw_k^0, \hat{\vz}\right>\hat{\vz}, 
\end{align*}
we observe that  since the RHS is independent of time $t$, hence we have  \[\Bar{\vtheta}_{\vw, \vz^{\perp}}(t)=\Bar{\vtheta}_{\vw, \vz^{\perp}}(0),\]
so we obtain that 
\begin{align*}
&\Norm{\vtheta_{\vw}(t)-\vtheta_{\vw}(0)}^2_2\\
=&\Norm{\Bar{\vtheta}_{\vw, \vz}(t)-\Bar{\vtheta}_{\vw, \vz}(0)+\widetilde{\vtheta}_{\vw, \vz}(t)}^2_2+\Norm{\Bar{\vtheta}_{\vw, \vz^{\perp}}(t)-\Bar{\vtheta}_{\vw, \vz^{\perp}}(0)+\widetilde{\vtheta}_{\vw, \vz^{\perp}}(t)}^2_2\\
=&\Norm{\Bar{\vtheta}_{\vw, \vz}(t)-\Bar{\vtheta}_{\vw, \vz}(0)+\widetilde{\vtheta}_{\vw, \vz}(t)}^2_2+\Norm{\widetilde{\vtheta}_{\vw, \vz^{\perp}}(t)}^2_2.
\end{align*} 
thus the ratio reads
\begin{align*}
\left(\frac{\Norm{{\vtheta}_{\vw}(t)-{\vtheta}_{\vw}(0)}_2}{\Norm{{\vtheta}_{\vw}(0)}_2}\right)^2&=\frac{\Norm{\Bar{\vtheta}_{\vw, \vz}(t)-\Bar{\vtheta}_{\vw, \vz}(0)+\widetilde{\vtheta}_{\vw, \vz}(t)}^2_2+\Norm{\widetilde{\vtheta}_{\vw, \vz^{\perp}}(t)}^2_2}{\Norm{{\vtheta}_{\vw}(0)}_2^2}\\
&=\frac{\Norm{\Bar{\vtheta}_{\vw, \vz}(t)-\Bar{\vtheta}_{\vw, \vz}(0)+\widetilde{\vtheta}_{\vw, \vz}(t)}_2^2}{\Norm{{\vtheta}_{\vw}(0)}_2^2}+\frac{\Norm{\widetilde{\vtheta}_{\vw, \vz^{\perp}}(t)}^2_2}{\Norm{{\vtheta}_{\vw}(0)}_2^2}.   
\end{align*}
As
\[
\Norm{{{\vtheta}}_{\vw}(0)}_2^2=\Norm{{{\vtheta}}_{\vw}^0}_2^2=\sum_{k=1}^m \left(\vw_k^0\right)^2,
\]
and we observe that since $\vw^0_k\sim \fN(\vzero, \mI_d)$, then $\left<\vw_k^0, \hat{\vz}\right>\sim \fN(0, 1)$. Moreover,    $\left\{a_k^0, \left<\vw_k^0, \hat{\vz}\right>\right\}_{k=1}^m\sim \fN(0, 1)$ are i.i.d.\ Gaussian variables. Hence, by application of \Cref{append...thm...BernsteinInequality}, with probability $1-\frac{\delta}{4}$ over the choice of $\vtheta^0$, for $m$ large enough,
\[
\frac{d}{2}\leq\frac{1}{m}\Norm{{\vtheta}_{\vw}^0}_2^2\leq \frac{3d}{2},
\]
and with probability $1-\frac{\delta}{4}$ over the choice of $\vtheta^0$, for $m$ large enough,
\[
\frac{1}{2}\leq\frac{1}{m}\sum_{k=1}^m\left(a_k^0\right)^2\leq \frac{3 }{2},
\]
and with probability $1-\frac{\delta}{4}$ over the choice of $\vtheta^0$, for $m$ large enough,
\[
\frac{1}{2}\leq\frac{1}{m}\sum_{k=1}^m\left(\left<\vw_k^0, \hat{\vz}\right>\right)^2\leq \frac{3 }{2},
\]
and with probability $1-\frac{\delta}{4}$ over the choice of $\vtheta^0$, for $m$ large enough,
\[
-\frac{1}{4}\leq\frac{1}{m}a_k^0\left<\vw_k^0, \hat{\vz}\right>\leq\frac{1}{4}.
\]
To sum up, the ratio
\begin{align*}
\left(\frac{\Norm{{\vtheta}_{\vw}(t)-{\vtheta}_{\vw}(0)}_2}{\Norm{{\vtheta}_{\vw}(0)}_2}\right)^2 &=\frac{\Norm{\Bar{\vtheta}_{\vw, \vz}(t)-\Bar{\vtheta}_{\vw, \vz}(0)+\widetilde{\vtheta}_{\vw, \vz}(t)}_2^2}{\Norm{{\vtheta}_{\vw}(0)}_2^2}+\frac{\Norm{\widetilde{\vtheta}_{\vw, \vz^{\perp}}(t)}^2_2}{\Norm{{\vtheta}_{\vw}(0)}_2^2},
\end{align*}
for the first part of the RHS:
\begin{align*}
&\frac{\Norm{\Bar{\vtheta}_{\vw, \vz}(t)-\Bar{\vtheta}_{\vw, \vz}(0)+\widetilde{\vtheta}_{\vw, \vz}(t)}_2}{\Norm{{\vtheta}_{\vw}(0)}_2}\\
&\geq   \frac{\Norm{\Bar{\vtheta}_{\vw, \vz}(t)-\Bar{\vtheta}_{\vw, \vz}(0)}_2-\Norm{\widetilde{\vtheta}_{\vw, \vz}(t)}_2}{\Norm{{\vtheta}_{\vw}(0)}_2}\\ 
&=\frac{\Norm{\Bar{\vtheta}_{\vw, \vz}(t)-\Bar{\vtheta}_{\vw, \vz}(0)}_2}{{\Norm{{\vtheta}_{\vw}(0)}_2}}-\frac{\Norm{\widetilde{\vtheta}_{\vw, \vz}(t)}_2}{\Norm{{\vtheta}_{\vw}(0)}_2},
\end{align*}
and 
\begin{equation*}
\begin{aligned}
 &~~\frac{1}{m}\Norm{\Bar{\vtheta}_{\vw, \vz}(t)-\Bar{\vtheta}_{\vw, \vz}(0)}_2^2  \\&=\sum_{k=1}^m\left[\frac{1}{2}\left(r^2(t)-{r^{-2}(t)}\right)\frac{\nu}{\eps}a_k^0+\frac{1}{2}\left(r^2(t)+{r^{-2}(t)}-2\right)\left<\vw_k^0, \hat{\vz}\right>\right]^2\\
&=\frac{1}{4}\left(r(t)-{r^{-1}(t)}\right)^2\sum_{k=1}^m\left[\left(r(t)+{r^{-1}(t)}\right)\frac{\nu}{\eps}a_k^0+\left(r(t)-{r^{-1}(t)}\right)\left<\vw_k^0, \hat{\vz}\right>\right]^2\\
&=\frac{1}{4}\left(r(t)-{r^{-1}(t)}\right)^2\sum_{k=1}^m\Bigg\{\left[\left(r(t)+ {r^{-1}(t)}\right)^2\left(\frac{\nu}{\eps}a_k^0\right)^2+\left(r(t)- {r^{-1}(t)}\right)^2\left(\left<\vw_k^0, \hat{\vz}\right>\right)^2\right]\\
&~~~~~~~~~~~~~~~~~~~~~~~~~~~~~~~~+2\left(r^2(t)- {r^{-2}(t)}\right)\frac{\nu}{\eps}a_k^0\left<\vw_k^0, \hat{\vz}\right>\Bigg\}\\
&\geq  \frac{1}{4}\left(r(t)- {r^{-1}(t)}\right)^2  \Bigg[\left(r(t)+ {r^{-1}(t)}\right)^2\left(\frac{\nu}{\eps}\right)^2\frac{1}{2}+\left(r(t)- {r^{-1}(t)}\right)^2 \frac{1}{2}\\
&~~~~~~~~~~~~~~~~~~~~~~~~~~~~~~~~~-\left(r^2(t)- {r^{-2}(t)}\right)\frac{\nu}{\eps}\frac{1}{2} \Bigg]\\
&\geq\frac{3}{32}\left(r(t)-{r^{-1}(t)}\right)^2 \left(r(t)- {r^{-1}(t)}\right)^2=\frac{3}{32}\left(r(t)-{r^{-1}(t)}\right)^4.
\end{aligned}
\end{equation*}
Then  with probability at least $1-\delta$ over the choice of $\vtheta^0$ and   large enough $m$,  for any $0\leq t \leq \bar{t}_0=\left({\frac{\gamma-\gamma'-1}{8}}\right)\log(m)-\log(2)$, the following holds:
\begin{align*}
&\frac{\Norm{\Bar{\vtheta}_{\vw, \vz}(t)-\Bar{\vtheta}_{\vw, \vz}(0)+\widetilde{\vtheta}_{\vw, \vz}(t)}_2}{\Norm{{\vtheta}_{\vw}(0)}_2}\\
&\geq  \frac{\Norm{\Bar{\vtheta}_{\vw, \vz}(t)-\Bar{\vtheta}_{\vw, \vz}(0)}_2}{{\Norm{{\vtheta}_{\vw}(0)}_2}}-\frac{\Norm{\widetilde{\vtheta}_{\vw}(t)}_2}{\Norm{{\vtheta}_{\vw}(0)}_2}\\
&\geq \sqrt{\frac{2}{3md}}\sqrt{\frac{3m}{32}} \left(r(t)-{r^{-1}(t)}\right)^2 -\sqrt{\frac{2}{d}}m^{-\frac{\gamma-\gamma'-1}{8}}.
\end{align*}
Specifically, if we choose $t=\Bar{t}_0$,
\begin{align*}
&\frac{\Norm{\Bar{\vtheta}_{\vw, \vz}(t_0)-\Bar{\vtheta}_{\vw, \vz}(0)+\widetilde{\vtheta}_{\vw, \vz}(t_0)}_2}{\Norm{{\vtheta}_{\vw}(0)}_2}\\
&\geq  \sqrt{\frac{1 }{16d}} \left(r^2(t_0)+r^{-2}(t_0)-2\right) -\sqrt{\frac{2}{d}}m^{-\frac{\gamma-\gamma'-1}{8}}\\
&\gtrsim m^{\frac{\gamma-\gamma'-1}{8}}-m^{-\frac{\gamma-\gamma'-1}{8}}.
\end{align*}
By taking limit, we obtain that 
\[
\lim_{m\to\infty}\frac{\Norm{\Bar{\vtheta}_{\vw, \vz}(t_0)-\Bar{\vtheta}_{\vw, \vz}(0)+\widetilde{\vtheta}_{\vw, \vz}(t_0)}_2}{\Norm{{\vtheta}_{\vw}(0)}_2}=+\infty.
\]
For the second part of the RHS:
\begin{align*}
 \frac{\Norm{\widetilde{\vtheta}_{\vw, \vz^{\perp}}(t)}_2}{\Norm{{\vtheta}_{\vw}(0)}_2}\leq \frac{\Norm{\widetilde{\vtheta}_{\vw}(t)}_2}{\Norm{{\vtheta}_{\vw}(0)}_2},
\end{align*}
then  with probability at least $1-\delta$ over the choice of $\vtheta^0$ and   large enough $m$,  for any $0\leq t \leq \bar{t}_0=\left({\frac{\gamma-\gamma'-1}{8}}\right)\log(m)-\log(2)$, the following holds:
\begin{align*}
 \frac{\Norm{\widetilde{\vtheta}_{\vw, \vz^{\perp}}(t)}_2}{\Norm{{\vtheta}_{\vw}(0)}_2}\leq \frac{\Norm{\widetilde{\vtheta}_{\vw}(t)}_2}{\Norm{{\vtheta}_{\vw}(0)}_2}\leq \sqrt{\frac{2}{d}}m^{-\frac{\gamma-\gamma'-1}{8}}.
\end{align*}
By taking limit, we obtain that 
\[
\lim_{m\to\infty}\frac{\Norm{\widetilde{\vtheta}_{\vw, \vz^{\perp}}(t)}_2}{\Norm{{\vtheta}_{\vw}(0)}_2}=0.
\]
To sum up, since $t_0\leq T_{\mathrm{eff}}$, we have that 
\begin{equation}
\lim_{m\to+\infty} \sup\limits_{t\in[0, T_{\mathrm{eff}}]}\mathrm{RD}(\vtheta_{\vw}(t))=+\infty+0=+\infty,
\end{equation}
which finishes the proof of \eqref{append...eq...thm...Condense...WLag...PartOne}.

In order to prove \eqref{append...eq...thm...Condense...WLag...PartTwo}, firstly we have
\[
\frac{\Norm{\vtheta_{\vw, \vz}(t) }_2}{\Norm{\vtheta_{\vw}(t)}_2}\leq 1,
\]
moreover, we observe that 
\begin{align*}
\left(\frac{\Norm{\vtheta_{\vw, \vz}(t) }_2}{\Norm{\vtheta_{\vw}(t)}_2}\right)^2&=\frac{\Norm{\vtheta_{\vw, \vz}(t)}^2_2}{\Norm{\vtheta_{\vw}(t)}^2_2}=\frac{\Norm{\vtheta_{\vw, \vz}(t)}^2_2}{\Norm{\vtheta_{\vw, \vz}(t)}^2_2+\Norm{\vtheta_{\vw, \vz^{\perp}}(t)}^2_2}\\
&=\frac{\Norm{\Bar{\vtheta}_{\vw, \vz}(t)+\widetilde{\vtheta}_{\vw, \vz}(t)}^2_2}{{\Norm{\Bar{\vtheta}_{\vw, \vz}(t)+\widetilde{\vtheta}_{\vw, \vz}(t)}^2_2}+\Norm{\Bar{\vtheta}_{\vw, \vz^{\perp}}(t)+\widetilde{\vtheta}_{\vw, \vz^{\perp}}(t)}^2_2}.
\end{align*}
Then  with probability at least $1-\delta$ over the choice of $\vtheta^0$ and   large enough $m$,  for any $0\leq t \leq \bar{t}_0=\left({\frac{\gamma-\gamma'-1}{8}}\right)\log(m)-\log(2)$, the following holds:
\begin{align*}
\Norm{\Bar{\vtheta}_{\vw, \vz}(t)+\widetilde{\vtheta}_{\vw, \vz}(t)}_2&\geq \Norm{\Bar{\vtheta}_{\vw, \vz}(t)-\Bar{\vtheta}_{\vw, \vz}(0)+ \widetilde{\vtheta}_{\vw, \vz}(t)}_2-   \Norm{\Bar{\vtheta}_{\vw, \vz}(0)}_2\\ 
&\geq \sqrt{\frac{3m}{32}} \left(r(t)-{r^{-1}(t)}\right)^2-\Norm{\widetilde{\vtheta}_{\vw}(t)}_2-\Norm{\Bar{\vtheta}_{\vw}(0)}_2\\
&\geq  \sqrt{\frac{3m}{32}} \left(r(t)-{r^{-1}(t)}\right)^2-\sqrt{m} m^{-\frac{\gamma-\gamma'-1}{8}}-\sqrt{\frac{3md}{2}},\\
\Norm{\Bar{\vtheta}_{\vw, \vz^{\perp}}(t)+\widetilde{\vtheta}_{\vw, \vz^{\perp}}(t)}_2&\leq \Norm{\Bar{\vtheta}_{\vw, \vz^{\perp}}(t)-\Bar{\vtheta}_{\vw, \vz^{\perp}}(0)+\widetilde{\vtheta}_{\vw, \vz^{\perp}}(t)}_2+ \Norm{\Bar{\vtheta}_{\vw, \vz^{\perp}}(0)}_2\\
&\leq  \Norm{\widetilde{\vtheta}_{\vw}(t)}_2+\Norm{\Bar{\vtheta}_{\vw}(0)}_2\\
&\leq  \sqrt{m} m^{-\frac{\gamma-\gamma'-1}{8}}+\sqrt{\frac{3md}{2}}.
\end{align*}
By taking $t=\bar{t}_0$, we observe that  $\Norm{\Bar{\vtheta}_{\vw, \vz}(t)+\widetilde{\vtheta}_{\vw, \vz}(t)}_2$ is of order at least $\sqrt{m}m^{\frac{\gamma-\gamma'-1}{8}}$, while $\Norm{\Bar{\vtheta}_{\vw, \vz^{\perp}}(t)+\widetilde{\vtheta}_{\vw, \vz^{\perp}}(t)}_2$  is of order at most $\sqrt{m}$, which finishes the proof of \eqref{append...eq...thm...Condense...WLag...PartTwo}.
\end{proof}
\subsection{$a$-lag regime}\label{subsection...append...condense...aLag}
Recall the real dynamics, 
\begin{equation*}
\begin{aligned}
 \frac{\D \begin{pmatrix}\frac{\nu}{\eps}a_k\\
\vw_k \end{pmatrix}}{\D t}
=\mA\begin{pmatrix}\frac{\nu}{\eps}a_k\\
\vw_k \end{pmatrix}
+
\begin{pmatrix}f_k\\
\frac{\nu}{\eps}\vg_k \end{pmatrix},\quad \begin{pmatrix}\frac{\nu}{\eps}a_k(0)\\
\vw_k(0) \end{pmatrix}=\begin{pmatrix}\frac{\nu}{\eps}a_k^0\\
\vw_k^0 \end{pmatrix},
\end{aligned}    
\end{equation*}
and its solution   \eqref{append...eq...text...Condense...Difference...RealDynamicsRewritten} reads
\begin{equation*}
 \begin{aligned}
\begin{pmatrix}\frac{\nu}{\eps}a_k\\
\vw_k \end{pmatrix}
=\exp\left(t\mA\right)\begin{pmatrix}\frac{\nu}{\eps}a_k^0\\
\vw_k^0 \end{pmatrix}
+
\int_{0}^t\exp\left((t-s)\mA\right)\begin{pmatrix}f_k(s)\\
\frac{\nu}{\eps}\vg_k(s) \end{pmatrix}\D s.
\end{aligned}        
\end{equation*}
Then, for any $\vu\in\vz^\perp$, and $\Norm{\vu}_2=1$, we obtain that 
\begin{align*}
 \frac{\D \left<\vw_k, \vu\right>}{\D t}&=\frac{\nu}{\eps}\left<\vz, \vu\right>+ \frac{\nu}{\eps}\left<\vg_k, \vu\right>\\ 
 &=\frac{\nu}{\eps}\left<\vg_k, \vu\right>,
\end{align*}
and as 
\[
\vtheta_{\vw, \vz^{\perp}}=\Bar{\vtheta}_{\vw, \vz^{\perp}}+\widetilde{\vtheta}_{\vw, \vz^{\perp}}
\]
hence the $k$-th component of  $\vtheta_{\vw, \vz^{\perp}}$, $\Bar{\vtheta}_{\vw, \vz^{\perp}}$ and $\widetilde{\vtheta}_{\vw, \vz^{\perp}}$ altogether reads
\begin{align*}
\left(\Bar{\vtheta}_{\vw, \vz^{\perp}}\right)_k &=\vw_k^0 -\left<\vw_k^0, \hat{\vz}\right>\hat{\vz},\\    
\left(\widetilde{\vtheta}_{\vw, \vz^{\perp}}\right)_k&=\frac{\nu}{\eps}\int_0^t\left[\vg_k(s)-\left<\vg_k(s), \hat{\vz}\right>\hat{\vz}\right]\D s,\\
\left({\vtheta}_{\vw, \vz^{\perp}}\right)_k &=\vw_k^0 -\left<\vw_k^0, \hat{\vz}\right>\hat{\vz}+\frac{\nu}{\eps}\int_0^t\left[\vg_k(s)-\left<\vg_k(s), \hat{\vz}\right>\hat{\vz}\right]\D s.
\end{align*}
Moreover, the real dynamics can also be written as
\begin{equation*}
\begin{aligned}
 \frac{\D \begin{pmatrix}\frac{\nu}{\eps}a_k\\
\left<\vw_k, \hat{\vz}\right> \end{pmatrix}}{\D t}
=\mB\begin{pmatrix}\frac{\nu}{\eps}a_k\\
\left<\vw_k,\hat{\vz}\right>  \end{pmatrix}
+
\begin{pmatrix}f_k\\
\frac{\nu}{\eps}\left<\vg_k, \hat{\vz}\right> \end{pmatrix},\quad \begin{pmatrix}\frac{\nu}{\eps}a_k(0)\\
\left<\vw_k, \hat{\vz}\right>(0) \end{pmatrix}=\begin{pmatrix}\frac{\nu}{\eps}a_k^0\\
\left<\vw_k^0,\hat{\vz}\right> \end{pmatrix},
\end{aligned}    
\end{equation*}
where
\[
\mB:=\begin{pmatrix}0 & \Norm{\vz}_2\\
\Norm{\vz}_2&0 \end{pmatrix}=\Norm{\vz}_2\begin{pmatrix}0 & 1\\
1&0 \end{pmatrix},
\]
since  $\mB$ is a full rank matrix, the solution to the above dyanmics can be explicitly written out as 
\begin{align*}
\begin{pmatrix}\frac{\nu}{\eps}a_k\\
\left<\vw_k, \hat{\vz}\right> \end{pmatrix}&=\exp\left(t\mB\right)\begin{pmatrix}\frac{\nu}{\eps}a_k^0\\
\left<\vw_k^0, \hat{\vz}\right> \end{pmatrix}
+
\int_{0}^t\exp\left((t-s)\mB\right)\begin{pmatrix}f_k(s)\\
\frac{\nu}{\eps}\left<\vg_k(s), \hat{\vz}\right> \end{pmatrix}\D s\\
&=\begin{pmatrix}\frac{1}{ {2}}r^2(t)+\frac{1}{ {2}}r^{-2}(t)&\frac{1}{ {2}}r^2(t)-\frac{1}{ {2}}r^{-2}(t) \\
\frac{1}{ {2}}r^2(t)-\frac{1}{ {2}}r^{-2}(t) &\frac{1}{ {2}}r^2(t)+\frac{1}{ {2}}r^{-2}(t)\end{pmatrix}\begin{pmatrix}\frac{\nu}{\eps}a_k^0\\
\left<\vw_k^0, \hat{\vz}\right> \end{pmatrix}\\
&~~+\int_{0}^t\begin{pmatrix}\frac{1}{ {2}}r^2(t-s)+\frac{1}{ {2}}r^{-2}(t-s)&\frac{1}{ {2}}r^2(t-s)-\frac{1}{ {2}}r^{-2}(t-s) \\
\frac{1}{ {2}}r^2(t-s)-\frac{1}{ {2}}r^{-2}(t-s) &\frac{1}{ {2}}r^2(t-s)+\frac{1}{ {2}}r^{-2}(t-s)\end{pmatrix}\begin{pmatrix}f_k(s)\\
\frac{\nu}{\eps}\left<\vg_k(s), \hat{\vz}\right> \end{pmatrix}\D s,
\end{align*}
hence we obtain that 
\begin{align*}
\left(\vtheta_a\right)_k&=\va_k= \left(\frac{1}{ {2}}r^2(t)+\frac{1}{ {2}}r^{-2}(t)\right) a_k^0+\frac{\eps}{\nu}\left(\frac{1}{ {2}}r^2(t)-\frac{1}{ {2}}r^{-2}(t)\right)\left<\vw_k^0, \hat{\vz}\right>\\    
&~~+\frac{\eps}{\nu}\int_{0}^t\left(\frac{1}{ {2}}r^2(t-s)+\frac{1}{ {2}}r^{-2}(t-s)\right)f_k(s)\D s\\
&~~+\int_{0}^t\left(\frac{1}{ {2}}r^2(t-s)-\frac{1}{ {2}}r^{-2}(t-s)\right)\left<\vg_k(s), \hat{\vz}\right>\D s,\\
\left(\vtheta_{\vw}\right)_k&=\vw_k=\frac{\nu}{\eps}\left(\frac{1}{ {2}}r^2(t)-\frac{1}{ {2}}r^{-2}(t)\right)a_k^0\hat{\vz} +\left(\frac{1}{ {2}}r^2(t)+\frac{1}{ {2}}r^{-2}(t)\right)\left<\vw_k^0, \hat{\vz}\right>\hat{\vz}\\
&~~+\int_{0}^t\left(\frac{1}{ {2}}r^2(t-s)-\frac{1}{ {2}}r^{-2}(t-s)\right)f_k(s)\D s\hat{\vz}\\
&~~+\frac{\nu}{\eps}\int_{0}^t\left(\frac{1}{ {2}}r^2(t-s)+\frac{1}{ {2}}r^{-2}(t-s)\right)\left<\vg_k(s), \hat{\vz}\right>\D s\hat{\vz}\\
&~~-\left<\vw_k^0, \hat{\vz}\right>\hat{\vz}+\vw_k^0+\frac{\nu}{\eps}\int_0^t\left[\vg_k(s)-\left<\vg_k(s), \hat{\vz}\right>\hat{\vz}\right]\D s.
\end{align*}
\begin{definition}[Neuron $\infty$-energy, $a$-lag regime]
In real dynamics, we define the  $\infty$-energy at time $t$ for each single neuron, i.e., for each $k\in[m]$, 
\begin{equation}
p_k(t):=\max_{k\in[m]}\left\{\Abs{a_k(t)},  \Norm{\vw_k(t)}_{\infty}\right\}.
\end{equation}
\end{definition}
We denote  
\begin{equation}
p_{\max}(t):=\max_{k\in[m]} p_k(t).
\end{equation}
For simplicity, we hereafter drop the $(t)$s for all $p_k(t)$ and $p_{\max}(t)$.  
Then the estimates on  $\{f_k, \vg_k\}_{k=1}^m$ read
\begin{proposition}\label{append...prop...CellProblemEstimates...A-lag}
For  any time $t>0$,
\begin{equation}
\begin{aligned}
\Abs{f_k}&\leq dm\nu\eps p_{\max}^2\Norm{{\vw}_k}_{\infty}+\eps d \Norm{{\vw}_k}_{\infty}^2,\\
\Norm{\vg_k}_{\infty}&\leq \sqrt{d}m\nu\eps p_{\max}^2\Abs{a_k}+\eps \sqrt{d}\Norm{{\vw}_k}_{\infty}\Abs{a_k}.
\end{aligned}    
\end{equation}
\end{proposition}
\begin{proof}
We obtain that 
\begin{align*}
\Abs{f_k}&=\Abs{\frac{1}{n}\sum_{i=1}^n \left((e_i+y_i)\frac{\sigma(\eps{\vw}_k^\T\vx_i)}{\eps}+y_i   {\vw}_k^\T\vx_i-y_i\frac{\sigma(\eps{\vw}_k^\T\vx_i)}{\eps}\right)}\\
&=\Abs{\frac{1}{n}\sum_{i=1}^n \left(\left(\sum_{k=1}^m \nu a_k \sigma(\eps{\vw}_k^\T\vx_i)\right)\frac{\sigma(\eps{\vw}_k^\T\vx_i)}{\eps}+y_i   {\vw}_k^\T\vx_i-y_i\frac{\sigma(\eps{\vw}_k^\T\vx_i)}{\eps}\right)}\\
&\leq \frac{1}{n}\sum_{i=1}^n \left(\left(\sum_{k=1}^m \nu\eps \Abs{a_k} \Norm{{\vw}_k}_2 \right)\Norm{{\vw}_k}_2+  \eps ({\vw}_k^\T\vx_i)^2 \right)\\
&\leq\frac{1}{n}\sum_{i=1}^n \left(\left(\sum_{k=1}^m \nu\eps \Abs{a_k}\sqrt{d}\Norm{{\vw}_k}_{\infty} \right)\sqrt{d}\Norm{{\vw}_k}_{\infty}+  \eps d\Norm{{\vw}_k}_{\infty}^2\right)  \\
& \leq dm\nu\eps r_{\max}^2\Norm{{\vw}_k}_{\infty}+\eps d \Norm{{\vw}_k}_{\infty}^2,
\end{align*}
and
\begin{align*}
\Norm{\vg_k}_{\infty}&=\Norm{\frac{1}{n}\sum_{i=1}^n \left((e_i+y_i){a}_k\sigma^{(1)}(\eps {\vw}_k^\T\vx_i)\vx_i+y_i   {a}_k\vx_i-y_i{a}_k\sigma^{(1)}(\eps {\vw}_k^\T\vx_i)\vx_i\right)}_{\infty}\\
&=\Norm{\frac{1}{n}\sum_{i=1}^n \left(\left(\sum_{k=1}^m \nu a_k \sigma(\eps{\vw}_k^\T\vx_i)\right){a}_k\sigma^{(1)}(\eps {\vw}_k^\T\vx_i)\vx_i+y_i   {a}_k\vx_i-y_i{a}_k\sigma^{(1)}(\eps {\vw}_k^\T\vx_i)\vx_i\right)}_{\infty}\\
&\leq \frac{1}{n}\sum_{i=1}^n \left(\left(\sum_{k=1}^m \nu\eps \Abs{a_k} \Norm{{\vw}_k}_2 \right)\Abs{a_k}+  \eps \Norm{{\vw}_k}_2\Abs{a_k} \right)\\
&\leq \frac{1}{n}\sum_{i=1}^n \left( \left(\sum_{k=1}^m \nu\eps \Abs{a_k}\sqrt{d}\Norm{{\vw}_k}_{\infty} \right)\Abs{a_k}+  \eps\sqrt{d}\Norm{{\vw}_k}_{\infty}\Abs{a_k}\right)\\
 &\leq \sqrt{d}m\nu\eps r_{\max}^2\Abs{a_k}+\eps \sqrt{d}\Norm{{\vw}_k}_{\infty}\Abs{a_k}.
\end{align*}
\end{proof}
We denote a useful quantity
\begin{equation} 
\psi(t):= \sup_{0\leq s\leq t} {p}_{\max}(s).
\end{equation}
Then directly from \Cref{append...lemma..Initialization}, we have with probability at least $1-\delta$ over the choice of $\vtheta^0$,
\begin{equation}
\max\limits_{k\in[m]}\left\{\abs{a_k^0},\;\norm{\vw_k^0}_{\infty}\right \}\leq \sqrt{2\log\frac{2m(d+1)}{\delta}},
\end{equation}
hence
\begin{equation} \label{apend...eq...text...PSI}
 \psi(0)\leq\sqrt{2(d+1)\log\frac{2m(d+1)}{\delta}},
\end{equation}
and  for all $k\in[m]$ and any time $t>0$,
\begin{align*}
\Abs{f_k}&\leq dm\nu\eps \psi^3(t)+\eps d \psi^2(t),\\
\Norm{\vg_k}_{\infty}&\leq \sqrt{d}m\nu\eps \psi^3(t)+\eps \sqrt{d}\psi^2(t).    
\end{align*}

We are hereby to conduct some simple calculations. Let $\alpha>0$, and there exists $t_{\alpha}>0$, such that 
\[
r^2(t_{\alpha})-r^{-2}(t_{\alpha})=2m^{-\alpha},
\]
then we obtain immediately that 
\begin{align*}
r^2(t_{\alpha})&=\exp(\Norm{\vz}_2 t_{\alpha}) =\sqrt{1+m^{-\alpha}}+   m^{-\alpha},
\end{align*}
and 
\[
  t_{\alpha}\lesssim \log(\sqrt{1+   m^{-\alpha}}+m^{-\alpha})\sim \frac{3}{2}m^{-\alpha}\sim\fO(1),
\]
with 
\[
  t_{\alpha}\gtrsim \log(\sqrt{1+   m^{-\alpha}}+m^{-\alpha})\sim \frac{3}{2}m^{-\alpha}\sim\Omega(1).
\]
Moreover,
\begin{align*}
\frac{1}{ {2}}r^2(t_{\alpha})+\frac{1}{ {2}}r^{-2}(t_{\alpha})&=\sqrt{1+m^{-\alpha}},\\
\frac{1}{ {2}}r^2(t_{\alpha})-\frac{1}{ {2}}r^{-2}(t_{\alpha})&=m^{-\alpha},\\
\int_{0}^{t_{\alpha}} \frac{1}{ {2}}r^2(t_{\alpha}-s)+\frac{1}{ {2}}r^{-2}(t_{\alpha}-s)\D s&=\frac{m^{-\alpha}}{\Norm{\vz}_2},\\
\int_{0}^{t_{\alpha}} \frac{1}{ {2}}r^2(t_{\alpha}-s)-\frac{1}{ {2}}r^{-2}(t_{\alpha}-s)\D s&=\frac{\sqrt{1+m^{-\alpha}}-1}{\Norm{\vz}_2},
\end{align*}
hence 
\begin{equation}\label{both}
\begin{aligned}
\Abs{a_k}&\leq \psi(0)+m^{-\alpha}\left(m\nu\eps \psi^3(t_{\alpha})+\eps\psi^2(t_{\alpha})\right),\\
\Norm{\vw_k}_{\infty}&\leq \frac{\nu}{\eps}m^{-\alpha}\psi(0)+\frac{\nu}{\eps}m^{-\alpha}\left(m\nu\eps \psi^3(t_{\alpha})+\eps\psi^2(t_{\alpha})\right),
\end{aligned}
\end{equation}
For some $\beta>1$,  we define
\begin{equation}
\widetilde{T}_{\mathrm{eff}, \beta}: = \inf\left\{t>0 \mid m\nu\eps\psi^3(t)+\eps\psi^2(t)>m^{-\widetilde{\tau}},\quad \widetilde{\tau}=-\frac{\gamma'}{\beta} \right\},
\end{equation}
then for $m$ large enough, as $\gamma >1$, based on \eqref{apend...eq...text...PSI}, we choose $\beta$ satisfying \[\min\left\{\gamma-1, \frac{1}{2}\right\}> -\frac{\gamma'}{\beta},\]
then, we obtain that
\begin{equation*}
m\nu\eps\psi^3(0)+\eps\psi^2(0)\leq  m\nu\eps\left(2(d+1)\log\frac{2m(d+1)}{\delta}\right)^{\frac{3}{2}}+m^{-\frac{1}{2}}\left(2(d+1)\log\frac{2m(d+1)}{\delta}\right) \leq m^{-\frac{\gamma-1}{2\beta}},
\end{equation*}
hence $\widetilde{T}_{\mathrm{eff}, \beta}\geq 0$.

We observe further that  by taking the $\infty$-norm on both sides of \eqref{both}, and by taking supreme over the  index $k$ and  time $0\leq t\leq \min\left\{t_{\alpha}, \widetilde{T}_{\mathrm{eff}, \beta}\right\}$   on both sides,  and  for  large enough $m$,   the following holds 
\begin{align*}
\psi(t)\leq \frac{\nu}{\eps} m^{-\alpha}\psi(0) +\frac{\nu}{\eps} m^{-\alpha}m^{-\frac{\gamma'}{\beta}}\leq   m^{-\gamma'-\alpha}\sqrt{2(d+1)\log\frac{2m(d+1)}{\delta}},
\end{align*}
now we shall choose $\alpha>0$ and $\beta>1$, such that 
\[ t_{\alpha}\leq  \widetilde{T}_{\mathrm{eff}, \beta},\]
which is equivalent to solve out the relation
\[
m^{1-\gamma}m^{-3\gamma'-3\alpha}+m^{-\frac{1}{2}}m^{-2\gamma'-2\alpha}\leq m^{\frac{\gamma'}{\beta}}.
\]
Then, we choose $\alpha$ satisfying
\[
0<\gamma'-\alpha\leq \min\left\{ \frac{1}{3}\left(\gamma-1+\frac{\gamma'}{\beta}\right), \frac{1}{2}\left(\frac{1}{2}+\frac{\gamma'}{\beta}\right)\right\},
\]
the existence of $\alpha$ can be guaranteed since $\gamma'>0,$ and\[\min\left\{ \frac{1}{3}\left(\gamma-1+\frac{\gamma'}{\beta}\right), \frac{1}{2}\left(\frac{1}{2}+\frac{\gamma'}{\beta}\right)\right\}\geq \frac{\gamma'}{6\beta}>0.\]
We observe that  
\begin{equation}
\widetilde{T}_{\mathrm{eff}, \beta}\geq      t_{\alpha},\quad \widetilde{T}_{\mathrm{eff}, \beta}\sim \Omega(1).
\end{equation}
WLOG, we choose $\Bar{\beta}:=\max\left\{-1024\frac{\gamma'}{\gamma-1}, -512 \gamma'\right\}$, and  $\Bar{\alpha}$ accordingly, and finally we have
\begin{theorem}[Condensed regime, $a$-lag regime]\label{append...thm..CondensedRegime...a-Einent}
Given any $\delta\in(0,1)$, under \Cref{Assumption....ActivationFunctions}, \Cref{assumption...GenericData} and \Cref{assump...LimitExistence}, if $\gamma> 1$ and $ \gamma'< 0$, 
then	with probability at least $1-\delta$ over the choice of $\vtheta^0$,
we have  
\begin{equation}\label{append...eq...thm...Condense...ALag...PartOne}
\lim_{m\to+\infty} \sup\limits_{t\in[0, \widetilde{T}_{\mathrm{eff}, \Bar{\beta}}]}\mathrm{RD}(\vtheta_{\vw}(t))=+\infty,
\end{equation}
and
\begin{equation}\label{append...eq...thm...Condense...ALag...PartTwo}
\lim_{m\to+\infty} \sup\limits_{t\in[0,\widetilde{T}_{\mathrm{eff}, \Bar{\beta}}]}\frac{\Norm{\vtheta_{\vw, \vz}(t) }_2}{\Norm{\vtheta_{\vw}(t)}_2} =1.
\end{equation}
\end{theorem}
As the details are almost the same as the ones in \Cref{append...thm..CondensedRegime...W-Einent...PartOne}, the proof of \Cref{append...thm..CondensedRegime...a-Einent} is written in a  slightly sketchy way.
\begin{proof}
We observe that
\begin{align*}
 &~~\left(\vtheta_{\vw}(t_{\Bar{\alpha}})-\vtheta_{\vw}(0)\right)_k \\
 &=\frac{\nu}{\eps}\left(\frac{1}{ {2}}r^2(t_{\Bar{\alpha}})-\frac{1}{ {2}}r^{-2}(t_{\Bar{\alpha}})\right)a_k^0\hat{\vz} +\left(\frac{1}{ {2}}r^2(t_{\Bar{\alpha}})+\frac{1}{ {2}}r^{-2}(t_{\Bar{\alpha}})-1\right)\left<\vw_k^0, \hat{\vz}\right>\hat{\vz}\\
&~~+\int_{0}^{t_{\Bar{\alpha}}}\left(\frac{1}{ {2}}r^2(t_{\Bar{\alpha}}-s)-\frac{1}{ {2}}r^{-2}(t_{\Bar{\alpha}}-s)\right)f_k(s)\D s\hat{\vz}\\
&~~+\frac{\nu}{\eps}\int_{0}^{t_{\Bar{\alpha}}}\left(\frac{1}{ {2}}r^2(t_{\Bar{\alpha}}-s)+\frac{1}{ {2}}r^{-2}(t_{\Bar{\alpha}}-s)\right)\left<\vg_k(s), \hat{\vz}\right>\D s\hat{\vz}\\
&~~+\frac{\nu}{\eps}\int_0^{t_{\Bar{\alpha}}}\left[\vg_k(s)-\left<\vg_k(s), \hat{\vz}\right>\hat{\vz}\right]\D s,
\end{align*} 
so we obtain that with probability at least $1-\delta$ over the choice of $\vtheta^0$ and   large enough $m$,  for any $0\leq t \leq  {t}_{\Bar{\alpha}}$, the following holds:
\begin{align*}
&~~\frac{1}{m}\Norm{\vtheta_{\vw}(t)-\vtheta_{\vw}(0)}^2_2\\
&\geq m^{-\gamma' -\Bar{\alpha}}\left(\frac{1}{2}-\sqrt{d}\max_{k\in[m]}\Norm{\vg_k}_{\infty}\right) -\left(\sqrt{1+m^{-\Bar{\alpha}}}-1\right)\left(\frac{3d}{2}+\sqrt{d}\max_{k\in[m]}\Abs{f_k}\right)\\
&\geq m^{-\gamma' -\Bar{\alpha}}\left(\frac{1}{2}-\sqrt{d}m^{-\widetilde{\tau}}\right)-\frac{1}{2}m^{ -\Bar{\alpha}}\left(\frac{3d}{2}+\sqrt{d}m^{-\widetilde{\tau}}\right),
\end{align*}
hence the ratio
\begin{align*}
&~~\left(\frac{\Norm{{\vtheta}_{\vw}(t)-{\vtheta}_{\vw}(0)}_2}{\Norm{{\vtheta}_{\vw}(0)}_2}\right)^2 =\frac{\frac{1}{m}\Norm{{\vtheta}_{\vw}(t)-{\vtheta}_{\vw}(0)}^2_2}{\frac{1}{m}\Norm{{\vtheta}_{\vw}(0)}_2^2}\\ 
&\geq \frac{2}{3d}\left[{m^{-\gamma' -\Bar{\alpha}}\left(\frac{1}{2}-\sqrt{d}m^{-\widetilde{\tau}}\right)-\frac{1}{2}m^{ -\Bar{\alpha}}\left(\frac{3d}{2}+\sqrt{d}m^{-\widetilde{\tau}}\right)} \right],
\end{align*}
by taking limit, we obtain that  for any $0\leq t \leq  {t}_{\Bar{\alpha}}$
\[
\lim_{m\to\infty}\frac{\Norm{{\vtheta}_{\vw}(t)-{\vtheta}_{\vw}(0)}_2}{\Norm{{\vtheta}_{\vw}(0)}_2}=+\infty.
\]
Moreover, since
\begin{align*}
 \left(\vtheta_{\vw, \vz}(t_{\Bar{\alpha}})\right)_k &=\frac{\nu}{\eps}\left(\frac{1}{ {2}}r^2(t_{\Bar{\alpha}})-\frac{1}{ {2}}r^{-2}(t_{\Bar{\alpha}})\right)a_k^0 +\left(\frac{1}{ {2}}r^2(t_{\Bar{\alpha}})+\frac{1}{ {2}}r^{-2}(t_{\Bar{\alpha}})\right)\left<\vw_k^0, \hat{\vz}\right>\\
&~~+\int_{0}^{t_{\Bar{\alpha}}}\left(\frac{1}{ {2}}r^2(t_{\Bar{\alpha}}-s)-\frac{1}{ {2}}r^{-2}(t_{\Bar{\alpha}}-s)\right)f_k(s)\D s \\
&~~+\frac{\nu}{\eps}\int_{0}^{t_{\Bar{\alpha}}}\left(\frac{1}{ {2}}r^2(t_{\Bar{\alpha}}-s)+\frac{1}{ {2}}r^{-2}(t_{\Bar{\alpha}}-s)\right)\left<\vg_k(s), \hat{\vz}\right>\D s,\\
\left(\vtheta_{\vw, \vz^{\perp}}(t_{\Bar{\alpha}})\right)_k &= \vw_k^0-\left<\vw_k^0, \hat{\vz}\right>\hat{\vz}+\frac{\nu}{\eps}\int_0^t\left[\vg_k(s)-\left<\vg_k(s), \hat{\vz}\right>\hat{\vz}\right]\D s,
\end{align*} 
so we obtain that with probability at least $1-\delta$ over the choice of $\vtheta^0$ and   large enough $m$,  for any $0\leq t \leq  {t}_{\Bar{\alpha}}$, the following holds:
\begin{align*}
\frac{1}{m}\Norm{\vtheta_{\vw, \vz}(t)}^2_2
&\geq m^{-\gamma' -\Bar{\alpha}}\left(\frac{1}{2}-\sqrt{d}\max_{k\in[m]}\Norm{\vg_k}_{\infty}\right) -\sqrt{1+m^{-\Bar{\alpha}}}\left(\frac{3d}{2}+\sqrt{d}\max_{k\in[m]}\Abs{f_k}\right)\\
&\geq m^{-\gamma' -\Bar{\alpha}}\left(\frac{1}{2}-\sqrt{d}m^{-\widetilde{\tau}}\right)-2\left(\frac{3d}{2}+\sqrt{d}m^{-\widetilde{\tau}}\right),\\
\frac{1}{m}\Norm{\vtheta_{\vw, \vz^{\perp}}(t)}^2_2&\leq \frac{3d}{2}+m^{-\gamma'-{\Bar{\alpha}}} \sqrt{d}\max_{k\in[m]}\Norm{\vg_k}_{\infty}\leq  \frac{3d}{2}+m^{-\gamma'-{\Bar{\alpha}}}\sqrt{d}m^{-\widetilde{\tau}}.
 \end{align*}
By taking $t={t}_{\Bar{\alpha}}$, we observe that  $\frac{1}{m}\Norm{{\vtheta}_{\vw, \vz}(t)}^2_2$ is of order at least $ m^{{-\gamma'-\Bar{\alpha}}}$, while $\frac{1}{m}\Norm{{\vtheta}_{\vw, \vz^{\perp}}(t)}^2_2$  is of order at most one, which finishes the proof. 
\end{proof}
\end{document}